\documentclass[twoside]{article}

\usepackage[preprint]{aistats2027}

\usepackage{pgfplots}
  \pgfplotsset{compat=newest}
  \usetikzlibrary{plotmarks}
  \usetikzlibrary{arrows.meta}
  \usepgfplotslibrary{patchplots}
  \usepackage{grffile}
  \usepackage{amsmath}
  
\usepgfplotslibrary{fillbetween}

\usepackage{algorithm}
\usepackage{algorithmic}
\usepackage{float} 
\usepgfplotslibrary{groupplots}\usepackage{microtype}

\usepackage{microtype}
\usepackage{graphicx}
\usepackage{subcaption}
\usepackage{amsmath,amssymb}
\usepackage{tikz}
\usetikzlibrary{arrows.meta,positioning,calc}
\usepackage{booktabs} 
\usepackage{hyperref}

\usepackage[textsize=tiny]{todonotes}

\usepackage{amsmath}
\usepackage{amssymb}
\usepackage{mathtools}
\usepackage{amsthm}

\theoremstyle{plain}
\newtheorem{theorem}{Theorem}[section]

\newtheorem{lemma}[theorem]{Lemma}
\newtheorem{corollary}[theorem]{Corollary}
\theoremstyle{definition}
\newtheorem{definition}[theorem]{Definition}

\theoremstyle{remark}
\newtheorem{remark}[theorem]{Remark}

\usepackage[capitalize,noabbrev]{cleveref}

\begin{document}

\twocolumn[
\aistatstitle{Locally Private Parametric Methods for Change-Point Detection}

\aistatsauthor{ Anuj Kumar Yadav \And Cemre Cadir \And  Yanina Y. Shkel \And Michael Gastpar }

\aistatsaddress{ EPFL\\Switzerland \And  EPFL\\Switzerland \And EPFL\\Switzerland \And EPFL\\Switzerland } 

]

\begin{abstract}
We study parametric change-point detection under local differential privacy, focusing on the offline localization of a single change point in a sequence of independent observations.~In the non-private setting, we derive improved finite-sample accuracy guarantees for a generalized likelihood-ratio estimator using martingale methods.~We establish finite-sample minimax lower bounds in both the non-private and locally private settings, showing that localization accuracy is fundamentally limited by the Chernoff information between the pre- and post-change distributions.~Guided by this, we design locally private CPD algorithms based on aggregated randomized response and binary mechanisms, which maximize the Chernoff information between the induced output distributions in low- and high-privacy regimes, respectively, and derive their finite-sample guarantees. Together, our upper and lower bounds quantify the statistical cost of privacy in change-point localization.~As part of our analysis, we also show that strong data processing inequality coefficients of R\'enyi divergence and Jeffreys--R\'enyi divergence of order infinity can be computed by restricting to pairs of input distributions with a common support of at most two symbols.~We corroborate our results through empirical evaluation on synthetic data.
\end{abstract}

\section{Introduction}
\label{introduction}
\subsection{Motivation}

Consider a scenario in which a city health office is responsible for the detection of disease outbreaks and the enforcement of appropriate public health measures. To accomplish this task, they rely on daily patient-admission records provided by a city-wide hospital network. The task is to detect an outbreak by tracking hospital admissions for designated symptoms and identifying statistically significant changes in incidence. This process of identifying distributional shifts in the underlying data is commonly referred to as Change-Point Detection (CPD). CPD has diverse applications in many areas including but not limited to health care~\cite{health}, finance~\cite{finance},  monitoring in communication networks ~\cite{bell} and industrial fault detection~\cite{spc}.

In our example, the hospitals are required to share the relevant data with the city health office. However, the data may incorporate sensitive patient information. This motivates the use of privacy-preserving algorithms that enable data sharing with minimal loss in utility. Ideally, these algorithms provide formal guarantees under well-defined mathematical notions of privacy.

Differential privacy (DP)~\cite{dwork2006our, dwork2006calibrating} is a widely adopted central notion of privacy. Standard DP assumes the presence of a trusted curator who has access to the raw data and applies a privacy-preserving algorithm before releasing the output. While this assumption is appropriate in some settings, it may be undesirable when the analyst responsible for releasing the output cannot be completely trusted with sensitive data. Local differential privacy (LDP)~\cite{bassily2015local, yang2023local} strengthens the privacy guarantee by requiring each data owner to locally privatize their data prior to sharing, thereby protecting the data at the source and removing reliance on a trusted curator. Motivated by this stronger privacy model, we study change-point detection under LDP.
\subsection{Related Work}
The change-point detection (CPD) problem has been studied in both parametric~\cite{page,moustakides, quandt, xie} and non-parametric settings~\cite{petit_np, gombay}. Parametric CPD under differential privacy was first studied in~\cite{cummings_org} (see also~\cite{zhang_cpd}). The authors considered parametric CPD and employed the well-known Laplace mechanism to design differentially private algorithms, deriving performance bounds. Recently, the problem of private change-point detection has attracted significant attention. In~\cite{lau2020privacy}, the authors used maximal leakage~\cite{issa2019operational} as a privacy constraint to study CPD.~\cite{cummings2020privately} investigated differentially private CPD in the non-parametric setting, while~\cite{berrett2021locally} examined multivariate online `non-parametric' CPD under LDP. Detailed related works appear in Appendix~\ref{app:related_works}.
\subsection{Our Contributions}
We study `parametric' change-point detection in the offline setting with local differential privacy.

\noindent $\bullet$ \textbf{Non-private CPD.} We establish finite-sample minimax lower bounds for change-point localization through a generalization of Le Cam's $2$-point method followed by exponential tilting refinement to derive a tighter error exponent.~Then, we derive improved accuracy guarantees for offline GLRT estimator.~For fixed pre- and post-change distributions, our Chernoff information-based bounds show that the GLRT is asymptotically optimal up to a factor of two in the worst-case localization-error exponent.~We also derive an alternative guarantee in terms of KL and Jeffreys–Rényi divergences, facilitating comparison with earlier analyses and connecting our divergence-contraction results to the privacy analysis. Both GLRT guarantees substantially improve earlier bounds for the same estimator. 

\noindent $\bullet$ \textbf{SDPI coefficients.} We characterize strong data processing inequality (SDPI) coefficients for the $\infty$-order R\'enyi and Jeffreys--Rényi divergences. We show that their defining supremum can be restricted to input distributions supported on at most two atoms, and derive closed-form expressions for $q$-ary symmetric channels. For the Jeffreys-Rényi divergence, our results recover consequences of the Birkhoff-Hopf theorem~\cite{birkhoff} by an alternate route through the SDPI framework. We use these results to establish minimax lower bounds for private CPD. These characterizations may also be of independent interest.

\noindent $\bullet$  \textbf{Locally private CPD.}
We derive private finite-sample minimax lower bounds under $\varepsilon$-LDP, uniformly over mechanisms that may vary across observations and all estimators. We use the induced Chernoff information to guide mechanism design, motivated by its role in localization-error exponents, and develop algorithms based on aggregated randomized response (ARR) and a binary mechanism (BM). We show that ARR maximizes this criterion among all $\varepsilon$-LDP mechanisms for sufficiently large $\varepsilon$ (low privacy) while BM maximizes it for sufficiently small $\varepsilon$ (high privacy). Our finite-sample localization guarantees reproduce this qualitative ordering. Our results also quantify the performance degradation of CPD algorithms under LDP. We show that the ratio of non-private to private minimax localization-error exponents is $\Theta(\tanh^{-2}(\varepsilon/2))$ for fixed pre- and post-change distributions $(P_0,P_1)$, yielding a $\Theta(\varepsilon^{-2})$ cost under high privacy regime (small $\varepsilon>0$).

\noindent $\bullet$ \textbf{Evaluation on synthetic data.}
We run several Monte Carlo simulations on synthetic (discrete as well as continuous) data to compare the empirical localization error of our algorithms with our theoretical lower and upper bounds, with and without LDP.

Additionally, we compare the empirical performance of our LDP-CPD algorithms with central-DP CPD algorithm of~\cite{cummings_org,zhang_cpd} in Appendix~\ref{app:dp_ldp_cpd} to contextualize the cost of `local' privacy.~In Appendix~\ref{app:testing}, we construct a testing-based benchmark for locally private CPD by combining the reduction of Canonne et al.~\cite{canonne2019structure} with the LDP testing results of Pensia et al.~\cite{pensia}, and compare its guarantees with those of our GLRT-based algorithms.

\section{Background}
\begin{figure*}[t]
  \begin{center}
    \begin{tikzpicture}[
scale = 0.8,
transform shape,
  font=\normalsize,
  box/.style={draw, thick, minimum width=2cm, minimum height=1cm, align=center},
  arr/.style={
    -{Latex[length=1mm, width=0.8mm]},
    thick,
    shorten <=0pt,
    shorten >=0pt
  },
  lab/.style={font=\normalsize, inner sep=1pt},
  node distance=10mm
]

\node[box] (D) at (0,0) {$(\mathcal{D},P_0,P_1)$};

\def\ldpgap{5mm}
\node[box] (m1) at (3.8, 1.53) {\textcolor{purple}{$\varepsilon$ - LDP}};
\node[box, below=\ldpgap of m1] (m2) {\textcolor{purple}{$\varepsilon$ - LDP}};
\node[box, below=7mm of m2] (m4) {\textcolor{purple}{$\varepsilon$ - LDP}};

\draw[very thick, dashed] (m2.south) -- (m4.north);

\node[box] (Dt) at (7.8,0) {$(\widetilde{\mathcal{D}},Q_0,Q_1)$};

\node[draw, very thick, minimum width=2.5cm, minimum height=1cm, align=center] (DA) at (11.3,0)
{\textcolor{red}{Data Analyst}};
\draw[arr] (DA.east) -- ++(1.5,0)
  node[midway, above] {$\hat{k}\in \mathcal{K}\,\cup\,\{\infty\}$};


\coordinate (s1) at ($(D.north east)!0.20!(D.south east)$);
\coordinate (s2) at ($(D.north east)!0.50!(D.south east)$);
\coordinate (s4) at ($(D.north east)!0.80!(D.south east)$);

\coordinate (t1) at ($(Dt.north west)!0.20!(Dt.south west)$);
\coordinate (t2) at ($(Dt.north west)!0.50!(Dt.south west)$);
\coordinate (t4) at ($(Dt.north west)!0.80!(Dt.south west)$);

\draw[arr] (s1) -- (m1.west) node[midway, sloped, above, lab] {$x_1$};
\draw[arr] (s2) -- (m2.west) node[midway, sloped, above, lab] {$x_2$};
\draw[arr] (s4) -- (m4.west) node[midway, sloped, above, lab] {$x_n$};

\draw[arr] (m1.east) -- (t1) node[midway, sloped, above, lab] {$y_1$};
\draw[arr] (m2.east) -- (t2) node[midway, sloped, above, lab] {$y_2$};
\draw[arr] (m4.east) -- (t4) node[midway, sloped, above, lab] {$y_n$};

\draw[arr] (Dt.east) -- (DA.west);

\end{tikzpicture}
    \caption{General framework for parametric change-point detection under local differential privacy. 
    }
    \label{fig:framework}
  \end{center}
  \vspace*{-5mm}
\end{figure*}
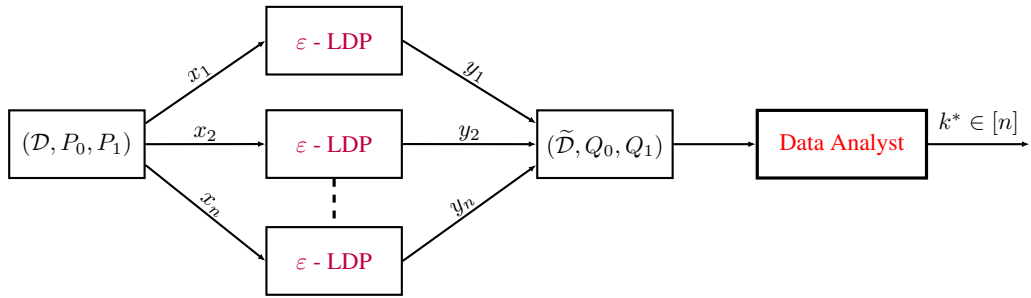    
\subsection{Notations and Preliminaries}
Detailed notations appear in Appendix~\ref{app:notations}.~We measure accuracy of our CPD algorithms  through the worst-case probability that the additive localization-error exceeds a given tolerance, as follows,
\begin{definition}[$(\alpha,\beta)-${accuracy}]
   Let $k^{\star}$ denote the true change-point of a dataset or a stream of data points. A change-point detection algorithm which outputs $\widehat{k}$ as an estimate of the change-point, is said to be $(\alpha,\beta)-$accurate if $\sup_{k^{\star} \in \mathcal{K}}\mathbb{P}_{k^{\star}}\Big(|\widehat{k}-k^{\star}| >\alpha \Big) \leq \beta$. Here, $\alpha \in [n]$ denotes the tolerance while $\beta$ denotes the upper bound on error probability for the given $\alpha$.  
\end{definition}
\begin{definition}[Chernoff Information]\label{def:ich}
    Let $P_0$ and $P_1$ be two probability mass functions (PMFs) on $\mathcal{X}$. The Chernoff Information $\mathrm{I}_{\,\mathrm{ch}}$ between $P_0$ and $P_1$ is 
    \begin{align}
        \mathrm{I}_{\,\mathrm{ch}}{(P_0,P_1)} = -\inf_{\lambda \in [0,1]}\!\log\left(\sum_{x \in \mathcal{X}} P_0(x)^{\lambda}P_1(x)^{1-\lambda}\right).\notag
    \end{align}
\end{definition}
\begin{definition}[$\infty-${Jeffreys-Rényi Divergences}]
   Let $P_0$ and $P_1$ be two PMFs on $\mathcal{X}$. The Rényi divergence of order $\rho = \infty$, between $P_0$ and $P_1$ is defined as $\mathrm{D}_{\infty}(P_0\|P_1)= \log\left(\max_{x \in \mathcal{X}}\frac{P_0(x)}{P_1(x)}\right)$. The {Jeffreys-Rényi divergence} of order infinity, is defined as $\mathrm{D}^{\mathrm{J}}_{\infty}(P_0,P_1)= \mathrm{D}_{\infty}(P_0\|P_1) +  \mathrm{D}_{\infty}(P_1\|P_0).$
\end{definition}

\subsection{Problem setting: Parametric offline CPD}

We study offline parametric CPD with known pre- and post-change distributions $(P_0,P_1)$.~This is the fully specified case of a parametric model: more generally, parametric models may contain unknown distributional parameters. On the other hand, non-parametric formulations do not restrict the distributions to specified finite-dimensional families~\cite{kernel,mmd}.~Our goal is to localize a change-point retrospectively from a batch of observations (offline), in contrast to online CPD, where decisions are made sequentially as observations arrive.~We treat the true change-point as a fixed unknown parameter and impose no prior distribution on it, thereby adopting a non-Bayesian framework.

The offline parametric CPD problem is defined by a dataset $\mathcal{D}= (x_1, \ldots, x_n)$ of size $n >1$, fully supported pre-and post-change distributions, $(P_0,P_1)$, on $\mathcal{X}$.~Let $\mathcal{K}=\{2,\ldots n\}$.~The data points in $\mathcal{D}$ are initially sampled i.i.d. from $P_0$; at some unknown time $k^{\star} \in \mathcal{K}$,  the underlying distribution changes to $P_1$.~The data analyst has access to $(\mathcal{D},P_0,P_1)$.~The goal of the analyst is to output an estimate of the change-point, denoted by $\widehat{k} \in \mathcal{K} \cup \{\infty\}$, where $\widehat{k} = \infty$ indicates that no change-point is detected.~We assume that $\mathcal{D}$ contains `exactly one' true change-point in $\mathcal{K}$.

\subsection{CPD with Local Differential Privacy}

LDP is a well-studied privacy framework under which data owners benefit from plausible deniability. Since it is a `local' privacy notion, there is no reliance on a trusted third party. 

\begin{definition}[$\varepsilon-${Local Differential Privacy (LDP)}]\label{def:ldp}
A randomized algorithm (mechanism) $W^{\varepsilon}_{Y|X} : \mathcal{X} \rightarrow \mathcal{Y}$ is $\varepsilon-$locally differentially private (or $\varepsilon-$LDP), if for any pair $(x, x') \in \mathcal{X} \times \mathcal{X}$ and any subset $\mathcal{A} \subseteq \mathcal{Y}$, we have $W^{\varepsilon}_{Y|X}(\mathcal{A}|x) \leq e^\varepsilon W^{\varepsilon}_{Y|X}(\mathcal{A}|x').$

\end{definition}

Let $(\mathcal{D},P_0,P_1)$ denote the true dataset, the pre- and the post-change distributions at the source. In the LDP-CPD framework, each data point $x_i$ in $\mathcal{D}$ is privatized independently at the source via an $\varepsilon-$LDP mechanism $W^{\varepsilon}$. Let $\widetilde{\mathcal{D}} $ denote the privatized dataset which contains the noisy data points $y_i$ for all $i \in [n]$.~Let $Q_0$ and $Q_1$ denote the distributions induced through the $\varepsilon-$LDP mechanism $W^{\varepsilon}$, corresponding to $P_0$ and $P_1$, respectively. The source then sends the tuple $(\widetilde{\mathcal{D}},Q_0,Q_1)$ to the data analyst. The analyst with access to the privatized dataset and induced pre- and post- change distributions $(Q_0,Q_1)$, estimates the change-point of the true dataset $\mathcal{D}$. (see Figure~\ref{fig:framework})
 
\section{Non-Private CPD}\label{sec:npcpd}
\subsection{Finite sample minimax lower bounds}
\begin{theorem}\label{thm:minimax_lb}
Let $(P_0,\,P_1)$ be distinct distributions on $\mathcal X$. Define
\begin{align}
  \lambda_\star\in\operatorname*{arg\,min}_{\lambda\in[0,1]}\ \sum_{x\in\mathcal{X}}P_0(x)^{\lambda}P_1(x)^{1-\lambda},
\end{align}
such that $\mathrm{I}_{\,\mathrm{ch}}(P_0,P_1)=-\log\sum_{x}P_0(x)^{\lambda_\star}P_1(x)^{1-\lambda_\star}$.
Let $P_\star$ be the Chernoff-tilted distribution
\begin{align*}
P_\star(x)\triangleq \frac{P_0(x)^{\lambda_\star}P_1(x)^{1-\lambda_\star}}{\sum_{x'\in\mathcal{X}}P_0(x')^{\lambda_\star}P_1(x')^{1-\lambda_\star}},
\end{align*}
and define the tilted second-moment by
\begin{align}
\Sigma^2=\mathbb{E}_{P_\star}\!\left[\Big(\log\frac{P_0(X)}{P_1(X)}\Big)^{2}\right].\label{eq:sigma}
\end{align}
For any integer $\alpha\ge0$,~let
$M_\alpha
=
1+\left\lfloor\frac{n-2}{2\alpha+1}\right\rfloor$.~If $M_\alpha\ge2$, then
\begin{align}
E_{\mathcal{K}}(\alpha)
&\triangleq
\inf_{\widehat{k}}\ \sup_{k^\star\in\mathcal{K}}
\mathbb{P}_{k^\star}\!\left(
|\widehat{k}-k^\star|>\alpha
\right)
\notag\\
&\geq
\left(1-\frac{1}{M_\alpha}\right)
e^{
-(2\alpha+1)\mathrm{I}_{\,\mathrm{ch}}(P_0,P_1)
-\frac{\Sigma}{2}\sqrt{2\alpha+1}},
\label{eq:thm-lb}
\end{align}
where the infimum is over all (possibly randomized) CPD estimators, including those that may declare no change. If $M_\alpha<2$, i.e.\ $2\alpha+1\ge n-1$, then $E_{\mathcal{K}}(\alpha)=0$.
\end{theorem}
\emph{Proof sketch.}~The proof appears in Appendix~\ref{app:minimax_lb}.~It is based on Le Cam's method with $M_\alpha$ hypotheses and then lower-bounds the resulting testing error by applying exponential tilting to $P_\star$, which yields the refined exponent.~The trailing exponent $\frac{\Sigma}{2}\sqrt{2\alpha+1}$ captures the fluctuations of the log-likelihood ratio around its mean, under $P_\star$. A simpler minimax lower bound is also presented in Cor.~\ref{cor:minimax_np}.
\subsection{GLRT algorithm: finite sample guarantees}
Inspired by the CUSUM approach~\cite{page}, our analysis is based on the GLRT described in Algorithm~\ref{alg:offline_cpd}.
    \vspace*{-2mm}
{
\setlength{\textfloatsep}{10mm}
\begin{algorithm}[H]
    \caption{Offline CPD ($\mathcal{D},\mathcal{X},P_0,P_1$)}
    \label{alg:offline_cpd}
    \begin{algorithmic}[1] 
        \STATE \textbf{Input:} Dataset $\mathcal{D}$, Alphabet $\mathcal{X}$, PMFs $(P_0, P_1)$
        \FOR{$k=1$ \textbf{to} $n$}
            \STATE Compute ratio $l(\mathcal{D},k) = \sum_{i=k}^{n}\log\frac{P_1(x_i)}{P_0(x_i)}$
        \ENDFOR
        \STATE \textbf{Output:} $\widehat{k} =\min (\arg\max_{k \in \mathcal{K}} l(\mathcal{D},k))$.
    \end{algorithmic}
\end{algorithm}
}
    \vspace*{-4mm}
 Now, we state our finite-sample accuracy bounds for the offline change-point estimator in Algorithm~\ref{alg:offline_cpd}.

\begin{theorem}\label{thm:npcpd}
    Let $k^{\star} \in \mathcal{K}$ be the true change-point of the dataset $\mathcal{D} \in \mathcal{X}^n$, with data points drawn from the distributions $P_0$, $P_1$ on $\mathcal{X}$. Then, the CPD estimator in Alg.~\ref{alg:offline_cpd} is $(\alpha,\beta)$-accurate for any $\alpha \in [n]$ and some
\begin{align}
  \beta \leq \,\,&2\, e^{-(\alpha+1) \mathrm{I}_{\,\mathrm{ch}}(P_0,P_1)}\left(1-\frac{e^{-(\alpha+1) \mathrm{I}_{\,\mathrm{ch}}(P_0,P_1)}}{2}\right). \label{eq:npcpd} 
\end{align}
\end{theorem}
We also present a guarantee in terms of the minimum KL divergence $C$ and the sensitivity of the log-likelihood ratio $S=\mathrm{D}^{\mathrm{J}}_{\infty}(\cdot,\cdot)$. This formulation also facilitates direct comparison with existing analyses~\cite{zhang_cpd}.
\begin{lemma}\label{cor:npcpd}
Let $S<\infty$ denote the sensitivity of the log-likelihood ratios w.r.t $(P_0,P_1)$ i.e., $S = \max_{x \in \mathcal{X}}\log\frac{P_1(x)}{P_0(x)} - \min_{x \in \mathcal{X}}\log\frac{P_1(x)}{P_0(x)} =\mathrm{D}^{\mathrm{J}}_{\infty}(P_0,P_1)$,
and $C=\min\{\mathrm{D}_{\mathrm{KL}}(P_0\|P_1), \mathrm{D}_{\mathrm{KL}}(P_1\|P_0) \}$. Then, under the assumptions of Thm.~\ref{thm:npcpd}, the estimator in Alg.~\ref{alg:offline_cpd} is $(\alpha,\beta)$-accurate for any $\alpha \in [n]$ and some
\begin{align}
\beta \leq
2\sum_{i=1}^{i^\star}
e^{
-\dfrac{2^{i-1}\alpha C^2}{S^2}}
\end{align}
where
$i^\star=
\left\lceil
\log_2\!\left(\frac{n-1}{\alpha}\right)
\right\rceil$.
\end{lemma}
\emph{Proof sketch.} Both proofs appear in Appendix~\ref{app:npcpd}.~We first bound the probability that a candidate more than $\alpha$ away from $k^\star$ attains a log-likelihood score at least as large as that of $k^\star$.~Then, we apply exponential martingale bounds for Thm.~~\ref{thm:npcpd} and maximal Azuma–Hoeffding inequalities for Lemma~\ref{cor:npcpd}.
\begin{remark}
For every $\alpha < n-1$, the guarantee in Thm.~\ref{thm:npcpd} is at least as sharp as the guarantee in Lemma~\ref{cor:npcpd}.~Our bound in Thm.~\ref{thm:npcpd} significantly improves on the corresponding guarantees of~\cite{zhang_cpd} for the same estimator in Algorithm~\ref{alg:offline_cpd}. See Appendix~\ref{app:npcomp} for detailed comparison.~Our results also extend to continuous distributions with positive densities on a common support with sums replaced by integrals, with Lemma~\ref{cor:npcpd} additionally requiring a bounded log-likelihood ratio.
\end{remark}
Together, Thm.~\ref{thm:minimax_lb}, Cor.~\ref{cor:minimax_np}, and Thm.~\ref{thm:npcpd} show that Chernoff information governs the minimax error exponent up to constant factors. For $n \geq 2\alpha+3$,
\begin{align*}
(\alpha+1) \mathrm{I}_{\,\mathrm{ch}} - \log 2
&\leq -\log E_{\mathcal{K}}(\alpha)\notag\\
&\leq \min\Big(2(2\alpha+1)\mathrm{I}_{\,\mathrm{ch}}+\log 4,\\ 
&\hspace*{4mm}(2\alpha+1) \mathrm{I}_{\,\mathrm{ch}} + \tfrac{\Sigma}{2}\sqrt{2\alpha+1} + \log 2\Big).
\end{align*}
For finite $\alpha$, these bounds place $-\log E_{\mathcal K}(\alpha)$ between affine functions of $\alpha$ with slopes $\mathrm I_{\mathrm{ch}}$ and $4\mathrm{I}_{\,\mathrm{ch}}$.~For fixed distinct $(P_0,P_1)$, the second term improves the asymptotic upper coefficient to $2\mathrm{I}_{\,\mathrm{ch}}$ as $\alpha\to\infty$ with $n\geq2\alpha+3$, establishing optimality of Algorithm~\ref{alg:offline_cpd} up to a factor of two in the leading error exponent.
\section{Private CPD: minimax lower bounds}~\label{sec:ldpcpd_lb}
An $\varepsilon$-LDP mechanism is essentially a noisy channel which privatizes the data points according to the conditional law $W^{\varepsilon}_{Y|X}$.~To derive minimax lower bounds and analyze our private CPD algorithms we quantify the information contraction of Jeffreys-Rényi divergence of order infinity, through an $\varepsilon$-LDP mechanism $W^{\varepsilon}_{Y|X}$, below.~We now restrict ourselves to finite $\mathcal{X}$ and $\mathcal{Y}$.
\subsection{SDPI coefficients} \label{subsec:sdpi}
Consider a channel $W_{Y|X} : \mathcal{X} \rightarrow \mathcal{Y}$. For any two input distributions $(P_0,P_1)$ in $\Delta(\mathcal{X})$, let $Q_0$ and $Q_1$ denote the distributions induced in $\Delta(\mathcal{Y})$ by passing $P_0$ and $P_1$ through $W_{Y|X}$, respectively, i.e.,
$Q_i(y) \;=\; \sum_{x\in\mathcal{X}} P_i(x)\,W_{Y|X}(y|x)$, for $i\in\{0,1\}$.

The data processing inequality (DPI)~\cite{cover-thomas} is a fundamental result which states that processing cannot increase the statistical distinguishability between distributions. Specifically, for any divergence $D$ that is monotone under stochastic transformations, we have $D(Q_0,Q_1) \leq D(P_0,P_1)$.
Thus, downstream processing of random variables can only bring them `closer' in the space of distributions. Over the past decade, significant progress has been made in strengthening the DPI by quantitatively characterizing the contraction of statistical divergences through the SDPI coefficients.~SDPI coefficients have a wide applicability across numerous areas ranging from portfolio theory~\cite{ekrip}, information percolation~\cite{eks}, convergence and mixing of Markov chains~\cite{raginsky_markov}, minimax lower bounds in  estimation~\cite{duchi_estimation,raginsky_estimation} and privacy~\cite{duchi_privacy}, among several others.
\begin{definition}[SDPI coefficient]\label{def:sdpi} The {SDPI coefficient} of a channel $W_{Y|X}$ under a divergence $D$ i.e., ${\eta}_{{D}}(W)$,  is defined as
\begin{align}\label{eq:cc}
    \eta_{D}(W) \triangleq \sup_{P_0,P_1 \in \Delta(\mathcal{X})} \frac{D(Q_0,Q_1)}{D(P_0,P_1)}.
\end{align}
\end{definition}
It was shown in~\cite{poly} that for the class of ${f-}$divergences, the supremum in Eq.~\ref{eq:cc} can be restricted to pairs of binary distributions (supported on at most two atoms). Here, we prove a similar result for the $\infty$-order Rényi divergences and Jeffreys-Rényi divergences.~Due to space constraints, we present results on Rényi divergences in Appendix~\ref{app:rd_sdpi} as they aren't directly used to establish results on private CPD in next sections.

\begin{theorem}\label{thm:jrdccbd}
For Jeffreys-Rényi divergence of order $\rho = \infty$ i.e., $\mathrm{D}^{\mathrm{J}}_{\infty}(\cdot,\cdot)$, the SDPI coefficient of channel $W$, $\eta^{\mathrm J}_{\infty}(W)$, is approached by binary input distributions.
\end{theorem}
The proof is deferred to Appendix~\ref{section:cc-jrd-proof}. Next, we explicitly derive the corresponding SDPI coefficient for symmetric channels, using Thm.~\ref{thm:jrdccbd}. 

\begin{theorem}\label{thm:jrdccbd-value}
    Let $W$ be a $q-$ary symmetric channel (see Definition~\ref{def:emk}).~The SDPI coefficient of $W$ for Jeffreys-Rényi divergence of order infinity is $\eta^{\mathrm J}_{\infty}(W) = \frac{|v-u|}{v+u}$.~It is approached by distributions converging to uniform distribution on a common two-point support.
\end{theorem}
The proof appears in Appendix~\ref{section:cc-jrd-value-proof}.~The above expression is the specialization of Birkhoff's contraction coefficient resulting from the Birkhoff-Hopf theorem~\cite{birkhoff,seneta}, which studies contraction of positive linear operators on cones. In the next lemma,  we show that the Jeffreys-Rényi divergence of order infinity between the output distributions of an $\varepsilon$-LDP mechanism cannot be too large. 
\begin{lemma}\label{cor:ldpp}
Let $(Q_0,Q_1)$ be the output distributions induced by the input distributions $(P_0,P_1)$ through an $\varepsilon$-LDP mechanism $W^{\varepsilon}$. Then,
\begin{align}
\mathrm{D}^{\mathrm{J}}_{\infty}(Q_0,Q_1) \leq \min\Big( 2\varepsilon \hspace{1mm};\eta^{\mathrm J}_{\infty}(W^{\varepsilon})\mathrm{D}^{\mathrm{J}}_{\infty}(P_0,P_1)\Big),
\end{align}
where $\eta^{\mathrm J}_{\infty}(W^{\varepsilon})$ is the SDPI coefficient of $W^{\epsilon}$ for Jeffreys-Rényi divergence of order infinity.
\end{lemma}
The proof appears in Appendix~\ref{section:jrdub-proof}.~Now, we can derive minimax lower bounds for CPD under $\varepsilon$-LDP using the above results on SDPI coefficients.
\subsection{Minimax lower bounds under $\varepsilon$-LDP}\label{subsec:ldpcpd_lb}
\begin{theorem}
\label{thm:private-minimax-lb}
Let $(P_0,P_1)$ be distinct PMFs on $\mathcal{X}$. Let $\varepsilon \in (0,\infty)$, and $\mathcal W_\varepsilon$ denote the class of non-interactive $\varepsilon$-LDP mechanisms.

Let $I=\mathrm{I}_{\,\mathrm{ch}}(P_0,P_1)$, $S=\mathrm{D}^{\mathrm{J}}_\infty(P_0,P_1)$, $\delta=\mathrm{d}_{\mathrm{TV}}(P_0,P_1)$, and $ c_{\varepsilon,2}=\tanh(\frac{\varepsilon}{2})$.~For any integer $\alpha\ge0$,~let
$M_\alpha
=
1+\left\lfloor\frac{n-2}{2\alpha+1}\right\rfloor$.~If $M_\alpha\ge2$, then
\begin{align}
\label{eq:private-minimax-lb}
E_{\mathcal K}^{\varepsilon}(\alpha)
&\triangleq
\inf_{\boldsymbol W^{\varepsilon}\in\mathcal W_\varepsilon^n}
\inf_{\widehat k}
\sup_{k^\star\in\mathcal K}
\mathbb P_{k^\star}^{\boldsymbol W^{\varepsilon}}
\left(
|\widehat k-k^\star|>\alpha
\right)\notag\\
&\geq
\left(1-\frac{1}{M_\alpha}\right)
e^{
-(2\alpha+1)\overline I_{\varepsilon}
-\tfrac{\overline S_\varepsilon}{2}\sqrt{2\alpha+1}}.
\end{align}
where
\[
\overline S_\varepsilon=\min\bigl(
\log(1+\delta (e^\varepsilon-1)),\,
2\operatorname{artanh}(c_{\varepsilon,2}\tanh\tfrac{S}{4})
\bigr),
\]
and
\[
\overline I_\varepsilon=\min\bigl(
-\log(1-c_{\varepsilon,2}^2(1-e^{-I})),\,
\log\cosh\tfrac{\overline S_\varepsilon}{2}
\bigr).
\]

For a collection
$\boldsymbol W^{\varepsilon}=(W^{\varepsilon}_1,\ldots,W^{\varepsilon}_n)\in\mathcal W_\varepsilon^n$,
the observations are privatized independently according to
$
\mathbb P(Y^n=y^n|X^n=x^n)
=
\prod_{i=1}^n W^{\varepsilon}_i(y_i|x_i)
$, where the mechanisms are chosen independently of the data and $k^{\star}$.~The inner infimum ranges over all possibly randomized estimators $(\widehat k)$ based on the privatized observations, including those that may
declare no change.
\end{theorem}

\emph{Proof sketch.}~The quantities $\overline S_\varepsilon$ and $\bar I_\varepsilon$ uniformly upper bound one-half the Jeffreys--R\'enyi divergence of order infinity and the Chernoff information, respectively, between the output distributions induced by any $\varepsilon$-LDP channel. The proof (Appendix~\ref{app:private_minimax_lb}) extends the $M_\alpha$-point testing argument of Thm.~\ref{thm:minimax_lb} to independent but not necessarily identically distributed privatized observations, bounding the overlap between each pair of adjacent hypotheses separately by exponential tilting on the block of $2\alpha+1$ observations where they differ. The Chernoff information and tilted log-likelihood variance of this block are at most sums of per-observation terms, which we bound uniformly over $\varepsilon$-LDP
channels: the variance via the log-likelihood-ratio range, by refining the arguments of Lemma~\ref{cor:ldpp} and Thm.~\ref{thm:jrdccbd}, and the Chernoff information via an SDPI combined with this range bound. Since the resulting bound does not depend on the mechanisms, it also holds for the infimum over them.

\begin{corollary}\label{cor:minimax_lb_private}
Under the assumptions of Thm.~\ref{thm:private-minimax-lb}, and with $I$, $S$ and $c_{\varepsilon,2}$ as defined there, we have the following simpler lower bound
\[
E^{\varepsilon}_{\mathcal{K}}(\alpha)\;\ge\;\Bigl(1-\frac{1}{M_\alpha}\Bigr)\,
e^{-(2\alpha+1)c_{\varepsilon,2}^{2} I-\tfrac{\min(\varepsilon,\,S/2)}{2}\sqrt{2\alpha+1}}.
\]
\end{corollary}
Thus, the leading error exponent in the lower bound is $(2\alpha+1)\,c_{\varepsilon,2}^2 I$ under $\varepsilon$-LDP versus $(2\alpha+1)I$ without privacy (Thm.~\ref{thm:minimax_lb}).~That is, the local privacy scales it by $c_{\varepsilon,2}^2 \approx \varepsilon^2/4$ for small
$\varepsilon$.
\section{Private CPD: finite sample guarantees}~\label{sec:ldpcpd_ub}
\vspace*{-7mm}
\subsection{Chernoff information as a design criterion} \label{sec:ch-proxy}
Thms.~\ref{thm:minimax_lb},~\ref{thm:npcpd}, and Cor.~\ref{cor:minimax_np} show that Chernoff information governs the asymptotic minimax error exponent up to a factor of two. A fixed $\varepsilon$-LDP mechanism applied independently to each data point turns the problem into a parametric CPD  with the induced pre- and post-change distributions $(Q_0,Q_1)$, so both results apply to $(Q_0,Q_1)$. The induced Chernoff information is therefore more than a proxy suggested by an upper bound: for a fixed mechanism, it characterizes the best achievable leading error exponent up to a factor of two. Thus, it makes maximizing $\mathrm{I}_{\,\mathrm{ch}}(Q_0,Q_1)$ the natural criterion for choosing the privatization mechanism. We identify $\varepsilon$-LDP mechanisms that maximize the induced $\mathrm{I}_{\,\mathrm{ch}}$ under low- and high-privacy (large and small $\varepsilon$). 

\begin{definition}[Aggregated randomized response $\mathrm{ARR}$]\label{def:arr}
Let $P_0,P_1$ be PMFs on $\mathcal X$.~Partition
$\mathcal{X}$ into groups of symbols with the same likelihood ratio $P_0(x)/P_1(x)$. Let $\mathcal{X}_1\subseteq\mathcal{X}$ contain one symbol from each group, let $c(x)\in\mathcal{X}_1$ be the symbol chosen from the group of $x$, and let $r=|\mathcal{X}_1|$ be the number of groups. For $\varepsilon>0$, aggregated randomized response $\mathrm{(ARR)}$ is an $\varepsilon$-LDP mechanism that applies $r$-ary randomized response to $c(x)$:
\[
W^{\varepsilon}_{\mathrm{agg}}(y|x)=
\begin{cases}
\dfrac{e^{\varepsilon}}{e^{\varepsilon}+r-1}, & y=c(x),\\[2mm]
\dfrac{1}{e^{\varepsilon}+r-1}, & y\neq c(x),
\end{cases}
\qquad y\in\mathcal{X}_1.
\]
Its output PMFs are
\[
Q_\nu(y)=\frac{1+(e^{\varepsilon}-1)\widetilde{P}_\nu(y)}{e^{\varepsilon}+r-1},
\qquad \nu\in\{0,1\},
\]
where $\widetilde{P}_\nu(y)=\sum_{x:\,c(x)=y}P_\nu(x)$. If all likelihood ratios are distinct, $W^{\varepsilon}_{\mathrm{agg}}$ reduces to the randomized response.
\end{definition}

\begin{theorem}\label{thm:high_eps_div}    
Let $(P_0, P_1)$ be any two PMFs. Then, there exists a $\varepsilon^{*}(P_0,P_1)<\infty$, such that for all $\varepsilon \geq \varepsilon^*$, the $\mathrm{ARR}$ mechanism maximizes the Chernoff information $\mathrm{I}_{\,\mathrm{ch}}{(Q_0,Q_1)}$ between the induced output distributions $(Q_0,Q_1)$, among all $\varepsilon$-LDP mechanisms.
\end{theorem}
In~\cite[Thm.~$5$]{kairouz2014extremal}, the authors state that for any $(P_0,P_1)$, the binary mechanism with $\tau=1$ (cf.~Def.~\ref{def:bm}) maximizes any $f$-divergence between $(Q_0,Q_1)$ for sufficiently small $\varepsilon$.~Their proof uses a strict inequality that fails when $P_0(x)=P_1(x)$ for some $x \in \mathcal{X}$. Using the linear-programming characterization in~\cite[Thm.~$4$]{kairouz2014extremal}, we handle this case and establish the optimality of a binary mechanism for $f_\lambda(u)=1-u^\lambda$, with a privacy threshold uniform over $\lambda\in(0,1)$, yielding the following corollary.
\begin{corollary}\label{cor:small-large-eps}
Let $(P_0,P_1)$ be any two PMFs. There exists a threshold $\varepsilon'(P_0,P_1)>0$, such that for all $\varepsilon \leq\varepsilon'$, a binary mechanism (with appropriately chosen $\tau$) maximizes the Chernoff information
$\mathrm{I}_{\,\mathrm{ch}}{(Q_0,Q_1)}$ over all $\varepsilon$-LDP mechanisms.
\end{corollary}
The proofs are deferred to Appendices~\ref{app:high_eps_div} and~\ref{app:small-large-eps}. With the induced Chernoff information $\mathrm{I}_{\,\mathrm{ch}}(Q_0,Q_1)$ as a design criterion, Thm.~\ref{thm:high_eps_div} and Cor.~\ref{cor:small-large-eps}  yield a mechanism-level ordering in terms of Chernoff information contraction. In particular, the binary mechanism incurs no greater contraction of Chernoff information than aggregated randomized response in the high-privacy regime, while the converse holds in the low-privacy regime. Later, we demonstrate that this mechanism-level ordering is reflected in our finite-sample guarantees as well as experiments. This motivates our study of CPD algorithms based on aggregated randomized response and binary mechanisms.

\subsection{Private CPD with Aggregated RR}
The $\mathrm{ARR}$ mechanism is an $\varepsilon$-LDP mechanism which randomizes the symbols over the aggregated alphabet.~The Algorithm~\ref{alg:offline_rrcpd} privatizes the data points independently through the $\varepsilon$-$\mathrm{{ARR}}$ mechanism and estimates the change-point by applying Algorithm~\ref{alg:offline_cpd} on the privatized dataset $\widetilde{\mathcal D}$.

\begin{algorithm}[htb]
    \caption{Offline $\mathrm{ARR-CPD}$ (\ensuremath{\mathcal{D}, \mathcal{X}, P_{0}, P_{1}, \varepsilon})}
    \label{alg:offline_rrcpd}
    \begin{algorithmic}[1]
        \STATE \textbf{Input:} Dataset $\mathcal{D}$, $\mathcal{X}$, PMFs $(P_0, P_1)$, $\varepsilon>0$.
        \STATE $(P_0,P_1)$ through the $\varepsilon$-$\mathrm{ARR}$ mechanism induce distributions $(Q_0,Q_1)$ on $\mathcal{Y}= \mathcal{X}_1$, respectively.
        \STATE For all $i \in [n]$, pass $x_i$ in $\mathcal{D}$ through the $\varepsilon$-$\mathrm{ARR}$ mechanism to get $y_i$ s.t. $\widetilde{\mathcal D}=(y_1,\ldots,y_n)$.
        \FOR{$k=1$ \textbf{to} $n$}
            \STATE Compute the ratio $l(\widetilde{\mathcal{D}},k) = \sum_{i=k}^{n}\log\dfrac{Q_1(y_i)}{Q_0(y_i)}$
        \ENDFOR
        \STATE \textbf{Output:} $\widehat{k} = \min(\arg\max_{k \in \mathcal{K}} l(\widetilde{\mathcal{D}},k))$.
    \end{algorithmic}
\end{algorithm}
Next, we provide finite-sample accuracy guarantees for Algorithms~\ref{alg:offline_rrcpd}, for accurately estimating the change-point  while satisfying $\varepsilon$-LDP. First, we define below some shorthand notations that will be useful to state our results in Thms.~\ref{thm:rrcpd} and~\ref{thm:bmcpd}.

Let $\delta=\mathrm{d}_{\mathrm{TV}}(P_0,P_1)$ and $D=\mathrm{H}^2(P_0,P_1)$ (squared-Hellinger distance). For a positive integer $m$, define
\begin{align}
c_{\varepsilon,m}
&=\frac{e^\varepsilon-1}{e^\varepsilon+m-1},\label{eq:not1}\\
L_{\varepsilon,m}(d)
&=\frac{\bigl(\sqrt{m+2(e^\varepsilon-1)d}-\sqrt m\bigr)^2}
{2(e^\varepsilon+m-1)}\label{eq:not2}.
\end{align}
\begin{theorem}\label{thm:rrcpd}
 Let $k^*\in\mathcal K$ be the true change point of the dataset $\mathcal D\in\mathcal X^n$, with data points drawn from distributions $(P_0,P_1)$ on $\mathcal X$. Define
\[
\rho_{\mathrm{ARR}}
=
\min\left(
\sqrt{1-c_{\varepsilon,r}^{\,2}\delta^2},
\,1-L_{\varepsilon,r}(D)
\right),
\]
where $r$ is the number of groups in $\varepsilon$-$\mathrm{ARR}$. Then, the $\mathrm{ARR-CPD}$ estimator in Algorithm~\ref{alg:offline_rrcpd} satisfies $\varepsilon$-LDP and is $(\alpha,\beta_{\mathrm{ARR}})$-accurate for every $\alpha\in[n]$ and some
\begin{align}
\beta_{\mathrm{ARR}}
\le
2\rho_{\mathrm{ARR}}^{\,\alpha+1}
\left(
1-\frac{\rho_{\mathrm{ARR}}^{\,\alpha+1}}{2}
\right).
\end{align}
where $c_{\varepsilon,r}$ and $L_{\varepsilon,r}$ are defined in~\eqref{eq:not1} and~\eqref{eq:not2}.
\end{theorem}
The proof is deferred to the Appendix~\ref{app:rrcpd}. We also have the following simpler upper bound,
\begin{corollary}
Under the assumptions of Thm.~\ref{thm:rrcpd}, the $\mathrm{ARR-CPD}$ estimator is $(\alpha,\beta_{\mathrm{ARR}})$-accurate for every $\alpha\in[n]$ and some
\[
\beta_{\mathrm{ARR}}
\le
2e^{-\frac{\alpha+1}{2}\delta^2c_{\varepsilon,r}^{\,2}}
\left(
1-\frac{
e^{-\frac{\alpha+1}{2}\delta^2c_{\varepsilon,r}^{\,2}}
}{2}
\right).
\]
\end{corollary}

\subsection{Private CPD with Binary Mechanism}
The Binary Mechanism ($\mathrm{BM}$) first partitions the space of all the data points into two bins, followed by a randomization over the resulting binary alphabet.

\begin{definition}[Binary Mechanism]\label{def:bm} For $\tau>0$, the binary mechanism $W^b_{\tau,Y|X}$ generates an output $Y\in\{0,1\}$ from an input $X\in\mathcal{X}$ via the following two steps:

\noindent\textit{(i) Quantization.} The quantizer $Z_{\tau}:\mathcal{X} \rightarrow \{0,1\}$ is a deterministic mapping parameterized by $\tau>0$ such that it partitions $\mathcal{X}$ into $\mathcal{S}_{\tau}$ and $\mathcal{S}^{c}_{\tau}$ as follows: $\mathcal{S}_{\tau} :=\{x \in \mathcal{X}: P_0(x)\geq \tau P_1(x)\}$.
The elements in $\mathcal{S}_{\tau}$ are mapped to bit $0$, while the elements in  $\mathcal{S}^{c}_{\tau}$ are mapped to bit $1$. Let $V=Z_{\tau}(X)$ where $V \in \{0,1\}$.

\noindent\textit{(ii) Privatization (binary randomized response).} Given $V=v$, sample $Y\in\{0,1\}$ according to
\begin{align}
    W_{Y|V}(y|v)=
\begin{cases}
\frac{e^{\varepsilon}}{e^{\varepsilon}+1}, & \text{if } y=v \vspace*{2mm}\\
\frac{1}{e^{\varepsilon}+1}, & \text{if } y=1-v.
\end{cases}
\end{align}
\end{definition}

Thus, we have $W^b_{\tau,Y|X}(y|x)=W_{Y|V}\bigl(y|Z_\tau(x)\bigr)$.
We next present the BM-based CPD algorithm and its Chernoff-optimal choice of $\tau$.
\begin{algorithm}[H]
    \caption{Offline $\mathrm{BM-CPD}$ (\ensuremath{\mathcal{D}, \mathcal{X}, P_{0}, P_{1}}, $\varepsilon$)}
    \label{alg:offline_bmcpd}
    \begin{algorithmic}[1]
        \STATE \textbf{Input:} Dataset $\mathcal{D}$, $\mathcal{X}$, PMFs $(P_0, P_1)$, $\varepsilon>0$
        \STATE $(P_0,P_1)$ through the $\mathrm{BM}$ $W^{b}_{\tau}$ induces distributions $(Q^{\tau}_0,Q^{\tau}_1)$, respectively.
        \STATE Compute $\tau^{*} \in \arg\max_{\tau}\mathrm{I}_{\mathrm{ch}}(Q^{{\tau}}_0,Q^{{{\tau}}}_1) $
        \STATE For all $i \in [n]$, pass $x_i$ in $\mathcal{D}$ through $W^{b}_{\tau^{*}}$ to get $y_i \in \{0,1\}$ s.t. $\widetilde{D}=(y_1,\ldots,y_n)$
        \FOR{$k=1$ \textbf{to} $n$}
            \STATE Compute the ratio $l(\widetilde{\mathcal{D}},k) = \sum_{i=k}^{n}\log\dfrac{Q^{\tau^{*}}_1(y_i)}{Q^{\tau^{*}}_0(y_i)}$
        \ENDFOR
        \STATE \textbf{Output:} $\widehat{k} = \min(\arg\max_{k \in \mathcal{K}} l(\widetilde{\mathcal{D}},k)).$
    \end{algorithmic}
\end{algorithm}
\vspace*{-6mm}
We choose $\tau^{*}$ (in step $3$) to maximize the Chernoff information $(\mathrm{I}_{\,\mathrm{ch}})$ between the induced Bernoulli output distributions $Q_0^\tau$ and $Q_1^\tau$, thereby minimizing our Chernoff-based GLRT error bound. Appendix~\ref{app:tau_comp} illustrates the empirical benefit of this choice, particularly in the low-privacy regime. For instance, at $\varepsilon=6$ in Fig.~\ref{fig:tau2}, $\tau=1$ yields about half the Chernoff information attained by $\tau^\star$. Under high-privacy, however, both choices give comparable empirical localization errors.

\textbf{Computation of $\tau^{*}$:}
Although step $3$ in Algorithm~\ref{alg:offline_bmcpd} is written as an optimization over a continuous parameter $\tau$, for a discrete finite alphabet $\mathcal{X}$, the induced BM changes only when $\tau$ crosses one of the finitely (at most $|\mathcal{X}|$) many likelihood ratio values $P_0(x)/P_1(x)$. Therefore, Chernoff-optimal $\tau^{*}$ can be found using a sort and scan procedure with time complexity of $O(|\mathcal{X}|\log|\mathcal{X}|)$. 

\begin{theorem}\label{thm:bmcpd}
Let $k^{\star} \in \mathcal{K}$ be the true change-point of the dataset $\mathcal{D} \in \mathcal{X}^n$, with data points drawn from distributions $(P_0,P_1)$ on $\mathcal{X}$. Let
$\mathcal{S}_1=\{x\in\mathcal X:P_0(x)\ge P_1(x)\}$.~Define
$D_1=\mathrm{H}^2\!\left(\operatorname{Ber}(P_0(\mathcal{S}_1)),\operatorname{Ber}(P_1(\mathcal{S}_1))\right)$, and
\[
\rho_{\mathrm{BM}}
=
\min\left(
\sqrt{1-c_{\varepsilon,2}^{\,2}\delta^2},
\,1-L_{\varepsilon,2}(D_1)
\right).
\]
Then, the $\mathrm{BM\text{-}CPD}$ estimator in Alg.~\ref{alg:offline_bmcpd} satisfies $\varepsilon$-LDP, and is $(\alpha,\beta_{\mathrm{BM}})$-accurate for any $\alpha \in [n]$ and some
\begin{align}
\beta_{\mathrm{BM}}
\leq 2\rho_{\mathrm{BM}}^{\alpha+1}\left(1-\frac{\rho_{\mathrm{BM}}^{\alpha+1}}{2}\right) 
 \label{eq:bmcpdbound}
\end{align}
where $c_{\varepsilon,2}$ and $L_{\varepsilon,2}$ are defined in~\eqref{eq:not1} and~\eqref{eq:not2}.
\end{theorem}
The proof is deferred to the Appendix~\ref{app:bmcpd}. We also have the following simpler upper bound,
\begin{corollary}
Under the assumptions of Thm.~\ref{thm:bmcpd}, the $\mathrm{BM\text{-}CPD}$ estimator is $(\alpha,\beta_{\mathrm{BM}})$-accurate for every $\alpha\in[n]$ and some
\[
\beta_{\mathrm{BM}}\leq 2e^{-\frac{\alpha+1}{2}\delta^2c_{\varepsilon,2}^{\,2}} \left( 1-\frac{ e^{-\frac{\alpha+1}{2}\delta^2c_{\varepsilon,2}^{\,2}}}{2}\right),
\]
where $c_{\varepsilon,2}=\tanh(\varepsilon/2)$.
\end{corollary}

\begin{figure*}[h]
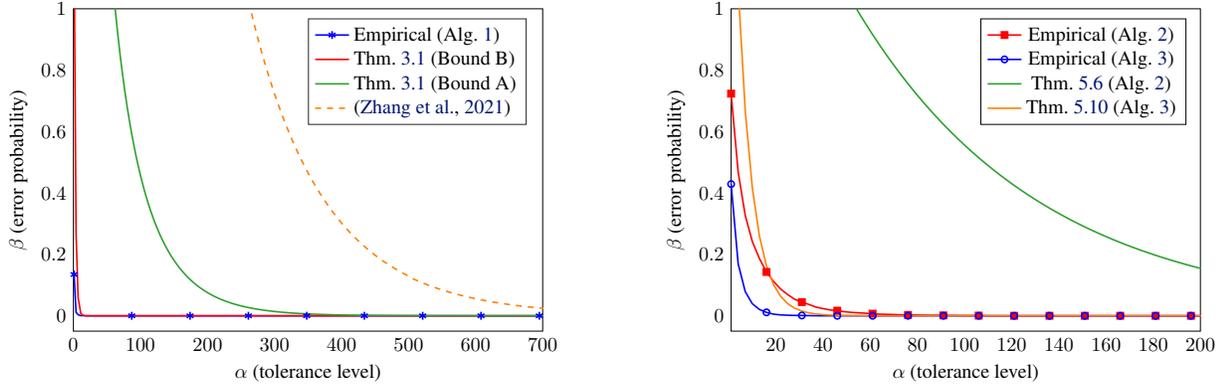
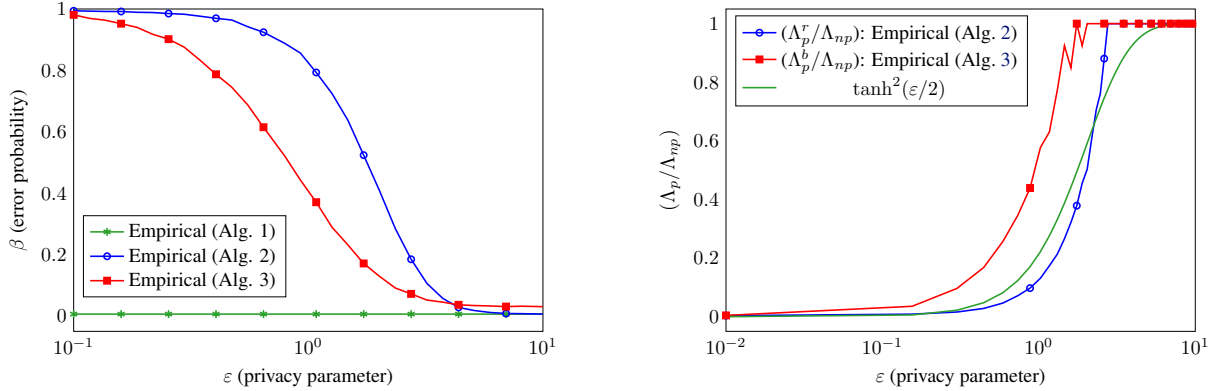

\centering

\begin{subfigure}[t]{0.48\textwidth}
  \centering
  \input{figures/np_tpoisson_0}
    \vspace{1.5mm}
  \caption{\textbf{Non-private setting:} Comparison of the empirical performance of Alg.~\ref{alg:offline_cpd} with our theoretical bounds.}
  \label{fig:np_poisson}
\end{subfigure}\hfill
\begin{subfigure}[t]{0.48\textwidth}
  \centering
  \input{figures/priv_tpois_main}
  \vspace{1.5mm}
  \caption{\textbf{Private setting: }
  Comparison of the empirical and theoretical performance of private algorithms. We fix $\varepsilon = 2$.}
  \label{fig:priv_tpois_main}
\end{subfigure}

\begin{subfigure}[t]{0.48\textwidth}
  \centering
  \input{figures/emp_tpoiss_0}
  \caption{ \textbf{Empirical performance of private algorithms:}
  We fix $\alpha =5$. The Alg.~\ref{alg:offline_rrcpd} exhibits better performance than Alg.~\ref{alg:offline_bmcpd} in the low privacy regime and the converse holds in the high privacy regime.}
  \label{fig:priv_tpois_eps}
\end{subfigure}\hfill
\begin{subfigure}[t]{0.48\textwidth}
  \centering
  \input{figures/cop_tp_0}
  \caption{\textbf{Cost of privacy:} 
  Bounds on the exponent cost of privacy as calculated in~\eqref{eq:cost_priv}. The shaded area is the possible range for $C_{\mathrm{exp}}(\varepsilon)$.}
  \label{fig:cop_tp_0}
\end{subfigure}
\caption{Experiment Plots (see Appendices~\ref{app:npcomp},~\ref{app:priv_comp},~\ref{app:exp3}, and ~\ref{app:continuous} for detailed experiments).}
\label{fig:priv_tpois_grid}
\end{figure*}

\noindent\textbf{Privacy-regime ordering of the guarantees:}
Interestingly, our finite-sample guarantees reflect the qualitative high- and low-privacy ordering established at the level of Chernoff information in Thm.~\ref{thm:high_eps_div} and Cor.~\ref{cor:small-large-eps}. Let $B_{\mathrm{ARR}}(\alpha,\varepsilon)$ and $B_{\mathrm{BM}}(\alpha,\varepsilon)$ denote the upper bounds in Thms.~\ref{thm:rrcpd} and~\ref{thm:bmcpd}, respectively. If $r>2$, then, for every $\alpha\in[n]$,
$$
B_{\mathrm{BM}}(\alpha,\varepsilon)<B_{\mathrm{ARR}}(\alpha,\varepsilon)
$$
whenever $\varepsilon>0$ satisfies
$(e^\varepsilon+1)^2<r(e^\varepsilon+r-1)$. In particular, this ordering holds for every $0<\varepsilon\leq1$. Conversely, $B_{\mathrm{ARR}}<B_{\mathrm{BM}}$ for all $\varepsilon$ above a finite threshold depending only on $(P_0,P_1)$, uniformly over $\alpha$. However, these sufficient ranges for ordering the upper bounds need not coincide with the Chernoff-optimality regimes. Furthermore, these comparisons concern upper bounds and do not by themselves establish an ordering of the actual localization errors.
(See Appendix~\ref{app:comparison}).

\noindent\textbf{Asymptotic cost of privacy:} To summarize the privacy dependence of our finite-sample bounds at the exponential scale, we compare the private and non-private minimax localization-error exponents. Let $E_{{np}}$ and $E_p$ denote the minimax error exponents as $\alpha_n\to\infty$ with $\alpha_n = o(n)$.~Combining our finite-sample guarantees and minimax lower bounds, the exponent cost of privacy $C_{\exp}(\varepsilon)=E_{{np}}/E_p$ satisfies
\begin{align}
    \max\left(1,\frac{1}{2\,c^2_{\varepsilon,2}}\right)
\leq C_{\exp}(\varepsilon)\leq \frac{4\mathrm{I}_{\mathrm{ch}}(P_0,P_1)}{c^2_{\varepsilon,2}\,\mathrm{d^2_{\mathrm{TV}}}(P_0,P_1)} \label{eq:cost_priv}
\end{align}
for every $\varepsilon>0$, where $c_{\varepsilon,2}=\tanh(\varepsilon/2)$.~Consequently, for fixed $(P_0,P_1)$, $C_{\exp}(\varepsilon)=\Theta\bigl(\tanh^{-2}(\varepsilon/2)\bigr)$ uniformly over $\varepsilon>0$, and $C_{\exp}(\varepsilon)=\Theta(\varepsilon^{-2})$ as
$\varepsilon\to0$. See Appendix~\ref{app:cop} for details.

\section{Experimental Results}\label{sec:exp}

We provide empirical comparisons by plotting bounds for several synthetic datasets. We fix the dataset size $n=2000$, the true change-point $k^{\star}=1000$ and plot theoretical bounds for different choices of $(P_0,P_1)$. The empirical curves are plotted via Monte Carlo simulations by repeatedly sampling data from $(P_0,P_1)$, estimating change-point using the proposed algorithms, and repeating this procedure $10{,}000$ times. With $10$K independent repetitions, the standard error of each estimated error probability is at most  $1/(2\sqrt{{10{}000}}) = 0.005$. Thus, the Monte Carlo uncertainty is small relative to the trends shown in plots. We run all the experiments on Apple M$1$ CPU with $16$GB RAM.

We consider truncated Poisson distributions $\mathrm{TPois}(\lambda,m)$ with truncation parameter $m=10$ i.e., $P_0 =\mathrm{TPois}(1,10)$, $P_1 = \mathrm{TPois}(4,10)$ for all experiments here. Additional experiment results are provided in Appendices~\ref{app:npcomp},~\ref{app:priv_comp},~\ref{app:exp3}, and~\ref{app:continuous}, based on Bernoulli, binomial, truncated geometric, and Gaussian distributions with different parameters.
\section{Conclusion}
We study offline parametric change-point localization under
local differential privacy.~In the non-private setting, we prove finite-sample
minimax lower bounds and sharpen the GLRT guarantees, showing optimality up to a factor of two in the leading localization-error exponent.~Under privacy, we prove minimax lower bounds that hold for every $\varepsilon$-LDP mechanism, via SDPI characterizations for the Jeffreys-Rényi divergence. Guided by the induced Chernoff information, we propose
$\mathrm{ARR-CPD}$ and $\mathrm{BM-CPD}$, whose $\varepsilon$-LDP mechanisms maximize this criterion for sufficiently large and small $\varepsilon$, respectively. Our finite-sample guarantees show the same pattern: the $\mathrm{BM-CPD}$ bound is smaller at high privacy and the $\mathrm{ARR-CPD}$ bound at low privacy. Further, we show that for fixed pre- and post-change distributions, privacy reduces the minimax localization-error exponent by a factor of order $\tanh^{2}(\varepsilon/2)$.~Our  experiments on synthetic data illustrate the bounds and the complementary behavior of the two algorithms.

\bibliography{example_paper}
\bibliographystyle{plain}

\clearpage
\appendix
\thispagestyle{empty}

\onecolumn
\section{Related Works}\label{app:related_works}
Change-point detection (CPD) has been an important research problem in statistics going back to early works of~\cite{shewhart}. Different algorithms~\cite{page, robert, pollak2} have been proposed in the literature to accurately estimate change-points under multiple frameworks. The CPD problem has been studied in both parametric~\cite{page,moustakides, quandt, xie} and non-parametric settings~\cite{petit_np, gombay}. In the parametric setting, the pre-change and post-change distributions, denoted by $P_0$ and $P_1$, are assumed to be known. In contrast, in the non-parametric setting, these distributions are unknown, and methods rely on discrepancy statistics or other distribution-free procedures~\cite{kernel,mmd}. Offline CPD estimates changes after observing a batch of data points of fixed (small) size, whereas online CPD detects changes sequentially as data arrive. In this work, we focus on the offline parametric CPD setting. 

An early work introduced the cumulative sum (CUSUM) algorithm for parametric CPD based on the generalized likelihood ratio test (GLRT)~\cite{page}. This algorithm has been employed in numerous studies within the non-Bayesian framework and has stood the test of time. The CUSUM algorithm was shown to be asymptotically optimal under Lorden’s criterion in~\cite{lorden}, 
while exact optimality was later established by~\cite{moustakides}. Our approach is also inspired by the CUSUM algorithm. 

Differential privacy is a well studied privacy notion~\cite{dwork2006our, dwork2006calibrating}.  There is extensive work on differentially private estimation~\cite{duchi_estimation,duchi_privacy,ye2018optimal}, learning~\cite{dl-dp,liu2021machine}, and hypothesis testing~\cite{kairouz2014extremal,canonne2019structure}. In this work, we are interested in a local variant of DP, namely local differential privacy (LDP)~\cite{bassily2015local, yang2023local}, to study change-point detection.

Parametric CPD under differential privacy was first studied in~\cite{cummings_org} (see also the extended work~\cite{zhang_cpd}). The authors considered parametric CPD and employed the well known Laplace mechanism to design differentially private algorithms, deriving performance bounds. In the recent years, the problem of private change-point detection has attracted significant attention. In~\cite{lau2020privacy}, the authors used maximal leakage~\cite{issa2019operational} as a privacy constraint to study CPD.~\cite{cummings2020privately} investigated differentially private CPD in the non-parametric setting, while~\cite{berrett2021locally} examined multivariate online non-parametric CPD under local differential privacy. This line of work was further extended to CPD over networks in~\cite{li2022network}. \cite{seif2025differentially} studied private community detection on graphs. A differentially private variant of the CUSUM algorithm was introduced in~\cite{xiesequential25}. CPD under homomorphic encryption has been recently studied in~\cite{dpcpd2}. See also~\cite{dpcpd1,regression, sequence, manifolds} for more recent work on sequential CPD algorithms. 

Very recently, Zhang et al.~\cite{zhang2026} studied online parametric CPD under local differential privacy. Their performance criteria are average run length to false alarm and detection delay, whereas our work studies offline change-point localization under the
$(\alpha,\beta)$-accuracy criterion. Thus, the two works address different
statistical objectives. 

\section{Notations}\label{app:notations}
We denote random variables by upper case letters (eg. $X$), the values they take by lower case letters (eg., $x$), and their alphabets by calligraphic letters (eg. $\cal{X}$).  
The set of real numbers, non-negative real numbers and real vectors (of length $n$) are denoted by $\mathbb{R}$, $\mathbb{R}_+$, and $\mathbb{R}^n$ respectively. The set of natural numbers is denoted by $\mathbb{N}$. For $a\in\mathbb{N}$, let $[a]:=\{1,2,\cdots, a\}$. For $ a,b\in\mathbb{R}$, $a\leq b$, let $[a,b]$ denote the closed interval defined between $a$ and $b$. 

Throughout the paper, we assume that the distributions $P_0$ and $P_1$ are distinct and fully supported on $\mathcal{X}$. The probability simplex $\Delta(\mathcal X)$ is the set of all PMFs on $\mathcal X$. 
The tensor power $P^{\otimes m}(x_1,\ldots,x_m)\triangleq\prod_{i=1}^{m}P(x_i)$ denotes the joint PMF of $m$ independent observations with common PMF $P$. For PMFs $P$ on $\mathcal X$ and $Q$ on $\mathcal Y$, their tensor product is the  independent joint PMF $(P\otimes Q)(x,y)\triangleq P(x)Q(y)$ on $\mathcal X\times\mathcal Y$. Given distributions $P_0$ and $P_1$ in the simplex $\Delta(\cal{X})$, $\mathrm{D}_{\mathrm{KL}}(P_0\|P_1)$ denotes the KL divergence between $P_0$ and $P_1$, while  $\mathrm{d}_{\mathrm{TV}}(P_0,P_1)=\frac{1}{2} \sum_{x\in\cal{X}}|P_0(x)-P_1(x)|  $ denotes the Total Variation (TV) distance. Under the true change point $k^\star$, the observations have joint PMF $\mathbb P_{k^\star}(X^n=x^n)=\prod_{i<k^\star}P_0(x_i)\prod_{i\ge k^\star}P_1(x_i)$. All logarithms are to the base $e$, unless stated otherwise. 

We also need Hellinger distance to state our private accuracy analysis results.
\begin{definition}[Hellinger distance]
Let $P_0$ and $P_1$ be PMFs on $\mathcal X$. Their squared
Hellinger distance is defined by
\[
\mathrm{H}^2(P_0,P_1)\triangleq
1-\sum_{x\in\mathcal X}\sqrt{P_0(x)P_1(x)},
\]
and their Hellinger distance is $\mathrm{H}(P_0,P_1)=\sqrt{\mathrm{H}^2(P_0,P_1)}$.
\end{definition}

\section{Non-Private Change-Point Detection}

\subsection{Proof of Theorem~\ref{thm:minimax_lb} (Non-private finite sample minimax lower bound)}\label{app:minimax_lb}
\begin{proof}
Our proof's is based on the generalization of the Le Cam's $2-$point method to $M$ points. Let
$E_{\mathcal{K}}(\alpha) \triangleq \inf_{\widehat{k}}\sup_{k^{\star} \in \mathcal{K}}\mathbb{P}_{k^{\star}}\left(|\widehat{k}-k^{\star}|>\alpha\right)
$. We consider a simpler problem where there are only $M$ potential change points in a given dataset. For a non-trivial error probability, these points should be separated by at least $2\alpha+1$ i.e., 
\begin{align}
    k_j = 2 +j(2\alpha+1); \qquad \qquad \text{for } j \in \{0,1,\ldots,M-1\}
\end{align}
Since we need $k_{M-1} \leq n$, we should have $(M-1)(2\alpha+1) \leq n-2$, yielding
\begin{align}
    M \leq 1 + \Big\lfloor\frac{n-2}{2\alpha+1}\Big\rfloor. \label{eq:constraint}
\end{align}
Let $\mathcal{K}_{M}$ denote the set $\{k_0,k_1,\ldots,k_{M-1}\}$. We have that a minimax lower bound under this simpler $M-$point problem is also a lower bound on the original problem, thus, $E_{\mathcal{K}}(\alpha)\geq E_{\mathcal{K}_M}(\alpha)$.

Consider any CPD estimator (possibly randomized) $\widehat{k}: \mathcal{X}^n  \rightarrow \mathcal{K}\, \cup\, \{\infty\}$, and a $M-$ary test $\psi: \mathcal{K} \, \cup\, \{\infty\} \rightarrow \mathcal{K}_M\, \cup\, \{\infty\}$ as follows
\begin{align}
    \psi(\widehat{k})= \begin{cases}
        k_j\, :\, \text{if } \widehat{k} \in \mathcal{K}\text{ and }|\widehat{k}-k_j| \leq \alpha \\
        \infty\, : \text{if otherwise}
    \end{cases}
\end{align}
Then, the sum of error probabilities of the test $\psi$ can be defined as
\begin{align}
    e(\psi(\widehat{k}))&= \sum_{j=0}^{M-1}\text{Pr}(\psi(\widehat{k}) \neq k_j \big|\text{True CP is $k_j$})\\
    &=\sum_{j=0}^{M-1}\mathbb{P}_{k_j}(\psi(\widehat{k}) \neq k_j)\\
    &\leq \sum_{j=0}^{M-1} \sup_{k^{\star} \in \mathcal{K}_M}\mathbb{P}_{k^{\star}}(\psi(\widehat{k}) \neq k^{\star})\\
    &=M \cdot \sup_{k^{\star} \in \mathcal{K}_M} \mathbb{P}_{k^{\star}}\left(|\widehat{k}-k^{\star}|>\alpha\right)\label{eq:ubb1}
\end{align}
Next, we will establish a lower bound on $  e(\psi(\widehat{k}))$. Combining it with~\eqref{eq:ubb1} will allow us to derive lower bound on the required minimax error $E_{\mathcal{K}_M}(\alpha)$. We have
\begin{align}
     e(\psi(\widehat{k}))&=M- \sum_{j=0}^{M-1}\mathbb{P}_{k_j}(\psi(\widehat{k})= k_j)\\
     &= M- \sum_{x^n \in \mathcal{X}^n}\sum_{j=0}^{M-1}\mathbb{P}_{k_{j}}(x^n)\mathbb{P}_{k_j}(\psi(\widehat{k})= k_j|x^n)\\
     &\geq M-  \sum_{x^n \in \mathcal{X}^n} \sup_{k_j \in \mathcal{K}_M}\mathbb{P}_{k_j}(x^n) \label{eq:lbb:11}
\end{align}
Note that for every $k_j \in \mathcal{K}_M$,
\begin{align}
    \mathbb{P}_{k_j}(x^n)&=\mathbb{P}_{k_0}(x^n) + \sum_{i=1}^{j} \left[\mathbb{P}_{k_i}(x^n)-\mathbb{P}_{k_{i-1}}(x^n)\right] \\
    &\leq \mathbb{P}_{k_0}(x^n) + \sum_{i=1}^{M-1} \left[\mathbb{P}_{k_{i}}(x^n)-\mathbb{P}_{k_{i-1}}(x^n)\right]^{+}
\end{align}
Therefore,
\begin{align}
  \sum_{x^n \in \mathcal{X}^n} \sup_{k_j \in \mathcal{K}_M}\mathbb{P}_{k_j}(x^n)  &\leq  \sum_{x^n \in \mathcal{X}^n}\mathbb{P}_{k_0}(x^n) +\sum_{x^n \in \mathcal{X}^n} \sum_{i=1}^{M-1} \left[\mathbb{P}_{k_{i}}(x^n)-\mathbb{P}_{k_{i-1}}(x^n)\right]^{+}\\
  &=1 + \sum_{i=1}^{M-1} \mathrm{d}_{\mathrm{TV}}(\mathbb{P}_{k_{i}},\mathbb{P}_{k_{i-1}})\\
  &= 1+ (M-1)\mathrm{d}_{\mathrm{TV}}(\mathbb{P}_{k_{0}},\mathbb{P}_{k_{1}})
\end{align}
Thus,~\eqref{eq:lbb:11} can be further written as
\begin{align}
  e(\psi(\widehat{k})) \geq (M-1)(1-\mathrm{d}_{\mathrm{TV}}(\mathbb{P}_{k_{0}},\mathbb{P}_{k_{1}}) )  \label{eq:lbb1}
\end{align}
Consequently, from~\eqref{eq:ubb1} and~\eqref{eq:lbb1}, we obtain
\begin{align}
   E_{\mathcal{K}_M}(\alpha)= \inf_{\widehat{k}}\sup_{k^{\star} \in \mathcal{K}_M} \mathbb{P}_{k^{\star}}\left(|\widehat{k}-k^{\star}|>\alpha\right) \geq \frac{M-1}{M}\left(1-\mathrm{d}_{\mathrm{TV}}(\mathbb{P}_{k_{0}},\mathbb{P}_{k_{1}})\right).\label{eq:main_lb}
\end{align}
Now, we will further apply the refinement on the R.H.S to obtain our final lower bound.
Since $\min(p,q)=p-(p-q)^+$ and
$\sum_{x^n}\big(\mathbb{P}_{k_0}(x^n)-\mathbb{P}_{k_1}(x^n)\big)^+=\mathrm{d}_{\mathrm{TV}}(\mathbb{P}_{k_0},\mathbb{P}_{k_1})$,
\begin{align}\label{eq:affinity}
1-\mathrm{d}_{\mathrm{TV}}(\mathbb{P}_{k_0},\mathbb{P}_{k_1})
\;=\;\sum_{x^n\in\mathcal{X}^n}\min\big(\mathbb{P}_{k_0}(x^n),\,\mathbb{P}_{k_1}(x^n)\big).
\end{align}
Note that the two hypotheses differ only on the block of $2\alpha+1$ coordinates
$x_{k_0},\dots,x_{k_1-1}$. These observations are drawn from $P_1$ when the true change-point is $k_0$ and from $P_0$ when it is $k_1$, while every other
coordinate follows the same distribution. Let
\begin{align*}
C(x^n)=\prod_{j<k_0}P_0(x_j)\prod_{j\geq k_1}P_1(x_j),
\end{align*}
Then, we have
\begin{align*}
\mathbb{P}_{k_0}(x^n)=C(x^n)\prod_{j=k_0}^{k_1-1}P_1(x_j)
\qquad \text{and} \qquad
\mathbb{P}_{k_1}(x^n)=C(x^n)\prod_{j=k_0}^{k_1-1}P_0(x_j).
\end{align*}
Since $\min(ca,cb)=c\min(a,b)$ for $c\ge0$, we can pull out the common factor $C(x^n)$ from the minimum in~\eqref{eq:affinity}. Summing first over the block and then over the remaining coordinates, over which $C$ sums to $1$, we obtain
\begin{equation}\label{eq:Adelta}
1-\mathrm{d}_{\mathrm{TV}}(\mathbb{P}_{k_0},\mathbb{P}_{k_1})
\;=\;\sum_{x'\in\mathcal{X}^{(2\alpha+1)}}\min\Bigg(\prod_{j=1}^{2\alpha+1}P_1(x'_j),\
\prod_{j=1}^{2\alpha+1}P_0(x'_j)\Bigg)\;\triangleq\;A_{2\alpha+1} .
\end{equation}
Let $Z(\lambda)\triangleq\sum_{x\in\mathcal{X}}P_0(x)^{\lambda}P_1(x)^{1-\lambda}$ for
$\lambda\in[0,1]$. Further, let the minimum in the definition of the Chernoff
information (see Def.~\ref{def:ich}) be attained at some $\lambda^\star\in(0,1)$, with
\begin{align*}
Z(\lambda^\star)=e^{-\mathrm{I}_{\,\mathrm{ch}}(P_0,P_1)}.
\end{align*}
Now, define the log-likelihood ratio, the tilted distribution, and its second moment by
\begin{align*}
\ell(x)\triangleq\log\frac{P_0(x)}{P_1(x)},\qquad
P_\star(x)\triangleq\frac{P_0(x)^{\lambda^\star}P_1(x)^{1-\lambda^\star}}{Z(\lambda^\star)},\qquad
\Sigma^2\triangleq\mathbb{E}_{P_\star}\big[\ell(X)^2\big].
\end{align*}
Then, we complete our proof in the following steps:

\emph{(a). $\ell$ is centered around zero under $P_\star$:} On, differentiating,
\begin{align*}
\frac{d}{d\lambda}\log Z(\lambda)
=\frac{\sum_{x}P_0(x)^{\lambda}P_1(x)^{1-\lambda}\,\ell(x)}{Z(\lambda)},
\end{align*}
which is equal to $\mathbb{E}_{P_\star}[\ell(X)]$ at $\lambda=\lambda^\star$. Since
$\lambda^\star$ is a minimizer of $\log Z$, we have that $\mathbb{E}_{P_\star}\big[\ell(X)\big]=0$ and $\Sigma^2=\mathrm{Var}_{P_\star}\big(\ell(X)\big)$.

\emph{(b) Changing te measure by tilting:} Since
$P_0^{\lambda^\star}(x)P_1^{1-\lambda^\star}(x)=P_1(x)\,e^{\lambda^\star\ell(x)}=P_0(x)\,e^{-(1-\lambda^\star)\ell(x)}$ for all $x \in \mathcal{X}$, we can write
\begin{align}\label{eq:tilt1}
P_1(x)=e^{-\mathrm{I}_{\,\mathrm{ch}}(P_0,P_1)}\,P_\star(x)\,e^{-\lambda^\star\ell(x)}\qquad \text{and}
\qquad
P_0(x)=e^{-\mathrm{I}_{\,\mathrm{ch}}(P_0,P_1)}\,P_\star(x)\,e^{(1-\lambda^\star)\ell(x)} .
\end{align}
For $x' =(x_1,\dots,x_{2\alpha+1})\in\mathcal{X}^{2\alpha+1}$, let
$S(x')\triangleq\sum_{j=1}^{2\alpha+1}\ell(x'_j)$ and
$P_\star^{\otimes {2\alpha+1}}(x')\triangleq\prod_{j=1}^{2\alpha+1}P_\star(x'_j)$.
Multiplying~\eqref{eq:tilt1} over $2\alpha+1$ coordinates gives
\begin{align}
\prod_{j=1}^{2\alpha+1}P_1(x'_j)&=e^{-{(2\alpha+1)}\,\mathrm{I}_{\,\mathrm{ch}}(P_0,P_1)}\,e^{-\lambda^\star S(x')}\,P_\star^{\otimes 2\alpha+1}(x')\notag\\
\prod_{j=1}^{2\alpha+1}P_0(x'_j)&=e^{-(2\alpha+1)\,\mathrm{I}_{\,\mathrm{ch}}(P_0,P_1)}\,e^{(1-\lambda^\star)S(x')}\,P_\star^{\otimes2\alpha+1}(x')\notag.
\end{align}
 
\emph{(c) Computing the minimum:} Let $S^+=\max(S,0)$ and $S^-=\max(-S,0)$. If
$S(x')\geq0$ then $e^{-\lambda^\star S}\leq1\leq e^{(1-\lambda^\star)S}$, and if
$S(x')<0$ the reverse holds. In both cases, we have
\begin{align*}
\min\Bigg(\prod_{j=1}^{2\alpha+1}P_1(x'_j),\ \prod_{j=1}^{2\alpha+1}P_0(x'_j)\Bigg)
=e^{-(2\alpha+1)\,\mathrm{I}_{\,\mathrm{ch}}(P_0,P_1)}\,
\exp\!\big(-\lambda^\star S^+(x')-(1-\lambda^\star)S^-(x')\big)\,P_\star^{\otimes 2\alpha+1}(x').
\end{align*}
Now, summing over $x' \in \mathcal{X}^{2\alpha+1}$ and substituting in~\eqref{eq:Adelta}, we get that
\begin{align}\label{eq:Aexact}
A_{2\alpha+1}=e^{-(2\alpha+1)\,\mathrm{I}_{\,\mathrm{ch}}(P_0,P_1)}\;
\mathbb{E}_{P^{\otimes}_{\star}}\Big[\exp\!\big(-\lambda^\star S^+-(1-\lambda^\star)S^-\big)\Big],
\end{align}
where $\mathbb{E}_{P^{\otimes}_{\star}}$ denotes expectation when $x'_1,\dots,x'_{2\alpha+1}$ are i.i.d.\
$P_\star$.
 
\emph{(d). Bounding the expectation:} Under $\mathbb{E}_{P^{\otimes}_{\star}}$, $S$ is a sum of
$2\alpha+1$ i.i.d.\ terms which have mean zero and variance $\Sigma^2$ (from (a).). Thus, $\mathbb{E}_{P^{\otimes}_{\star}}[S]=0$ and $\mathbb{E}_{P^{\otimes}_{\star}}[S^2]=(2\alpha+1)\Sigma^2$. Since $S=S^+-S^-$ has mean zero and $|S|=S^++S^-$, we have $\mathbb{E}_{P^{\otimes}_{\star}}  [S^+]=\mathbb{E}_{P^{\otimes}_{\star}}[S^-]=\frac12\mathbb{E}_{P^{\otimes}_{\star}}[|S|]$,
and therefore
\begin{align*}
\mathbb{E}_{P^{\otimes}_{\star}}\big[\lambda^\star S^++(1-\lambda^\star)S^-\big]=\frac12\,\mathbb{E}_{P^{\otimes}_{\star}}[|S|]. 
\end{align*}
Consequently, we have
\begin{align}
\mathbb{E}_{P^{\otimes}_{\star}}\Big[\exp\!\big(-\lambda^\star S^+-(1-\lambda^\star)S^-\big)\Big]
\;\geq\;\exp\!\Big(-\frac12\,\mathbb{E}_{P^{\otimes}_{\star}}[|S|]\Big)
\;\geq\;\exp\!\Big(-\frac{\Sigma}{2}\sqrt{2\alpha+1}\Big)
\end{align}
where the first inequality follows by Jensen's inequality, and the second follows by applying the Cauchy--Schwarz inequality. On, combining with~\eqref{eq:Adelta} and~\eqref{eq:Aexact}, we obtain
\begin{equation}\label{eq:stepdd}
1-\mathrm{d}_{\mathrm{TV}}(\mathbb{P}_{k_0},\mathbb{P}_{k_1})\;=\;A_{2\alpha+1}\;\geq\;
\exp\!\Big(-(2\alpha+1)\,\mathrm{I}_{\,\mathrm{ch}}(P_0,P_1)-\frac{\Sigma}{2}\sqrt{2\alpha+1}\Big).
\end{equation}
Finally, on substituting~\eqref{eq:stepdd} in~\eqref{eq:main_lb}, we have
\begin{align}
E_{\mathcal{K}}(\alpha)\;\ge\;E_{\mathcal{K}_M}(\alpha)\;\geq\;\Big(\frac{M-1}{M}\Big)
\exp\!\Big[-(2\alpha+1)\,\mathrm{I}_{\,\mathrm{ch}}(P_0,P_1)-\frac{\Sigma}{2}\sqrt{2\alpha+1}\Big]\label{eq:lb1}.
\end{align}
The final result is obtained by choosing the largest $M$ (say $M_{\alpha}$) satisfying~\eqref{eq:constraint}. This completes the proof of our theorem.

Alternatively, we can bound the TV distance term in~\ref{eq:main_lb} using an inequality (Cor. 1) from Sason~\cite{sason}, instead of the tilting argument. It yields a simpler lower bound
that depends on $(P_0,P_1)$ only through $\mathrm{I}_{\,\mathrm{ch}}(P_0,P_1)$. Its exponent, however, is roughly twice that of~\eqref{eq:lb1}. We present it below.

\begin{corollary}\label{cor:minimax_np}
A simpler minimax lower bound depending only on $\mathrm{I}_{\,\mathrm{ch}}(P_0,P_1)$ also holds,
$$E_{\mathcal{K}}(\alpha)\geq\frac{1}{2}\Big(1-\frac{1}{M_\alpha}\Big)\,e^{-2(2\alpha+1)\,\mathrm{I}_{\,\mathrm{ch}}(P_0,P_1)}.$$ However, its leading exponent is roughly twice that of~\eqref{eq:thm-lb}, so it can improve on Thm.~\ref{thm:minimax_lb} only for small $\alpha$ when $\Sigma$ is large relative to $\mathrm{I}_{\,\mathrm{ch}}(P_0,P_1)$.
\end{corollary}

From Cor. $1$. in~\cite{sason}, we have
\begin{align}
    \mathrm{I}_{\,\mathrm{ch}}(\mathbb{P}_{k_0},\mathbb{P}_{k_1}) \geq -\frac{1}{2}\log\left(1-\mathrm{d}^2_{\mathrm{TV}}(\mathbb{P}_{k_{0}},\mathbb{P}_{k_{1}})\right)
\end{align}
Using the above inequality in~\eqref{eq:main_lb}, we obtain
\begin{align}
      E_{\mathcal{K}_M}(\alpha)= \inf_{\widehat{k}}\sup_{k^{\star} \in \mathcal{K}_M} \mathbb{P}_{k^{\star}}\left(|\widehat{k}-k^{\star}|>\alpha\right) &\geq \frac{M-1}{M}\left(1-\mathrm{d}_{\mathrm{TV}}(\mathbb{P}_{k_{0}},\mathbb{P}_{k_{1}})\right)\notag\\
      &\geq \frac{M-1}{M}\left(1-\sqrt{1-\exp\,(-2\mathrm{I}_{\,\mathrm{ch}}(\mathbb{P}_{k_{0}},\mathbb{P}_{k_{1}}))} \right)\\
      &= \frac{M-1}{M}\left(1-\sqrt{1-\exp\,(-2 (2\alpha+1)\mathrm{I}_{\,\mathrm{ch}}({P}_{{0}},{P}_{{1}}))} \right)\\
      & \geq \frac{M-1}{2M} \,\exp\, \Big[-2(2\alpha+1)\mathrm{I}_{\,\mathrm{ch}}({P}_{{0}},{P}_{{1}}) \Big]
\end{align}
Therefore, 
\begin{align}
    E_{\mathcal{K}}(\alpha) \geq \Big(\frac{M-1}{2M}\Big) \,\exp\, \Big[-2(2\alpha+1)\mathrm{I}_{\,\mathrm{ch}}({P}_{{0}},{P}_{{1}}) \Big]\label{eq:lb2}
\end{align}
The final result is obtained by choosing the largest $M$ (say $M_{\alpha}$) satisfying~\eqref{eq:constraint}. Note that the bound in~\eqref{eq:lb1} exceeds~\eqref{eq:lb2} if and only if
$(2\alpha+1)\,\mathrm{I}_{\,\mathrm{ch}}(P_0,P_1)+\log 2\ \geq\ \frac{\Sigma}{2}\sqrt{2\alpha+1}$.
Since the L.H.S grows linearly in $2\alpha+1$ and the R.H.S only as its square root,~\eqref{eq:lb1} is the larger (better) bound for all sufficiently large $\alpha$ and~\eqref{eq:lb2} can be larger only for small $\alpha$ when $\Sigma$ is large relative to $\mathrm{I}_{\,\mathrm{ch}}(P_0,P_1)$. This completes our proof.
\end{proof}

\subsection{Non-private finite-sample guarantees for GLRT (Algorithm~\ref{alg:offline_cpd})} \label{app:npcpd}
\begin{proof}
 For the tolerance level $\alpha < n - 1$, fix a hyperparameter $b \in \left[2,\left\lceil \frac{n-1}{\alpha} \right\rceil\right] $, where $b$ is an integer. To prove an upper bound, we can enlarge the competing candidate locations from $\mathcal K$ to $[n]$, while keeping $k^\star\in\mathcal K$. Now, we partition the interval $[1,n]$, into different exponentially growing sub-intervals in $b$, around the true change-point $k^{\star} \in [n]$
\begin{figure}[H]
  \begin{center}
     \vspace{5mm}
    \begin{tikzpicture}[xscale=1, yscale=1]

    \coordinate (One) at (0,0);
    \coordinate (Kstar) at (6,0);
    \coordinate (N) at (12,0);
    \coordinate (KstarMinusAlpha) at (4.5,0);
    \coordinate (KstarPlusAlpha) at (7.5,0);
    \coordinate (KstarMinusBAlpha) at (2.5,0);
    \coordinate (KstarPlusBAlpha) at (9.5,0);

    \draw[red, ultra thick] (One) -- (KstarMinusAlpha);
    \draw[black, ultra thick] (KstarMinusAlpha) -- (KstarPlusAlpha);
    \draw[red, ultra thick] (KstarPlusAlpha) -- (N);

    \foreach \pos/\label in {One/1, KstarMinusAlpha/k^\star-\alpha, Kstar/k^\star, KstarPlusAlpha/k^\star+\alpha, N/n} {
        \draw[black, line width=2.5pt] (\pos) ++(0,-0.25) -- ++(0,0.5);
        \node[below=12pt, black] at (\pos) {$\label$};
    }

    \foreach \pos/\label in {KstarMinusBAlpha/k^\star-b\alpha, KstarPlusBAlpha/k^\star+b\alpha} {
        \draw[red, ultra thick] (\pos) ++(0,-0.25) -- ++(0,0.5);
        \node[below=12pt, red] at (\pos) {$\label$};
    }

    \node[below=12pt, red] at (1.25, 0) {$\dots$};
    \node[below=12pt, red] at (10.75, 0) {$\dots$};

    \end{tikzpicture}
        \vspace{5mm}
    \caption{Good and {\color{red}{Bad}} sub intervals around the true change-point $k^{\star} \in [n]$.}
    \label{fig:interval}
  \end{center}
\end{figure}
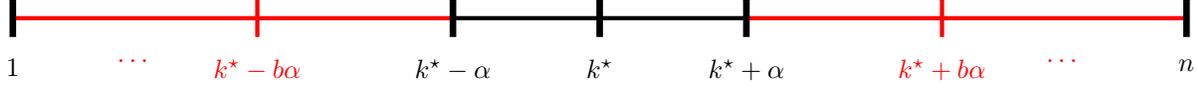
Let $[k^{\star}-\alpha,k^{\star}+\alpha]$ denote the `good interval' while we define the `bad interval' $\mathcal{R}$ in the following manner. Define $\mathcal{R}^{-}_i = [k^{\star}-b^{i}\alpha,k^{\star}-b^{(i-1)}\alpha)$ and $\mathcal{R}^{+}_i = (k^{\star}+b^{(i-1)}\alpha, k^{\star}+b^{i}\alpha]$. Let $\mathcal{R}_i = \mathcal{R}^{-}_i \cup \mathcal{R}^{+}_i $. The `bad interval' $\mathcal{R}$ is 
\begin{align*}
    \mathcal{R} = \bigcup_{i=1}^{i^{*}}\mathcal{R}_i
\end{align*}
where $i^{*}$ is the smallest integer satisfying 
\begin{align*}
    i^{*} \geq \max_{k^{\star} \in [n]}\max\left\{\log_b\left(\frac{k^{\star}-1}{\alpha}\right), \log_b\left(\frac{n-k^{\star}}{\alpha}\right) \right\} 
\end{align*}
Consequently, we have $i^{*} = \lceil\log_b\left(\frac{n-1}{\alpha}\right)\rceil \in\left[1,\lceil\log_2\left(\frac{n-1}{\alpha}\right)\rceil\right]$. 

Now our goal is to upper bound the probability of the change-point estimate $\widehat{k}$ being in the bad interval $\mathcal{R}$. We have,
\begin{align}
\mathbb{P}_{k^{\star}}\Big(|\widehat{k}-k^{\star}| >\alpha \Big) &\leq \mathbb{P}\Big[\exists \text{ } k \in \mathcal{R} \text{ } \text{  s.t. }\text{ } l(\mathcal{D},k)- l(\mathcal{D},k^{\star})\geq 0\Big]\label{eq:proberror101}\\
& \leq \sum_{i=1}^{i^{*}}\Bigg\{ \mathbb{P}\left[\max_{k \in \mathcal{R}^{-}_i}\{l(\mathcal{D},k)-l(\mathcal{D},k^{\star})\}\geq 0 \right] \notag \\ &\hspace{20mm}+ \mathbb{P}\left[\max_{k \in \mathcal{R}^{+}_i}\{l(\mathcal{D},k)-l(\mathcal{D},k^{\star})\}\geq 0 \right]  \Bigg\} \label{eq:proberror}
\end{align}
Note here that the likelihood ratios $l(\mathcal{D},k)$ are dependent on each other across different values of $k \in [1,n)$. However, the pairwise difference $l(\mathcal{D},k) - l(\mathcal{D},k+1) = \log\frac{P_1(x_k)}{P_0(x_k)}$ depends only on the $k^{th}$ data point (and its probability under $P_0$ and $P_1$). Thus, the pairwise differences are independent for all values of $k$, and identically distributed for $k<k^{\star}$ (with $P_0$) and $k\geq k^{\star}$ (with $P_1$).

\subsubsection{Proof of Theorem~\ref{thm:npcpd}: Upper Bound based on $\mathrm{I}_{\,\mathrm{ch}}$}\label{app:np_ich}

Define random variable $W'_j$, for all $ j \in [n]$
\begin{align}
    W'_j:= 
\begin{cases}
-\log\frac{P_0(X_j)}{P_1(X_j)} & ; j < k^{\star}\\
-\log\frac{P_1(X_j)}{P_0(X_j)} & ; j \geq k^{\star} 
\end{cases}
\end{align}
The variables $W'_j$ are i.i.d.\ within each segment, since the observations $X_j$ are i.i.d.\ from $P_0$ for $j<k^\star$ and from $P_1$ for $j\geq k^\star$. Now, we express the difference $l(\mathcal{D},k)-l(\mathcal{D},k^{\star})$ as a sum of random variables $W'_j$'s as follows
\begin{align}
    l(\mathcal{D},k)-l(\mathcal{D},k^{\star})= 
\begin{cases}
\sum^{k^{\star}-1}_{j=k}W'_j  & ; k < k^{\star}\\
\sum^{k-1}_{j=k^{\star}}W'_j & ; k \geq k^{\star} 
\end{cases}
\end{align}
Let us now define $U'_j := W'_{k^\star-j}$, for all $1 \leq j \leq k^\star-1$. Similarly, define $V'_j := W'_{k^\star+j-1}$, for all $1 \leq j \leq n-k^\star$. Corresponding to bad sets $\mathcal{R}^{-}$ and $\mathcal{R}^{+}$, let us define the events
\begin{align}
A_-
&:= \bigcup^{i^{*}}_{i=1} \left\{ \exists\,\, k \in \mathcal{R}^{-}_i\,\, \text{ s.t. }\,\, l(D,k)-l(D,k^{\star}) \geq 0 \right\}
\label{eq:glrt-left}\\
A_+
&:= \bigcup^{i^{*}}_{i=1} \left\{ \exists\,\, k \in \mathcal{R}^{+}_i\,\, \text{ s.t. }\,\, l(D,k)-l(D,k^{\star}) \geq 0 \right\}
\label{eq:glrt}
\end{align}
Consequently we can upper bound~\eqref{eq:proberror101} as,
\begin{align}
\mathbb{P}_{k^\star}\!\left(
|\widehat{k}-k^\star|>\alpha
\right)
&\leq
\mathbb{P}_{k^\star}(A_-\cup A_+) \notag \\
&= 1- (1- \mathbb{P}_{k^\star}(A_-))(1-\mathbb{P}_{k^\star}(A_+))\notag\\
&=1-\left(1-\mathbb{P}\left[
\max_{\alpha+1 \leq k^{\star}-k \leq k^\star-1}
\sum_{j=1}^{ k^{\star}-k} U'_j \geq 0
\right]\right)\left(1-\mathbb{P} \left[
\max_{\alpha+1 \leq  k-k^{\star} \leq n-k^\star}
\sum_{j=1}^{ k-k^{\star}} V'_j \geq 0
\right]\right),\label{eq:glrt-errr}
\end{align}
since $A_-$ and $A_{+}$ are independent events. Now, let's analyze the $\mathbb{P}(A_-)$ in Eq.~\ref{eq:glrt-errr} carefully. Let $k^{\star}-k:=m$, then define $S_m = \sum_{j=1}^{m}U'_j$. Now, let $\psi(\lambda)$ be the log-MGF function of the random variables $U'_1,U'_2,\dots U'_m$ defined as
\begin{align}
    \psi(\lambda) &= \log \mathbb{E}[e^{\lambda U'_1}]\\
    &= \log \left(\sum_{x}P_0(x)\left(\frac{P_1(x)}{P_0(x)}\right)^{\lambda}\right)\\
    &= \log \left(\sum_{x}P_0(x)^{1-\lambda}P_1(x)^{\lambda} \right)
\end{align}
For $\lambda \in (0,1)$, we have that $\psi(\lambda) < 0$. From Definition~\ref{def:ich} of Chernoff information, we see that 
\begin{align}
    \mathrm{I}_{\,\mathrm{ch}}{(P_0,P_1)} =  -\inf_{\lambda \in (0,1)}\psi(\lambda)
\end{align}
Now for any $\lambda \in (0,1)$, let $M_m:=\exp(\lambda S_m-m\psi(\lambda))$. Then, we have that 
\begin{align}
    \mathbb{P}_{k^{\star}}(A_-)=\mathbb{P}\left[\max_{\alpha+1 \leq  m \leq k^\star-1}
   {S_m\geq 0}\right]&=\mathbb{P}_{k^\star}\Big(
    \exists\, m\in\{\alpha+1,\ldots,k^\star-1\}:
    M_m\ge \exp(-m\psi(\lambda))
\Big)\label{eq:abracadabra000}\\
 &\leq \inf_{\lambda \in (0,1)}\mathbb{P}\left[\max_{\alpha+1 \leq  m \leq k^\star-1}
   {M_m \geq \exp(-(\alpha+1)\psi(\lambda)) }\right]\label{eq:abracadabra00}
\end{align}
We will now show that the sequence $(M_m, m\geq 1)$ is a non-negative martingale (consequently, a non-negative sub-martingale). Then, we will use the Doob's maximal inequality to  upper bound the RHS in Eq.~\ref{eq:abracadabra00}.

Let $\mathcal{F}_{m}:=\sigma(U'_1,\dots,U'_{m})$ be the $\sigma-$field generated by $(U'_j, j \in [m])$. Since $M_t = f(U'_1,\dots,U'_t)$, we have that $M_m$ is $\mathcal{F}_m$ measurable. Note that $M_m>0$ for all $m\geq 1$ and $\mathbb{E}(M_m)=\exp(-m\psi(\lambda))\mathbb{E}[e^{\lambda S_m}]=\exp(-m\psi(\lambda))\prod_{j=1}^{m}\mathbb{E}[e^{\lambda U_j'}]=1 <\infty$. Additionally, $M_m$ follows the recursion $M_{m+1}=M_me^{\lambda U_{m+1}'}\exp(-\psi(\lambda)).$
Therefore,
$\mathbb{E}(M_{m+1}|\mathcal{F}_m)=\mathbb{E}(M_{m}|\mathcal{F}_m)\mathbb{E}(e^{\lambda U_{m+1}'})\exp(-\psi(\lambda))=M_m$. Therefore, 
$(M_m,m \geq 1)$ is a martingale adapted to the natural filtration of $(U'_m, m \geq 1)$. 

\begin{lemma}[Doob's maximal inequality, see e.g., Theorem 35.3 in~\cite{billingsley1995probability}]\label{lemma:doob}
      Let $(H_k,k \leq n)$ be a real-valued non-negative submartingale sequence adapted to the filtration $(\mathcal{F}_k, k \leq n)$. Then for every $n\geq 1$ and $\lambda>0$, we have
\begin{align}
\mathbb{P}\left[\max_{0 \leq k \leq n} H_k \geq \lambda\right]
\leq
\frac{\mathbb{E}[H_n]}{\lambda}.
\end{align}

\end{lemma}

Therefore, fix a $\lambda \in (0,1)$, on using Doob's maximal inequality on the non-negative submartingale $(M_m,m\geq1)$, we have
\begin{align}
    \mathbb{P}\left[\max_{\alpha+1 \leq  m \leq k^\star-1}
   {M_m \geq \exp(-(\alpha+1)\psi(\lambda)) }\right] &\leq \frac{\mathbb{E}[M_{k^\star-1}]}{\exp(-(\alpha+1)\psi(\lambda))}\\
   &=\exp((\alpha+1)\psi(\lambda))
\end{align}
Now, Eq.~\ref{eq:abracadabra00} reduces to
\begin{align}
  \mathbb{P}_{k^{\star}}(A_-)=    \mathbb{P}\left[\max_{\alpha+1 \leq  m \leq k^\star-1}
   {S_m\geq  0}\right] &\leq \inf_{\lambda \in (0,1)}\exp\left((\alpha+1)\psi(\lambda)\right)\\
   &=\exp\left((\alpha+1)\inf_{\lambda \in (0,1)}\psi(\lambda))\right)\\
   &=\exp\left(-(\alpha+1) \mathrm{I}_{\,\mathrm{ch}}{(P_0,P_1)}\right)\label{eq:last}
\end{align}
In a similar way to the above analysis, we can also show that
\begin{align}
\mathbb{P}_{k^{\star}}(A_+) \leq  \exp\left(-(\alpha+1) \mathrm{I}_{\,\mathrm{ch}}{(P_0,P_1)}\right) \label{eq:last2}
\end{align}

On substituting~\eqref{eq:last} and~\eqref{eq:last2} in~\eqref{eq:glrt-errr} and taking supremum over all true change-point locations $(k^{\star})$, we get our final result
\begin{align}
  \sup_{k^{\star}\in \mathcal{K}}  \mathbb{P}_{k^{\star}}\Big(|\widehat{k}-k^{\star}| >\alpha \Big) &\leq 2 {\exp\left(-{(\alpha+1) \mathrm{I}_{\,\mathrm{ch}}{(P_0,P_1)}}\right)} \left(1-\frac{\exp\left(-{(\alpha+1) \mathrm{I}_{\,\mathrm{ch}}{(P_0,P_1)}}\right)}{2}\right)\label{eq:new_ich}
\end{align}
This completes our proof for Theorem~\ref{thm:npcpd}.

\subsubsection{Proof of Lemma~\ref{cor:npcpd}: Upper Bound based on $C$ and $S$}\label{app:np_cs}
Under the assumption that $S <\infty$ (finite alphabet case), we will use the maximal Azuma-Hoeffding's inequality (Lemma~\ref{lemma:azuma}) to upper bound the probabilities in Eq.~\ref{eq:proberror}. Therefore, we would like to express the pairwise differences  of likelihood ratios as a sum of zero mean i.i.d random variables. Define random variable $W_j$, for all $ j \in [n]$ as
\begin{align}
    W_j:= 
\begin{cases}
-\log\frac{P_0(X_j)}{P_1(X_j)} + \mathrm{D}_{\mathrm{KL}}(P_0\|P_1)  & ; j < k^{\star}\\
-\log\frac{P_1(X_j)}{P_0(X_j)} + \mathrm{D}_{\mathrm{KL}}(P_1\|P_0)  & ; j \geq k^{\star}.
\end{cases}
\end{align}
The variables $W_j$ are independent, identically distributed within
each segment, and have mean zero. Now, we express the difference $l(\mathcal{D},k)-l(\mathcal{D},k^{\star})$ as a sum of random variables $W_j$'s as follows
\begin{align}
    l(\mathcal{D},k)-l(\mathcal{D},k^{\star})= 
\begin{cases}
\sum^{k^{\star}-1}_{j=k}W_j - (k^{\star}-k)\mathrm{D}_{\mathrm{KL}}(P_0\|P_1)  & ; k < k^{\star}\\
\sum^{k-1}_{j=k^{\star}}W_j - (k-k^{\star})\mathrm{D}_{\mathrm{KL}}(P_1\|P_0)  & ; k \geq k^{\star} 
\end{cases}
\end{align}
Now, we can re-write the Eq.~\ref{eq:proberror} as follows
\begin{align}
\mathbb{P}_{k^{\star}}\Big(|\widehat{k}-k^{\star}| >\alpha \Big) 
\leq \sum_{i=1}^{i^{*}}   \hspace{2mm} &\mathbb{P}\left[\max_{k \in \mathcal{R}^{-}_i}\left\{\sum^{k^{\star}-1}_{j=k}W_j - (k^{\star}-k)\mathrm{D}_{\mathrm{KL}}(P_0\|P_1)\right\}\geq 0\right]  \notag \\ & \hspace{2mm}+\mathbb{P}\left[\max_{k \in \mathcal{R}^{+}_i}\left\{\sum^{k^{}-1}_{j=k^{\star}}W_j - (k-k^{\star})\mathrm{D}_{\mathrm{KL}}(P_1\|P_0)\right\}\geq 0 \right] \label{eq:proberror1}
\end{align}
Note that for $k \in \mathcal{R}^{-}_i$, we have that $b^{(i-1)}\alpha<k^{\star}-k\leq b^{i}\alpha$. Similarly, for $k \in \mathcal{R}^{+}_i$, we have that $b^{(i-1)}\alpha<k-k^{\star}\leq b^{i}\alpha$. 
Therefore,
\begin{align}
\mathbb{P}_{k^{\star}}\Big(|\widehat{k}-k^{\star}| >\alpha \Big) 
&\leq \sum_{i=1}^{i^{*}} \hspace{2mm}\mathbb{P}\left[\max_{k \in \mathcal{R}^{-}_i}
   {\sum^{k^{\star}-1}_{j=k}W_j > b^{(i-1)}\alpha \mathrm{D}_{\mathrm{KL}}(P_0\|P_1)}\right] \notag \\
&\quad \hspace{10mm} + \hspace{2mm}\mathbb{P}\left[\max_{k \in \mathcal{R}^{+}_i}
   {\sum^{k-1}_{j=k^{\star}}W_j > b^{(i-1)}\alpha \mathrm{D}_{\mathrm{KL}}(P_1\|P_0)}\right] \label{eq:proberror11} \\
&\leq \sum_{i=1}^{i^{*}} \hspace{2mm} \Bigg\{\mathbb{P}\left[\max_{k \in \mathcal{R}^{-}_i}
   {\sum^{k^{\star}-1}_{j=k}W_j > b^{(i-1)}\alpha C}\right]\notag\\
  &\quad \hspace{15mm} + \mathbb{P}\left[\max_{k \in \mathcal{R}^{+}_i}
   {\sum^{k-1}_{j=k^{\star}}W_j > b^{(i-1)}\alpha C}\right]\Bigg\}
\end{align}
Recall that $C=\min\{\mathrm{D}_{\mathrm{KL}}(P_0\|P_1),\mathrm{D}_{\mathrm{KL}}(P_1\|P_0)\}$. For $j<k^{\star}$, let $U_1:= W_{k^{\star}-1}$, $U_2:= W_{k^{\star}-2}$,$\dots$,$U_{k^{\star}-k}:= W_{k}$. Similarly, for $j\geq k^{\star}$, let $V_1:=W_{k^{\star}}$, $V_2:= W_{k^{\star}+1}$,$\dots$,$V_{k-k^{\star}}:= W_{k-1}$. For each $i$, define
$N_i^- := \min\{b^i\alpha,k^\star-1\}$, and 
$N_i^+ := \min\{b^i\alpha,n-k^\star\}$. Events involving a maximum over an empty index set have probability
zero. The concentration bounds below are applied only when the
corresponding $N_i^\pm$ is positive. Therefore, we have
\begin{align}
\mathbb{P}_{k^{\star}}\Big(|\widehat{k}-k^{\star}| >\alpha \Big) 
&\leq  \sum_{i=1}^{i^{*}} \hspace{2mm} \Bigg\{\mathbb{P}\left[\max_{k \in \mathcal{R}^{-}_i}
   {\sum^{k^{\star}-k}_{j=1}U_j > b^{(i-1)}\alpha C}\right]\notag\\
&\quad \hspace{20mm}   + \mathbb{P}\left[\max_{k \in \mathcal{R}^{+}_i}
   {\sum^{k-k^{\star}}_{j=1}V_j > b^{(i-1)}\alpha C}\right]\Bigg\}\label{eq:beforeazuma}\\
   &\leq  \sum_{i=1}^{i^{*}} \hspace{2mm} \Bigg\{\mathbb{P}\left[\max_{1\le m\le N_i^-}\sum_{j=1}^{m}U_j > b^{(i-1)}\alpha C\right]\notag\\
&\quad \hspace{20mm}   
   + \mathbb{P}\left[\max_{1\le m\le N_i^+}\sum_{j=1}^{m}V_j > b^{(i-1)}\alpha C\right]\Bigg\}\label{eq:beforeazuma2}
\end{align}

For a fixed $i \leq i^{*}$, let's analyze the first term in the Eq.~\ref{eq:beforeazuma2}. $U_j$'s are zero mean i.i.d random variables supported on a bounded interval $[a_1,a_2]$ where $a_1:= \min_{x \in \mathcal{X}}\log\frac{P_1(x)}{P_0(x)} +\mathrm{D}_{\mathrm{KL}}(P_0\|P_1)$ and $a_2:=\max_{x \in \mathcal{X}}\log\frac{P_1(x)}{P_0(x)} +\mathrm{D}_{\mathrm{KL}}(P_0\|P_1)$. 

Let $S_t:=\sum_{j=1}^{t}U_j$ denote the sum of $t$ random variables. Let $\mathcal F_t:=\sigma(U_1,\ldots,U_t)$ for $1\le t\leq N_i^-$ be the $\sigma-$field generated by $U_1,\ldots,U_t$. Since $S_t = f(U_1,\dots,U_t)$, we have that $S_t$ is $\mathcal{F}_t$ measurable. Note that $\mathbb{E}(|S_t|)=\mathbb{E}(|U_1+\dots+U_t|)\leq \sum_{j=1}^{t}\mathbb{E}(|U_j|)\leq t.S <\infty$. Additionally, $\mathbb{E}(S_{t+1}|\mathcal{F}_t)=\mathbb{E}(S_{t}+U_{t+1}|\mathcal{F}_t)=\mathbb{E}(S_{t}|\mathcal{F}_t)+ \mathbb{E}(U_{t+1}|\mathcal{F}_t)=S_t$. Therefore, 
$(S_t,t\in[N_i^-])$ is a martingale adapted to the natural filtration of $(U_t, t\in[N_i^-])$. Furthermore, for all $t$, we have that $|S_{t+1}-S_t| \leq a_2-a_1 \leq S$. Thus, $(S_t,t\in[N_i^-])$ is a bounded difference martingale.

We will now use the maximal Azuma-Hoeffding inequality (stated below) to upper bound the tail probability of the bounded difference martingale $(S_t,t\in[N_i^-])$.

\begin{lemma}[{Azuma-Hoeffding's inequality}, see e.g.~\cite{raginsky}]\label{lemma:azuma}
   Let $(M_k,k \leq n)$ be a real-valued bounded difference martingale sequence adapted to the filtration $(\mathcal{F}_k, k \leq n)$, such that $M_{k}-M_{k-1} \in [a_k,b_k]$, for every $k \in [n]$. Then for every $n\geq 1$ and $t>0$, we have
   \begin{align}
       \mathbb{P}\left[ \max_{k \in [n]}(M_k-M_0) > t\right] &\leq \exp\left(\frac{-2t^2}{\sum_{k=1}^{n}(b_k-a_k)^2}\right)\label{eq:azuma1}\\
       \mathbb{P}\left[\min_{k \in [n]} (M_k-M_0) < -t\right] &\leq \exp\left(\frac{-2t^2}{\sum_{k=1}^{n}(b_k-a_k)^2}\right)\label{eq:azuma2}
   \end{align}
\end{lemma}
Assume, $S_0 =U_0=0$. Now, using Eq.~\ref{eq:azuma1}, we have that
\begin{align}
\mathbb{P}_{k^\star}\left(\max_{1\le m\le N_i^-}\sum_{j=1}^{m}U_j> b^{i-1}\alpha C
\right)&\leq\exp\left(-\frac{2b^{2(i-1)}\alpha^2 C^2}{N_i^-S^2}\right)\\
&\leq \exp\left(-\frac{2b^{2(i-1)}\alpha C^2}{b^iS^2}\right)\\
&=\exp\left(-\frac{2b^{i-2}\alpha C^2}{S^2}\right).\label{eq:afterazuma1}
\end{align}
Similarly, applying the upper-tail inequality~\ref{eq:azuma1} to the partial
sums of $V_j$ and using $N_i^+\le b^i\alpha$, we obtain
\begin{align}
\mathbb{P}_{k^\star}\left(\max_{1\le m\le N_i^+}\sum_{j=1}^{m}V_j> b^{i-1}\alpha C\right)
&\leq\exp\left(-\frac{2b^{2(i-1)}\alpha^2 C^2}{N_i^+S^2}\right)
\notag\\
&\leq\exp\left(-\frac{2b^{i-2}\alpha C^2}{S^2}\right).\label{eq:afterazuma2}
\end{align}
Substituting Eq.~\ref{eq:afterazuma1} and Eq.~\ref{eq:afterazuma2} back into Eq.~\ref{eq:beforeazuma2}, we have the following result on the error probability
\begin{align}
\mathbb{P}_{k^{\star}}\Big(|\widehat{k}-k^{\star}| >\alpha \Big) &\leq 2\sum_{i=1}^{i^{*}} \hspace{2mm}  {\exp\left(-\frac{2\alpha C^2}{S^2}\right)^{b^{(i-2)}}}
\end{align}
Choosing $b=2$ in the above equation gives us the final bound
\begin{align}
\mathbb{P}_{k^{\star}}\Big(|\widehat{k}-k^{\star}| >\alpha \Big) \leq 
2\sum_{i=1}^{i^{*}} \hspace{2mm}  {\exp\left(-\frac{\alpha C^2}{S^2}\right)^{2^{(i-1)}}} \label{eq:artist}
\end{align}
where $i^{*} = \lceil\log_2\left(\frac{n-1}{\alpha}\right)\rceil$.

Let $r=\exp(-\frac{\alpha C^2}{S^2})$. We can upper bound the RHS in Eq.~\ref{eq:artist} to obtain a closed-form expression in terms of the size of the dataset $(n)$. Further, taking supremum over all possible true change-points, we have that
\begin{align}
\sup_{k^{\star}\in \mathcal{K}}\mathbb{P}_{k^{\star}}\Big(|\widehat{k}-k^{\star}| >\alpha \Big) 
\leq 2\sum_{i=1}^{i^{*}} \hspace{2mm}  {r^{2^{(i-1)}}}&\leq 2\sum_{m=1}^{2^{i^{*}-1}} \hspace{2mm}  {r^{m}}\\
&\leq \frac{2r(1-r^M)}{1-r}\label{eq:wer}
\end{align}
where $M:=\lfloor\frac{n-1}{\alpha}\rfloor$ and we have that $2^{i^{*}-1} \leq M$. 
Furthermore, as $n \rightarrow \infty$,
\begin{align}
\sup_{k^{\star}\in \mathcal{K}}\mathbb{P}_{k^{\star}}\Big(|\widehat{k}-k^{\star}| >\alpha \Big) 
    \leq \frac{2r}{1-r} = \frac{2\exp(-\alpha C^2/S^2)}{1-\exp(-\alpha C^2/S^2)}
\end{align}
It is interesting to note that for fixed $\alpha$ and $(P_0,P_1)$, the upper bound in~\eqref{eq:wer} is
non-decreasing in $n$ and converges to $2r/(1-r)$ as $n\to\infty$. 

This completes our proof for Lemma~\ref{cor:npcpd}.\end{proof}

\subsection{Comparison with existing accuracy bounds}\label{app:npcomp}
It was shown in~\cite{zhang_cpd} [Theorem $10$] that
\begin{align}
    \beta \leq 8 \min\left(\sum_{i\geq 1}\exp\left(\frac{-2^{i-2}\alpha C^2}{ S^2}\right);\sum_{i\geq 1}\exp\left(\frac{-2^{i-3}\alpha C^2_M}{C_M+8} \right)\right)\label{eq:zhanga}
\end{align}
where $C_M = \min\left\{D_{\mathrm{KL}}\left(P_0\|\frac{P_0+P_1}{2}\right); D_{\mathrm{KL}}\left(P_1\|\frac{P_0+P_1}{2}\right)\right\}$. The upper bound in Eq.~\ref{eq:zhanga} is meaningfully defined only under the assumption that $\alpha> \min \left\{2\ln 2 \frac{S^2}{C^2};4\ln 2 \;\frac{C_M+8}{C_M^2}\right\}$ (see the transition from Eq.(10) to Eq.(11) and Eq. (16) to Eq. (17) in~\cite{zhang_cpd}). For ease of notation, we will refer to the first term inside the minimization in Eq.~\ref{eq:zhanga} as Bound C and the second term as Bound D.

The derivation of the bound in Eq.~\ref{eq:zhanga} in~\cite{zhang_cpd} is based on the use of Ottaviani's inequality. However, our bound does not rely on this concentration inequality, instead we take a martingale-based approach and provide a tighter bound using Doob's maximal inequality for Theorem~\ref{thm:npcpd} and maximal Azuma-Hoeffding's inequality for Lemma~\ref{cor:npcpd}. 

Therefore, our analyses yield meaningful upper bounds even for significantly lower values of $\alpha$. Moreover, over the range of $\alpha$ for which all bounds are well defined, our bound in Theorem~\ref{thm:npcpd} provides provably better accuracy guarantees than the one in Lemma~\ref{cor:npcpd} as well as than the ones in Eq.~\ref{eq:zhanga}. We further provide an empirical comparison by plotting the bounds for several synthetic datasets in Figures~\ref{fig:theory1},~\ref{fig:theory2} and~\ref{fig:theory3}.

We fix the block size $n=2000$, the true change-point $k^{\star}=1000$ and compare the upper bounds for different choices of pre- and post-change distributions $(P_0,P_1)$. The empirical curve is computed via Monte Carlo simulations by repeatedly sampling data from $(P_0,P_1)$, estimating the change-point using Algorithm~\ref{alg:offline_cpd}, and repeating this procedure $10,000$ times. 
\vspace*{8mm}
\begin{figure}
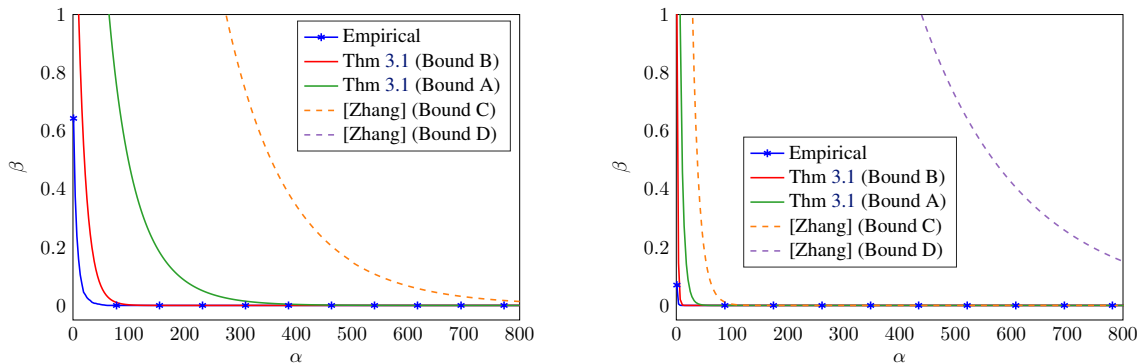

\begin{center}
        \input{figures/np_bernoulli_1}
                \hspace*{6mm}
        \input{figures/np_bernoulli_2}
                \caption{Comparison of the empirical error probability $\beta$ for Algorithm~\ref{alg:offline_cpd} and theoretical {upper bounds} on $\beta$ in the non-private setting, plotted as functions of $\alpha$. We consider binary (Bernoulli) distributions. \textbf{Left:} $P_0 \sim \mathrm{Ber}(0.1)$, $P_1 \sim \mathrm{Ber}(0.4)$. \textbf{Right:} $P_0 \sim \mathrm{Ber}(0.1)$, $P_1 \sim \mathrm{Ber}(0.95)$.}
        \label{fig:theory1}
        \vspace*{4mm}
                \end{center}

\end{figure}
\begin{figure}
\begin{center}
        \input{figures/np_tpoisson_1}  
                \hspace*{6mm}
        \input{figures/np_tpoisson_2} 
                \caption{Comparison of the empirical error probability $\beta$ for Algorithm~\ref{alg:offline_cpd} and theoretical {upper bounds} on $\beta$ in the non-private setting, plotted as functions of $\alpha$. We consider truncated Poisson  distributions $\mathrm{TPois}(\lambda,m)$ with truncation parameter $m=10$. \textbf{Left:} $P_0 \sim \mathrm{TPois}(1,10)$, $P_1 \sim \mathrm{TPois}(4,10)$. \textbf{Right:} $P_0 \sim \mathrm{TPois}(1,10)$, $P_1 \sim \mathrm{TPois}(10,10)$.}
        \label{fig:theory2}
        \vspace*{8mm}
\input{figures/np_gaussian_1}        \hspace*{6mm}
        \input{figures/np_gaussian_2}  
                \caption{Comparison of the empirical error probability $\beta$ for Algorithm~\ref{alg:offline_cpd} and theoretical {upper bounds} on $\beta$ in the non-private setting, plotted as functions of $\alpha$. We consider the Gaussian distributions (continuous alphabet). Lemma~\ref{cor:npcpd} and Bound C does not hold as $S\rightarrow \infty$.  \textbf{Left:} $P_0 \sim \mathcal{N}(0,5)$, $P_1 \sim \mathcal{N}(5,5)$. \textbf{Right:} $P_0 \sim \mathcal{N}(0,5)$, $P_1 \sim \mathcal{N}(10,5)$.}
        \label{fig:theory3}
\end{center}
\end{figure}
\clearpage
\section{Characterization of SDPI Coefficients}

\subsection{SDPI coefficients for $\infty$-order Rényi divergences}\label{app:rd_sdpi}

\begin{theorem}\label{thm:rdccbd}
For Rényi divergence of order $\rho = \infty$ i.e., $\mathrm{D}_{\infty}(\cdot \|\cdot)$, the SDPI coefficient of a channel $W$, $\eta_{\infty}(W)$, is achieved by binary input distributions.
\end{theorem}
\begin{proof}
We restrict the supremum to pairs $(P_0, P_1)$ such that $P_0 \ll P_1$ i.e., $P_1(x) =0 \implies P_0(x)=0$. The pairs violating absolute continuity have infinite input Rényi divergence and cannot improve the SDPI ratio. Therefore this restriction is without loss of generality.

Recall from the definition of SDPI coefficient that
\begin{align*}
    \eta_{\infty}(W) := \sup_{P_0,P_1 } \frac{\mathrm{D}_{\infty}(Q_0\|Q_1)}{\mathrm{D}_{\infty}(P_0\|P_1)}
\end{align*}
where $\mathrm{D}_{\infty}(P_0\|P_1):=\log(H_{\infty}(P_0\|P_1))$ and $H_{\infty}(P_0\|P_1)=\max_{x\in \mathcal{X}}\frac{P_0(x)}{P_1(x)}$.
All supremum over input pairs are restricted to $0<\mathrm D_\infty(P_0\|P_1)<\infty$.
The case $\eta_\infty(W)=0$ is immediate, so assume $\eta_\infty(W)>0$. From the definition of $\eta_{\infty}(W)$, with $\lambda \in (0,1)$,
\begin{align}
\eta_{\infty}(W) &= \sup\left\{ \lambda:\sup_{P_0,P_1}\frac{\mathrm{D}_{\infty}(Q_0\|Q_1)}{\mathrm{D}_{\infty}(P_0\|P_1)} > \lambda \right\}\\
&=\sup\left\{\lambda: \sup_{P_0,P_1}\{\mathrm{D}_{\infty}(Q_0\|Q_1) - \lambda \mathrm{D}_{\infty}(P_0\|P_1)\}> 0 \right\}\\
&=\sup\left\{\lambda: \log\left(\sup_{P_0,P_1}\frac{H_{\infty}(Q_0\|Q_1)}{H_{\infty}(P_0\|P_1)^{\lambda}}\right)> 0 \right\}\\
&=\sup\left\{\lambda: \sup_{P_0,P_1}\{H_{\infty}(Q_0\|Q_1)-H_{\infty}(P_0\|P_1)^{\lambda}\}> 0 \right\}
\end{align}
Now, let $L^{\infty}_{\lambda}(P_0,P_1):=H_{\infty}(Q_0\|Q_1)-H_{\infty}(P_0\|P_1)^{\lambda}$. Therefore,
\begin{align*}
    \eta_{\infty}(W) = \sup\left\{\lambda : \sup_{P_0,P_1 } L^{\infty}_{\lambda}(P_0,P_1)> 0\right\}
\end{align*}
We will now show that for every $(P_0,P_1)$, there exists a $(\hat{P_0},\hat{P_1})$ supported on the binary set, such that $L^{\infty}_{\lambda}(\hat{P_0},\hat{P_1})\geq L^{\infty}_{\lambda}(P_0,P_1)$.
Given any $P_0, P_1 \in \Delta(\mathcal{X})$, define the set $\mathcal{S}(P_0,P_1)$ as follows
\begin{align}
    \mathcal{S}(P_0,P_1):=\left\{ \hat{P_1} \in \Delta({\hat{\mathcal{X}}}) : \hat{\mathcal{X}}\subseteq \mathcal{X} \text{ and }\sum_{x \in \mathcal{X}}\frac{P_0(x)}{P_1(x)}\hat{P_1}(x) =1 \right\} \label{eq:polyset}
\end{align}
$\mathcal{S}(P_0,P_1)$ is a convex set. It is also bounded (under any $\ell_{p\geq 1}-$norm) and closed, i.e., a compact set.  Furthermore, $\mathcal{S}(P_0,P_1)$ is an intersection of a probability simplex and a hyperplane, thus its extreme points are always supported on at most two atoms.

Consider a function $g:\mathcal{S}(P_0,P_1) \rightarrow \mathbb{R}$ such that $g^{\infty}(\hat{P_1}):=L^{\infty}_{\lambda}(\frac{P_0}{P_1}\hat{P_1},\hat{P_1})$. Note that $P_1 \in \mathcal{S}(P_0,P_1)$. Therefore, we have
\begin{align}
    \max_{\hat{P_1} \in \mathcal{S}} g^{\infty}(\hat{P_1}) \geq g^{\infty}(P_1) = L^{\infty}_{\lambda}(P_0,P_1)
\end{align}
Next, we will show that $\arg\max_{\hat{P_1} \in \mathcal{S}}g^{\infty}(\hat{P_1})$ is a binary distribution (not necessarily unique) i.e., an extreme point of $\mathcal{S}$, which proves our result. 

Let $\hat{P_0} = \frac{P_0}{P_1}\hat{P}_1$. Recall that $g^{\infty}(\hat{P_1})=L^{\infty}_{\lambda}(\hat{P_0},\hat{P_1})= H_{\infty}(\hat{Q_0}\|\hat{Q_1})- H_{\infty}(\hat{P_0}\|\hat{P_1})^{\lambda}$. Thus,
\begin{align}
    g^{\infty}(\hat{P_1})&=\max_{y \in \mathcal{Y}}\frac{\sum_{x \in \mathcal{\hat{X}}}\frac{P_0(x)}{P_1(x)}\hat{P_1}(x)W(y|x)  }{\sum_{x \in \hat{\mathcal{X}}}\hat{P_1}(x)W(y|x)} - \left(\max_{x \in \hat{\mathcal{X}}}\frac{\hat{P_0}(x)}{\hat{P_1}(x)}\right)^{\lambda}\\
    & = \max_{y \in \mathcal{Y}}\frac{\sum_{x \in \mathcal{\hat{X}}}\frac{P_0(x)}{P_1(x)}\hat{P_1}(x)W(y|x)  }{\sum_{x \in \hat{\mathcal{X}}}\hat{P_1}(x)W(y|x)} - \left(\max_{x \in \hat{\mathcal{X}}}\frac{{P_0}(x)}{{P_1}(x)}\right)^{\lambda}\\
    & = \max_{y \in \mathcal{Y}}\frac{\sum_{x \in \mathcal{\hat{X}}}\frac{P_0(x)}{P_1(x)}\hat{P_1}(x)W(y|x)  }{\sum_{x \in \hat{\mathcal{X}}}\hat{P_1}(x)W(y|x)} - \max_{x \in \hat{\mathcal{X}}}\left(\frac{{P_0}(x)}{{P_1}(x)}\right)^{\lambda}\label{eq:ql}\\
    &:= \max_{y \in \mathcal{Y}} f_y(\hat{P_1}) \label{eq:fyqc}
\end{align}

To complete our proof, we will now show that for every $y \in \mathcal{Y}$, $f_y(\hat{P_1})$ is a quasi-convex function of $\hat{P_1}$. Let $\hat{P_1} \in \Delta(\mathcal{\hat{X}}_1)$ and $\hat{P_2} \in \Delta(\mathcal{\hat{X}}_2)$ be any two distributions, and let $\theta \in (0,1)$. Then, we know that $\theta\hat{P_1} + (1-\theta)\hat{P_2} \in \Delta(\mathcal{\hat{X}}_1 \text{ } \cup \text{ } \mathcal{\hat{X}}_2)$. Furthermore, let 
\begin{align}
    \tau_y:=\max\left\{ \frac{\sum_{x \in \mathcal{\hat{X}}_1}\frac{P_0(x)}{P_1(x)}\hat{P_1}(x)W(y|x)  }{\sum_{x \in \hat{\mathcal{X}}_1}\hat{P_1}(x)W(y|x)}, \frac{\sum_{x \in \mathcal{\hat{X}}_2}\frac{P_0(x)}{P_1(x)}\hat{P_2}(x)W(y|x)  }{\sum_{x \in \hat{\mathcal{X}}_2}\hat{P_2}(x)W(y|x)} \right\}
\end{align}

Therefore, 
\begin{align}
    f_y(\theta\hat{P_1} + (1-\theta)\hat{P_2})\notag\\& \hspace{-15mm}= \frac{\sum_{x \in \mathcal{\hat{X}}_1 \text{ } \cup \text{ } \mathcal{\hat{X}}_2}\frac{P_0(x)}{P_1(x)}(\theta\hat{P_1} + (1-\theta)\hat{P_2})(x)W(y|x)  }{\sum_{x \in \mathcal{\hat{X}}_1 \text{ } \cup \text{ } \mathcal{\hat{X}}_2}(\theta\hat{P_1} + (1-\theta)\hat{P_2})(x)W(y|x)}-\max_{x \in \mathcal{\hat{X}}_1 \text{ } \cup \text{ } \mathcal{\hat{X}}_2}\left(\frac{{P_0}(x)}{{P_1}(x)}\right)^{\lambda}\\
    &\hspace{-15mm}=\frac{\sum_{x \in \mathcal{\hat{X}}_1 \text{ } \cup \text{ } \mathcal{\hat{X}}_2}\frac{P_0(x)}{P_1(x)}(\theta\hat{P_1} + (1-\theta)\hat{P_2})(x)W(y|x)  }{\sum_{x \in \mathcal{\hat{X}}_1 \text{ } \cup \text{ } \mathcal{\hat{X}}_2}(\theta\hat{P_1} + (1-\theta)\hat{P_2})(x)W(y|x)}\\
    &\hspace{30mm}-\max\left\{\max_{x \in \mathcal{\hat{X}}_1}\left(\frac{{P_0}(x)}{{P_1}(x)}\right)^{\lambda},\max_{x \in \mathcal{\hat{X}}_2}\left(\frac{{P_0}(x)}{{P_1}(x)}\right)^{\lambda}\right\}\\
    &\hspace{-15mm}=\frac{\theta\sum_{x \in \mathcal{\hat{X}}_1 \text{ } \cup \text{ } \mathcal{\hat{X}}_2}\frac{P_0(x)}{P_1(x)}\hat{P_1}(x)W(y|x) +(1-\theta)\sum_{x \in \mathcal{\hat{X}}_1 \text{ } \cup \text{ } \mathcal{\hat{X}}_2}\frac{P_0(x)}{P_1(x)}\hat{P_2}(x)W(y|x) }{\theta\sum_{x \in \mathcal{\hat{X}}_1 \text{ } \cup \text{ } \mathcal{\hat{X}}_2}\hat{P_1}(x)W(y|x) +(1-\theta)\sum_{x \in \mathcal{\hat{X}}_1 \text{ } \cup \text{ } \mathcal{\hat{X}}_2}\hat{P_2}(x)W(y|x) }\notag\\
   & \hspace{30mm}-\max\left\{\max_{x \in \mathcal{\hat{X}}_1}\left(\frac{{P_0}(x)}{{P_1}(x)}\right)^{\lambda},\max_{x \in \mathcal{\hat{X}}_2}\left(\frac{{P_0}(x)}{{P_1}(x)}\right)^{\lambda}\right\}\\
   &\hspace{-15mm}\leq \frac{\theta \tau_y\sum_{x \in \mathcal{\hat{X}}_1 \text{ } \cup \text{ } \mathcal{\hat{X}}_2}\hat{P_1}(x)W(y|x) +(1-\theta)\tau_y\sum_{x \in \mathcal{\hat{X}}_1 \text{ } \cup \text{ } \mathcal{\hat{X}}_2}\hat{P_2}(x)W(y|x) }{\theta\sum_{x \in \mathcal{\hat{X}}_1 \text{ } \cup \text{ } \mathcal{\hat{X}}_2}\hat{P_1}(x)W(y|x) +(1-\theta)\sum_{x \in \mathcal{\hat{X}}_1 \text{ } \cup \text{ } \mathcal{\hat{X}}_2}\hat{P_2}(x)W(y|x) }\notag\\
   & \hspace{30mm}-\max\left\{\max_{x \in \mathcal{\hat{X}}_1}\left(\frac{{P_0}(x)}{{P_1}(x)}\right)^{\lambda},\max_{x \in \mathcal{\hat{X}}_2}\left(\frac{{P_0}(x)}{{P_1}(x)}\right)^{\lambda}\right\}\\
    &\hspace{-15mm}= \tau_y-\max\left\{\max_{x \in \mathcal{\hat{X}}_1}\left(\frac{{P_0}(x)}{{P_1}(x)}\right)^{\lambda},\max_{x \in \mathcal{\hat{X}}_2}\left(\frac{{P_0}(x)}{{P_1}(x)}\right)^{\lambda}\right\}\\
     &\hspace{-15mm}\leq \max\Bigg\{ \frac{\sum_{x \in \mathcal{\hat{X}}_1}\frac{P_0(x)}{P_1(x)}\hat{P_1}(x)W(y|x)  }{\sum_{x \in \hat{\mathcal{X}}_1}\hat{P_1}(x)W(y|x)}-\max_{x \in \mathcal{\hat{X}}_1}\left(\frac{{P_0}(x)}{{P_1}(x)}\right)^{\lambda},\notag \\
     &\hspace{25mm}\frac{\sum_{x \in \mathcal{\hat{X}}_2}\frac{P_0(x)}{P_1(x)}\hat{P_2}(x)W(y|x)  }{\sum_{x \in \hat{\mathcal{X}}_2}\hat{P_2}(x)W(y|x)}-\max_{x \in \mathcal{\hat{X}}_2}\left(\frac{{P_0}(x)}{{P_1}(x)}\right)^{\lambda} \Bigg\}\\
     &\hspace{-15mm}=\max\left\{f_y(\hat{P_1}),f_y(\hat{P_2})\right\}
\end{align}

Thus, $f_y(\hat{P_1})$ is quasi-convex in  $\hat{P_1}$, for every $y \in \mathcal{Y}$. Furthermore, the maximum of quasi-convex functions is also quasi-convex. Thus, the function $g^{\infty}(\hat{P_1}) = \max_{y}f_y(\hat{P_1})$ is a quasi-convex function of $\hat{P_1}$.

Since, a quasi-convex function defined on a compact and convex set achieves it's maximum at an extreme point (not necessarily unique)~\cite{quasi-opt}, it directly follows that $\arg\max_{\hat{P_1} \in \mathcal{S}}g^{\infty}(\hat{P_1})$ is supported on at most two atoms (extreme points of $\mathcal{S}(P_0,P_1)$). This completes the proof.
\end{proof}

Next, using our result in Thm.~\ref{thm:rdccbd}, we compute the SDPI coefficient for infinite-order Rényi divergence, for a $q-$ary symmetric channel (defined below).

\begin{definition}[$q$-ary Symmetric Channel]\label{def:emk}
Let $\mathcal{X}$ and $\mathcal{Y}$ be finite alphabets with
$|\mathcal{X}|=|\mathcal{Y}|=q$. A channel $W_{Y|X}:\mathcal{X}\to\mathcal{Y}$ is said to be $q$-ary symmetric if, for some $u,v \in [0,1]$, we have $W_{Y|X}(y|x) = v\,\cdot\mathbf{1}\{y=x\} + u\,\cdot\mathbf{1}\{y\ne x\}$, where $v=1-(q-1)u$.
\end{definition}

\begin{theorem}\label{thm:rdccbd-value}
    Let $W$ be a $q-$ary symmetric channel. Then, the SDPI coefficient of $W$ for Rényi divergence of order infinity is $\eta_{\infty}(W) = ({|v-u|}/\max\{u,v\})$.
Furthermore, it is achieved by almost degenerate binary distributions (i.e., $\text{Bernoulli }(p)$ such that $p\rightarrow 0 $ or $p \rightarrow 1$).
\end{theorem}
\begin{proof}
From Theorem~\ref{thm:rdccbd}, we already know that the input distributions $P_0$ and $P_1$ supported on at most binary atoms achieve the SDPI coefficient.

Let $x_1$, $x_2 \in \mathcal{X}$ be such that $p_k(x_1), p_k(x_2)>0$ for $k \in \{0,1\}$. Thus, 
\begin{align}
    P_0 = [0, \dots,p_0(x_1),\dots,p_0(x_2),\dots,0] \notag\\
    P_1 = [0, \dots,p_1(x_1),\dots,p_1(x_2),\dots,0]
\end{align}
For a $q-$ary symmetric channel $W$, we have
 \[
W =
\begin{bmatrix}
v & u &  \cdots & u \\
u & v& \cdots & u \\
\vdots & \vdots & \ddots & \vdots \\
u  & \cdots & \cdots & v
\end{bmatrix}_{q\times q}
\]
where $v=1-(q-1)u$. Therefore, the corresponding output distributions $Q_0$ and $Q_1$ through this Markov kernel $W$ are
\begin{align}
    Q_0 = [u, \dots,p_0(x_1)v+p_0(x_2)u,\dots,p_0(x_2)v+p_0(x_1)u,\dots,u] \notag\\
    Q_1 = [u, \dots,p_1(x_1)v+p_1(x_2)u,\dots,p_1(x_2)v+p_1(x_1)u,\dots,u]
\end{align}
We know that
\begin{align}
    \eta_{\infty}(W) &= \frac{\mathrm{D}_{\infty}(Q_0\|Q_1)}{\mathrm{D}_{\infty}(P_0\|P_1)}
    =\frac{\log\left(\max\left(\frac{Q_0(x_1)}{Q_1(x_1)},\frac{Q_0(x_2)}{Q_1(x_2)},1\right)\right)}{\log\left(\max\left(\frac{P_0(x_1)}{P_1(x_1)},\frac{P_0(x_2)}{P_1(x_2)},1\right)\right)}
\end{align}

For the rest of the analysis, with slight abuse of notation, let us denote $p_1(x_1):=p_1$, $p_0(x_1)=p_0$, $Q_0(x_1) = p_0v+(1-p_0)u: = q_0$, and $Q_1(x_1) = p_1v+(1-p_1)u: = q_1$. Further, let $q_k = p_k(v-u)+u:=f(p_k)$, for $k \in \{0,1\}$.

\noindent \textit{Case} $1$, $v>u$ : Without loss of generality assume that $p_0> p_1$. Since $f(p)$ is increasing in $p$ for $v>u$, we have that $q_0 > q_1$. Thus,
\begin{align}
    \eta_{\infty}(W)
    =\frac{\log\left(\frac{q_0}{q_1}\right)}{\log\left(\frac{p_0}{p_1}\right)}
    &=\frac{\log(f(p_0))-\log(f(p_1))}{\log(p_0)-\log(p_1)}\\
    &=\frac{\int_{p_1}^{p_0}d\log(f(p))}{\int_{p_1}^{p_0}d\log(p)}\label{eq:aasd}
\end{align}
Let $R(p)= \dfrac{\frac{d}{dp}\log(f(p))}{\frac{d}{dp}\log(p)}$. Therefore,
\begin{align}
    R(p) = \frac{p(v-u)}{p(v-u)+u}
\end{align}
Note that for $p \in [0,1]$, $R(p)$ is increasing in $p$. Consequently, going back to Eq.~\ref{eq:aasd}, we have that 
\begin{align}
      \eta_{\infty}(W)
    &=  \frac{\int_{p_1}^{p_0}R(p)d\log(p)}{\int_{p_1}^{p_0}d\log(p)} \\
    &\leq \frac{\int_{p_1}^{p_0}R(1)d\log(p)}{\int_{p_1}^{p_0}d\log(p)} \\
    &= R(1) \\
    &= \frac{v-u}{v}
\end{align}
The upper bound is approached as $p_0,p_1 \uparrow 1$ i.e., yielding almost degenerate probability distributions $P_0$ and $P_1$. This completes our proof for the case $v>u$.

\noindent \textit{Case} $2$, $v<u$ : Without loss of generality, assume $p_0>p_1$. 
Let $g(p):=Q(x_2)=v+p(u-v)$. Since $u>v$, $g(p)$ is increasing in $p$, and hence
$Q_0(x_2)=g(p_0)>g(p_1)=Q_1(x_2)$. Therefore,
\begin{align}
D_\infty(Q_0\|Q_1)=\log\frac{Q_0(x_2)}{Q_1(x_2)}.
\end{align}
Now, we have
\begin{align}
    \eta_{\infty}(W)
    =\frac{\log\left(\frac{Q_0(x_2)}{Q_1(x_2)}\right)}{\log\left(\frac{p_0}{p_1}\right)}
    &=\frac{\log(v+p_0(u-v))-\log(v+p_1(u-v))}{\log(p_0)-\log(p_1)}\\
    &=\frac{\int_{p_1}^{p_0}d\log(g(p))}{\int_{p_1}^{p_0}d\log(p)}\label{eq:aasd1}
\end{align}
Let $R(p)= \dfrac{\frac{d}{dp}\log(g(p))}{\frac{d}{dp}\log(p)}$. Therefore,
\begin{align}
    R(p) = \frac{p(u-v)}{p(u-v)+v}
\end{align}
Note that for $p \in [0,1]$, $R(p)$ is increasing in $p$. Consequently, going back to Eq.~\ref{eq:aasd1}, we have that 
\begin{align}
      \eta_{\infty}(W)
    &=  \frac{\int_{p_1}^{p_0}R(p)d\log(p)}{\int_{p_1}^{p_0}d\log(p)} \\
    &\leq \frac{\int_{p_1}^{p_0}R(1)d\log(p)}{\int_{p_1}^{p_0}d\log(p)} \\
    &= R(1) \\
    &= \frac{u-v}{u}
\end{align}
The upper bound is approached as $p_0,p_1 \uparrow 1$ i.e., by almost-degenerate binary input distributions. This completes the proof for the case $v<u$.
\end{proof}

\subsection{Proof of Theorem~\ref{thm:jrdccbd} (SDPI coefficient for $\infty$-order Jeffreys-Rényi divergence)}
\label{section:cc-jrd-proof}
All supremums below are over pairs satisfying $0<D_\infty^{\mathrm J}(P_0,P_1)<\infty$.
Input likelihood ratios and their extremums are evaluated on $\operatorname{supp}(P_0)=\operatorname{supp}(P_1)$. We first assume that $W(y|x)>0$ for every $x,y$ and the more general case is treated at the end.

We know that
\begin{align}
    \eta^{\mathrm J}_{\infty}(W)&:=\sup_{P_0,P_1}\frac{\mathrm{D}^{\mathrm{J}}_{\infty}(Q_0,Q_1)}{\mathrm{D}^{\mathrm{J}}_{\infty}(P_0,P_1)}\\&=\sup_{P_0,P_1}\frac{\mathrm{D}_{\infty}(Q_0\otimes Q_1\|Q_1 \otimes Q_0)}{D^{}_{\infty}(P_0 \otimes P_1\|P_1 \otimes P_0)}\\   &=\sup_{P_0,P_1}\frac{\max_{y,y'}\left(\log\left(\dfrac{Q_0(y)Q_1(y')}{Q_1(y)Q_0(y')}\right)\right)}{D^{}_{\infty}(P_0 \otimes P_1\|P_1 \otimes P_0)}\\ 
&=\max_{y,y'}\sup_{P_0,P_1}\frac{\log\left(\dfrac{Q_0(y)Q_1(y')}{Q_1(y)Q_0(y')}\right)}{D^{}_{\infty}(P_0 \otimes P_1\|P_1 \otimes P_0)}\label{eq:suppp}
\end{align}
If $\eta_\infty^{\mathrm J}(W)=0$, the claim is immediate. Otherwise, it suffices to consider an admissible input pair and outputs $(y,y')$ for which the numerator in~\eqref{eq:suppp} is positive. We construct distributions on a common two-point support
whose ratio for these outputs is no smaller.

For the channel $W$, let $a$ and $b$ be the row vectors which denote the column of $W$ corresponding to $y$ and $y'$, respectively i.e., for every $x \in \mathcal{X}$, we have $a(x)=W(y|x)$ and $b(x)=W(y'|x)$. Therefore, $Q_k(y)=\sum_{x}P_k(x)a(x):= \langle a, P_k\rangle$ and $Q_k(y')=\sum_{x}P_k(x)b(x):= \langle b, P_k\rangle$ for $k \in [2]$. Using this, we have
\begin{align}
    \log\left(\frac{Q_0(y)Q_1(y')}{Q_1(y)Q_0(y')}\right)=\log\left(\frac{\langle a,P_0\rangle \langle b,P_1\rangle}{\langle a,P_1\rangle \langle b,P_0\rangle}\right)
\end{align}
It is interesting to observe here that the multiplicative transformation is isometric with respect to Jeffreys-Rényi divergence of order $\infty$ i.e., for any $t \in \mathbb{R}^{|\mathcal{X}|\times1}$, we have $\mathrm{D}^{\mathrm{J}}_{\infty}(t\circ P_0 ,t \circ P_1)=\mathrm{D}^{\mathrm{J}}_{\infty}(P_0,P_1)$, where $\circ$ denotes the element-wise product. Note that $t \circ P_1$ may not necessarily be a valid distribution.

Now, let $\widetilde{P}_0= b \circ P_0$ and $\widetilde{P}_1= b \circ P_1$. Therefore, we have
\begin{align}
    \log\left(\frac{\langle a,P_0\rangle \langle b,P_1\rangle}{\langle a,P_1\rangle \langle b,P_0\rangle}\right)=\log\left(\frac{\langle r,\widetilde{P}_0\rangle \langle \mathbf{1},\widetilde{P}_1\rangle}{\langle r,\widetilde{P}_1\rangle \langle \mathbf{1},\widetilde{P}_0\rangle}\right)\label{eq:alo}
\end{align}
where $r$ is a row vector defined as $r(x):=\dfrac{a(x)}{b(x)}$ and $\mathbf{1}$ denotes the unit vector. Now, let $s$ denote the likelihood ratio vector i.e., $s(x)=\dfrac{P_0(x)}{P_1(x)}=\dfrac{\widetilde{P}_0(x)}{\widetilde{P}_1(x)}$, for all $x \in \mathcal{X}$. Therefore, Eq.~\eqref{eq:alo} can be further written as
\begin{align}
\log\left(\frac{\langle r,\widetilde{P}_0\rangle \langle \mathbf{1},\widetilde{P}_1\rangle}{\langle r,\widetilde{P}_1\rangle \langle \mathbf{1},\widetilde{P}_0\rangle}\right)&=\log\left(\frac{\langle r,s \circ\widetilde{P}_1\rangle \langle \mathbf{1},\widetilde{P}_1\rangle}{\langle r,\widetilde{P}_1\rangle \langle \mathbf{1},s \circ\widetilde{P}_1\rangle}\right)\\
&=\log\left(\frac{\langle r\circ s, \widetilde{P}_1\rangle \langle \mathbf{1},\widetilde{P}_1\rangle}{\langle r,\widetilde{P}_1\rangle \langle s,\widetilde{P}_1\rangle}\right)\\
&:= F_{y,y'}(r,s,\widetilde{P}_1)
\end{align}
Consequently, Eq.~\eqref{eq:suppp} can now be written as
\begin{align}
    \eta^{\mathrm J}_{\infty}(W) &= \max_{y,y'}\sup_{s \neq \text{constt.},{P}_1}\frac{F_{y,y'}(r,s,\widetilde{P}_1)}{D^{}_{\infty}(P_0 \otimes P_1\|P_1 \otimes P_0)}\\
    &=\max_{y,y'}\sup_{s \neq \text{constt.},{P}_1}\frac{F_{y,y'}(r,s,\widetilde{P}_1)}{\log\left(\dfrac{\max_{x}s(x)}{\min_x s(x)}\right)}\label{eq:wert}
\end{align}
Now, we will show that the function $F_{y,y'}$ is maximized by binary vectors $P_1$ and $s$ supported on at most two atoms.

Also, it is important to note that the function $F_{y,y'}(r,s,\widetilde{P}_1)$ is zero-order homogeneous function of $r$, $s$, and $\widetilde{P}_1$ i.e., for some scalars $\alpha$, $\beta$, and $\gamma$, we have $F_{y,y'}(\alpha r,\beta s,\gamma\widetilde{P}_1)=F_{y,y'}(r,s,\widetilde{P}_1)$. Since, $F_{y,y'}$ is zero-order homogeneous in $\widetilde{P}_1$, we can re-scale $\widetilde{P}_1 \rightarrow \widetilde{p}_1$ (where $\widetilde{p} = c \widetilde{P}_1$) such that $\sum_{x \in \mathcal{X}}\widetilde{p}(x)=1$. Consequently,
\begin{align}
     F_{y,y'}(r,s,\widetilde{P}_1) = F_{y,y'}(r,s,\widetilde{p}) = \log\left(\frac{\langle r\circ s, \widetilde{p}\rangle }{\langle r,\widetilde{p}\rangle \langle s,\widetilde{p}\rangle}\right)
\end{align}

Now, given a $s$, let 
$$m:=\min_{x \in \mathcal{X}}s(x) \hspace{10mm}\text{ and }\hspace{10mm}M:=\max_{x \in \mathcal{X}}s(x)$$

Thus, the denominator in Eq.~\eqref{eq:wert} is a constant and equal to $\log(M/m)$. Given a $P_1$ (and equivalently $\widetilde{p}$), note that $F_{y,y'}(r,s,\widetilde{p})$ is the composition of the logarithmic function with a linear-fractional function in $s$, thus, a quasi-convex function i.e.,
\begin{align}
  F_{y,y'}(r,s,\widetilde{p}) &= \log\left(\frac{\left(\sum_{x \in \mathcal{X}}r(x)s(x)\widetilde{p}(x)\right)}{\left(\sum_{x \in \mathcal{X}}r(x)\widetilde{p}(x)\right)\left(\sum_{x \in \mathcal{X}}s(x)\widetilde{p}(x)\right)}\right)  \\
  &=\log\left(\frac{\mathbb{E}_{\widetilde{p}}[rs]}{\mathbb{E}_{\widetilde{p}}[r]\mathbb{E}_{\widetilde{p}}[s]} \right) \label{eq:tbu}
\end{align}
Thus, a maximizer $s^{*}$ of $F_{y,y'}(r,s,\widetilde{p})$ a maximizer can be chosen at a vertex of the convex and compact set $[m,M]^{|\mathcal{X}|}$ i.e., $s^{*}(x)\in\{m,M\}$, for every $x \in \mathcal{X}$. 

Therefore, w.l.o.g we can assume that $s$ is binary-valued such that every coordinate takes values from the set $\{m,M\}$. Let $\mathcal{A}:=\{x \in \mathcal{X}: s(x)=M\}$ and $\mathcal{A}^c:=\{x \in \mathcal{X}: s(x)=m\}$. Let 
$$\alpha :=\widetilde{p}(A)= \sum_{x \in \mathcal{A}}\widetilde{p}(x)$$
Since the objective under consideration is positive and does
not decrease after the endpoint reduction, $0<\alpha<1$:
if $\alpha\in\{0,1\}$, the objective would be zero. Furthermore, let the conditional expectations be denoted as
$$\mu_1:=\mathbb{E}_{\widetilde{p}}[r(X)|X \in \mathcal{A}]\hspace{10mm}\text{ and }\hspace{10mm}\mu_2:=\mathbb{E}_{\widetilde{p}}[r(X)|X \in \mathcal{A}^c]$$
where 
$\mu_1,\mu_2 \geq0$ since $r(x)\geq 0$, for all $x \in \mathcal{X}$. Thus,
\begin{align}
 F_{y,y'}(r,s,\widetilde{p})   &=\log\left(\frac{\alpha M\mu_1 +(1-\alpha)m\mu_2}{(\alpha \mu1+(1-\alpha)\mu_2)(\alpha M+(1-\alpha)m)} \right)
\end{align}
For $M>m$, the function $ F_{y,y'}(r,s,\widetilde{p})$ is monotonically increasing in $\mu_1$ and monotonically decreasing in $\mu_2$. This monotonicity allows us to force extreme (two-point) solution.

Since, $\mu_1 \in [\min_{x \in \mathcal{A}}r(x),\max_{x \in \mathcal{A}}r(x)]$ and $\mu_2 \in [\min_{x \in \mathcal{A}^c}r(x),\max_{x \in \mathcal{A}^c}r(x)]$, $ F_{y,y'}(r,s,\widetilde{p})$ can be upper bounded by choosing $\mu^{*}_1 =\max_{x \in \mathcal{A}}r(x)$ and $\mu^{*}_2 =\min_{x \in \mathcal{A}^c}r(x)$. Now, we will show that there exists a binary $\widetilde{p}^{*}$ such that $ F_{y,y'}(r,s,\widetilde{p}^{*}) \geq  F_{y,y'}(r,s,\widetilde{p})$.

Let $x_1 = \arg\max_{x \in \mathcal{A}}r(x)$ ad  $x_2 = \arg\min_{x \in \mathcal{A}^c}r(x)$. Define $\widetilde{p}^{*}$ supported on the set $\{x_1,x_2\}$ such that $\widetilde{p}^{*}(x_1)=\alpha$ and $\widetilde{p}^{*}(x_2)=1-\alpha$. We have that $r(x_1)=\mu^{*}_1$ and $r(x_2)=\mu^{*}_2$. Consequently,  $F_{y,y'}(r,s,\widetilde{p}^{*})$ is at least $F_{y,y'}(r,s,\widetilde{p})$ for any other $\widetilde{p}$ with the same $\alpha \in [0,1]$.
Thus, we have shown that for a binary $s$, there exists an optimizer $\widetilde{p}^{*}$ (and therefore $\widetilde{P}^{*}_1$) supported on at most two atoms. 
Now, define
\[
\widehat P_1(x)=\frac{\widetilde p^\star(x)/b(x)}
     {\sum_z\widetilde p^\star(z)/b(z)},
\qquad \text{and} \qquad
\widehat P_0(x)=\frac{s(x)\widehat P_1(x)}{\sum_zs(z)\widehat P_1(z)}.
\]
Both are probability distributions supported on the set $\{x_1,x_2\}$. Their likelihood ratio is a positive scalar multiple of $s$, hence
\[
D_\infty^{\mathrm J}(\widehat P_0,\widehat P_1)
=\log(M/m).
\]
By homogeneity, we get that their output cross-ratio equals $F_{y,y'}(r,s,\widetilde p^\star)$, which is atleast as large as the original objective. Thus, the ratio in~\eqref{eq:suppp} does not decrease, proving the binary-support reduction for strictly positive channels.

For a general channel, we first discard identically zero output columns. If a remaining column contains a zero, we choose $x_1,x_2,y$ such that $W(y|x_1)>0=W(y|x_2)$. For $0<t<1/2$, let us define distributions on $\{x_1,x_2\}$ by
\[
P_{0,t}=(2t,1-2t),\qquad \text{and} \qquad P_{1,t}=(t,1-t),
\]
and let $Q_{\nu,t}=P_{\nu,t}W$. Then
\[
D_\infty^{\mathrm J}(P_{0,t},P_{1,t})
=\log\frac{2(1-t)}{1-2t},
\]
and $\frac{Q_{0,t}(y)}{Q_{1,t}(y)}=2$. Consequently, nonn-negativity of the Rényi divergences of order infinity and data processing inequality gives us
\[
\frac{\log 2}{\log\frac{2(1-t)}{1-2t}}
\leq
\frac{D_\infty^{\mathrm J}(Q_{0,t},Q_{1,t})}
     {D_\infty^{\mathrm J}(P_{0,t},P_{1,t})}
\leq 1.
\]
Letting $t\to 0$ shows that $\eta_\infty^{\mathrm J}(W)=1$ is approached by binary
input pairs. Otherwise, the remaining channel is strictly positive, and the previous argument applies. This completes our proof.

\subsection{Proof of Theorem~\ref{thm:jrdccbd-value} (Closed-form expression of SDPI coefficient)}\label{section:cc-jrd-value-proof}
Let $P_0$, $P_1$, $W$, $Q_0$, $Q_1$ be as in the proof of Theorem~\ref{thm:rdccbd-value} (Appendix~\ref{app:rd_sdpi}). Then, we have that
\begin{align}
         \eta^{\mathrm J}_{\infty}(W) &= \frac{D_{\infty}(Q_0\|Q_1) + D_{\infty}(Q_1\|Q_0)}{D_{\infty}(P_0\|P_1)+D_{\infty}(P_1\|P_0)}\\
    &=\frac{\log\left(\max\left(\frac{Q_0(x_1)}{Q_1(x_1)},\frac{Q_0(x_2)}{Q_1(x_2)},1\right)\right)+\log\left(\max\left(\frac{Q_1(x_1)}{Q_0(x_1)},\frac{Q_1(x_2)}{Q_0(x_2)},1\right)\right)}{\log\left(\max\left(\frac{P_0(x_1)}{P_1(x_1)},\frac{P_0(x_2)}{P_1(x_2)},1\right)\right)+\log\left(\max\left(\frac{P_1(x_1)}{P_0(x_1)},\frac{P_1(x_2)}{P_0(x_2)},1\right)\right)}
\end{align}
For the rest of the analysis, let $p_0:=P_0(x_1)$ and $p_1:=P_1(x_1)$, and define
\[
f(p):=u+p(v-u).
\]
Then $Q_k(x_1)=f(p_k)$ and $Q_k(x_2)=f(1-p_k)$ for $k\in\{0,1\}$. Similar to the previous case, we split the analysis in two cases, we analyze the first case $(v>u)$ in detail.

\noindent \textit{Case} $1$, $v>u$ : Without loss of generality assume that $p_0> p_1$. Since $f$ is increasing and $p_0>p_1$, we have $Q_0(x_1)=f(p_0)>f(p_1)=Q_1(x_1)$
and $Q_1(x_2)=f(1-p_1)>f(1-p_0)=Q_0(x_2).
$
 Thus,

\begin{align}
    \eta^{\mathrm J}_{\infty}(W)
    &=\frac{\log\!\left(\frac{Q_0(x_1)}{Q_1(x_1)}\right)+\log\!\left(\frac{Q_1(x_2)}{Q_0(x_2)}\right)}
{\log\!\left(\frac{p_0}{p_1}\right)+\log\!\left(\frac{1-p_1}{1-p_0}\right)}\\
     &=\frac{\log\left(\frac{f(p_0)}{f(1-p_0)}\right)-\log\left(\frac{f(p_1)}{f(1-p_1)}\right)}{\log\left(\frac{p_0}{1-p_0}\right)-\log\left(\frac{p_1}{1-p_1}\right)}
\end{align}
Define, for any $p \in [0,1]$, the functions $g(p)=\log\left(\frac{p}{1-p}\right)$ and $h(p)=\log\left(\frac{f(p)}{f(1-p)}\right)$. Then, 
\begin{align}
 \eta^{\mathrm J}_{\infty}(W)
    &=\frac{h(p_0)-h(p_1)}{g(p_0)-g(p_1)} \\
    &=\frac{\int_{p_1}^{p_0}h'(p)dp}{\int_{p_1}^{p_0}g'(p)dp} \label{eq:tcb}
\end{align}
Now, we will prove an upper bound on the ratio of derivatives of $h(p)$ and $g(p)$ and then substitute it back in Eq.~\ref{eq:tcb} to prove our result. From the definition of $h(p)$ and $g(p)$, we have that
\begin{align}
    \frac{h'(p)}{g'(p)} &= \frac{\frac{v^2-u^2}{f(p)f(1-p)}}{\frac{1}{p(1-p)}}\\
    &=(v^2-u^2)\frac{p(1-p)}{f(p)f(1-p)} := (v^2-u^2)R(p)
\end{align}
Note that the function $R(p)$ is continuous in $p$, $R(0)=R(1)=0$, and symmetric around $p=1/2$. Moreover, it is monotonically increasing in $p \in [0,1/2]$.
Therefore,
\begin{align}
     \frac{h'(p)}{g'(p)} &\leq (v^2-u^2)R\left(\frac{1}{2}\right)\\
     &=(v^2-u^2)\frac{1}{(v+u)^2} = \frac{v-u}{v+u}
\end{align}
Going back to Eq.~\ref{eq:tcb}, we have that
\begin{align}
 \eta^{\mathrm J}_{\infty}(W)
    &\leq\frac{\int_{p_1}^{p_0}(\frac{v-u}{v+u})g'(p)dp}{\int_{p_1}^{p_0}g'(p)dp} = \frac{v-u}{v+u} \label{eq:awe}
\end{align}

The upper bound is approached as $p_0,p_1\to \tfrac12$ with $p_0\neq p_1$, i.e., by binary input distributions that become arbitrarily close to uniform. Similar to the above analysis, it can be shown for $u>v$ that $\eta^{\mathrm J}_{\infty}(W) = \dfrac{u-v}{u+v}$. This completes our proof for Theorem~\ref{thm:jrdccbd-value}.
\section{Private Change-Point Detection}
\subsection{Proof of Lemma~\ref{cor:ldpp} (Upper bound on induced Jeffreys-Rényi divergence via $\varepsilon$-LDP)}\label{section:jrdub-proof}
\begin{proof}
First, we state and prove the following auxiliary result below which will be used in our proof of our main Lemma~\ref{cor:ldpp}.
\begin{lemma}\label{thm:jrdub}
   Let $(Q_0,Q_1)$ be the output distributions corresponding to the input distributions $(P_0,P_1)$, through an $\varepsilon$-LDP mechanism. Then, the Jeffreys-Rényi divergence of order infinity between $Q_0$ and $Q_1$ satisfies  $\mathrm{D}^{\mathrm{J}}_{\infty}(Q_0,Q_1) \leq 2\varepsilon$, for any $\varepsilon>0$, irrespective of $P_0$ and $P_1$. 
\end{lemma}
\begin{proof}
Recall that,
\begin{align}
    S&= \max_{x \in \mathcal{X}}\log\frac{P_1(x)}{P_0(x)} - \min_{x \in \mathcal{X}}\log\frac{P_1(x)}{P_0(x)}\\
    &= \mathrm{D}_{\infty}(P_1\|P_0) + \mathrm{D}_{\infty}(P_0\|P_1)\\ 
    &= \mathrm{D}^{\mathrm{J}}_{\infty}(P_0,P_1).
\end{align}

For our proof, we establish upper and lower bounds on $\log\frac{Q_1(y)}{Q_0(y)}$ for any $\varepsilon$-LDP mechanism $W_{Y|X}$. 
In the following, we use that for all $x,x' \in \mathcal{X}, y \in \mathcal{Y}$,
\begin{align}
    e^{-\varepsilon} \leq \frac{W_{Y|X}(y|x)}{W_{Y|X}(y|x')} \leq e^{\varepsilon}.
\end{align}

\textit{Upper bound:} First, we note that for all $x \in \mathcal{X}, y \in \mathcal{Y}$,
\begin{align}
    \frac{W_{Y|X}(y|x)}{\min_{x' \in \mathcal{X}} W_{Y|X}(y|x')} &\leq e^{\varepsilon} \\
    \implies W_{Y|X}(y|x) &\leq e^{\varepsilon} \min_{x' \in \mathcal{X}} W_{Y|X}(y|x').
\end{align}
Then, we have
\begin{align}
    \log\frac{Q_1(y)}{Q_0(y)} & = \log\frac{\sum_{x \in \mathcal{X}} P_1(x) W_{Y|X}(y|x)} {\sum_{x \in \mathcal{X}} P_0(x) W_{Y|X}(y|x)} \\
    & \leq \log\frac{\sum_{x \in \mathcal{X}} P_1(x) W_{Y|X}(y|x)}{\min_{x' \in \mathcal{X}} W_{Y|X}(y|x')} \\
    & \leq \log\frac{\left(\sum_{x \in \mathcal{X}} e^{\varepsilon} P_1(x)\right) \min_{x' \in \mathcal{X}} W_{Y|X}(y|x') }{\min_{x' \in \mathcal{X}} W_{Y|X}(y|x')} \\
    & = \log e^{\varepsilon} = \varepsilon.
\end{align}
This shows that
\begin{align}
    \max_{y \in \mathcal{Y}} \log\frac{Q_1(y)}{Q_0(y)} \leq \varepsilon.
\end{align}

\textit{Lower bound:} Similar to the upper bound, we first see that  for all $x \in \mathcal{X}, y \in \mathcal{Y}$,
\begin{align}
    \frac{W_{Y|X}(y|x)}{\max_{x' \in \mathcal{X}} W_{Y|X}(y|x')} &\geq e^{-\varepsilon} \\
     W_{Y|X}(y|x) &\geq e^{-\varepsilon} \max_{x' \in \mathcal{X}} W_{Y|X}(y|x').
\end{align}
Then, we have
\begin{align}
    \log\frac{Q_1(y)}{Q_0(y)} & = \log\frac{\sum_{x \in \mathcal{X}} P_1(x) W_{Y|X}(y|x)} {\sum_{x \in \mathcal{X}} P_0(x) W_{Y|X}(y|x)} \\
    & \geq \log\frac{\sum_{x \in \mathcal{X}} P_1(x) W_{Y|X}(y|x)}{\max_{x' \in \mathcal{X}} W_{Y|X}(y|x')} \\
    & \geq \log\frac{\sum_{x \in \mathcal{X}} e^{-\varepsilon} P_1(x) \max_{x' \in \mathcal{X}} W_{Y|X}(y|x') }{\max_{x' \in \mathcal{X}} W_{Y|X}(y|x')} \\
    & = \log e^{-\varepsilon} = -\varepsilon.
\end{align}
This results in
\begin{align}
    \min_{y \in \mathcal{Y}} \log\frac{Q_1(y)}{Q_0(y)} \geq -\varepsilon.
\end{align}

Therefore,
\begin{align}
    \mathrm{D}_{\infty}(Q_1\|Q_0) + \mathrm{D}_{\infty}(Q_0\|Q_1) &= \max_{y \in \mathcal{Y}}\log\frac{Q_1(y)}{Q_0(y)} - \min_{y \in \mathcal{Y}}\log\frac{Q_1(y)}{Q_0(y)} \\
    &\leq 2\varepsilon.
\end{align}
\end{proof}
From our results on SDPI coefficients in Section~\ref{subsec:sdpi}, we also know that 
\begin{align}
\mathrm{D}^{\mathrm{J}}_{\infty}(Q_0,Q_1) \leq \eta^{\mathrm J}_{\infty}(W^{\varepsilon})\mathrm{D}^{\mathrm{J}}_{\infty}(P_0,P_1).
\end{align}
Thus, combining the above result due to SDPI coefficients with Lemma~\ref{thm:jrdub}, we obtain that
\begin{align}
\mathrm{D}^{\mathrm{J}}_{\infty}(Q_0,Q_1) \leq \min\Big( 2\varepsilon \hspace{1mm},\eta^{\mathrm J}_{\infty}(W^{\varepsilon})\mathrm{D}^{\mathrm{J}}_{\infty}(P_0,P_1)\Big).
\end{align}
This completes our proof of Lemma~\ref{cor:ldpp}.
\end{proof}

\subsection{Proof of Theorem~\ref{thm:private-minimax-lb} (Finite-sample minimax lower bound under $\varepsilon$-LDP)}\label{app:private_minimax_lb}
\begin{proof}
Since the mechanisms may differ across observations, the privatized observations are independent, but, unlike in Theorem~\ref{thm:minimax_lb}, not necessarily identically distributed, even within the pre-change or the post-change segment. We fix an arbitrary collection $\boldsymbol{W}^{\varepsilon}\in\mathcal{W}_\varepsilon^n$ and proceed in three steps. In Step~1, as in Theorem~\ref{thm:minimax_lb}, we reduce the problem to testing among $M_\alpha$ candidate change points spaced
$2\alpha+1$ apart, and lower-bound the risk by the overlaps $1-\mathrm{d}_{\mathrm{TV}}$ between adjacent candidates, which here need not be equal. In Step~2, we lower-bound each overlap by exponential tilting on the block of $2\alpha+1$ coordinates on which the two adjacent hypotheses differ; the resulting bound depends on $\boldsymbol{W}^{\varepsilon}$ only through per-coordinate Chernoff information and range of log-likelihood ratios. In Step~3, we bound these per-coordinate quantities by $\bar I_\varepsilon$ and $2\overline S_\varepsilon$, respectively, uniformly over all $\varepsilon$-LDP mechanisms. Since the final bound does not depend on $\boldsymbol{W}^{\varepsilon}$, it also holds for the infimum over $\boldsymbol{W}^{\varepsilon}$.


Let us begin with fixing an arbitrary collection of $\varepsilon$-LDP mechanisms $\boldsymbol{W}^{\varepsilon} =(W^{\varepsilon}_1,\ldots,W^{\varepsilon}_n)\in\mathcal W_\varepsilon^n$, applied independently across all observations. For each coordinate $i \leq n$, let
$Q_{0,i}(y)=\sum_{x}P_0(x)W^{\varepsilon}_i(y|x)$
and 
$Q_{1,i}(y)=\sum_{x}P_1(x)W^{\varepsilon}_i(y|x)$, for every $y \in \mathcal{Y}$. Without loss of generality, we assume both PMFs are
positive on the same support, since each $W^{\varepsilon}_i$ satisfies $\varepsilon$-LDP.

Under a true change-point $k^\star$, the privatized observations have joint PMF
\begin{align}
\label{eq:private-joint-pmf}
\mathbb P_{k^\star}^{\boldsymbol{W}^{\varepsilon}}(y^n)
=\prod_{i<k^\star}Q_{0,i}(y_i) \prod_{i\geq k^\star}Q_{1,i}(y_i).
\end{align}
For this fixed collection of $\varepsilon$-LDP mechanisms, we define
\begin{align}
E_{\mathcal K}^{\boldsymbol{W}^{\varepsilon}}(\alpha)
\triangleq
\inf_{\widehat k}
\sup_{k^\star\in\mathcal K}
\mathbb P_{k^\star}^{\boldsymbol{W}^{\varepsilon}}
\left(|\widehat k-k^\star|>\alpha\right).
\end{align}
Then, by definition, we have that
\begin{equation}
\label{eq:private-mechanism-infimum}
E_{\mathcal K}^{\varepsilon}(\alpha)
=
\inf_{\boldsymbol{W}^{\varepsilon}\in\mathcal W_\varepsilon^n}
E_{\mathcal K}^{\boldsymbol{W}^{\varepsilon}}(\alpha).
\end{equation}

Let $m=2\alpha+1$. As in the proof of Theorem~\ref{thm:minimax_lb}, consider the potential change points
\begin{align*}
\label{eq:private-packing}
k_j=2+jm,
\qquad \text{where } 
j\in\{0,\ldots,M-1\}.
\end{align*}
Since $k_{M-1}\leq n$, we must have
\begin{align}
M \leq 1+\left\lfloor\frac{n-2}{m}\right\rfloor = M_\alpha.
\end{align}
Let $\mathcal K_M=\{k_0,\ldots,k_{M-1}\}$. On restricting the
supremum to this set, we get
$
E_{\mathcal K}^{\boldsymbol{W}^{\varepsilon}}(\alpha) \geq
E_{\mathcal K_M}^{\boldsymbol{W}^{\varepsilon}}(\alpha)
$. Consider any possibly randomized estimator on the privatized symbols $Y^n$, $\widehat k=\widehat k (Y^n)\in\mathcal K\cup\{\infty\}$, and a M-ary test
\begin{equation}
\label{eq:private-test}
\psi(\widehat k)
=
\begin{cases}
k_j,
&
\text{if }\widehat k\in\mathcal K
\text{ and }|\widehat k-k_j|\le\alpha,
\\
\infty,
&
\text{otherwise}.
\end{cases}
\end{equation}
The success regions for all $k_j$ are disjoint because consecutive
candidates are separated by $2\alpha+1$. Consequently, under hypothesis $k_j$,
$\{\psi(\widehat k)\ne k_j\}
= \{|\widehat k-k_j|>\alpha\}$,
where declaring $\infty$ is counted as an error event. The sum of the error probabilities therefore satisfies
\begin{align}
\label{eq:private-testing-to-estimation}
e_{\boldsymbol{W}^{\varepsilon}}(\psi(\widehat k))
\triangleq
\sum_{j=0}^{M-1}
\mathbb P_{k_j}^{\boldsymbol{W}^{\varepsilon}}
\left(\psi(\widehat k)\ne k_j\right)
&=
\sum_{j=0}^{M-1}
\mathbb P_{k_j}^{\boldsymbol{W}^{\varepsilon}}
\left(|\widehat k-k_j|>\alpha\right)
\notag\\
&\leq
M\sup_{k^\star\in\mathcal K_M}
\mathbb P_{k^\star}^{\boldsymbol{W}^{\varepsilon}}
\left(|\widehat k-k^\star|>\alpha\right).
\end{align}
We will next lower bound this sum of testing errors. To handle
estimator randomization explicitly, define
\[
d_j(y^n)
\triangleq
\mathbb P_U
\left(\psi(\widehat k(y^n,U))=k_j\right).
\]
where $U$ denotes estimator's randomness independent of $Y^n$ and $k^{\star}$. These functions $d_j(\cdot)$ do not depend on the true hypothesis and satisfy
\[
d_j(y^n)\ge0,
\qquad
\sum_{j=0}^{M-1}d_j(y^n)\le1.
\]
The inequality allows the test to output $\infty$.
Thus,
\begin{align}
\label{eq:private-testing-error-max}
e_{\boldsymbol{W}^{\varepsilon}}(\psi(\widehat k))
&=
M-
\sum_{y^n \in \mathcal{Y}^n}\sum_{j=0}^{M-1}
\mathbb P_{k_j}^{\boldsymbol{W}^{\varepsilon}}(y^n)d_j(y^n)
\notag\\
&\geq
M-
\sum_{y^n \in \mathcal{Y}^n}
\sup_{k_j \in \mathcal{K}_M}
\mathbb P_{k_j}^{\boldsymbol{W}^{\varepsilon}}(y^n).
\end{align}

For every $k_j$, we have
\begin{align}
\mathbb P_{k_j}^{\boldsymbol{W}^{\varepsilon}}(y^n)
&=
\mathbb P_{k_0}^{\boldsymbol{W}^{\varepsilon}}(y^n)
+
\sum_{i=1}^{j}
\left(
\mathbb P_{k_i}^{\boldsymbol{W}^{\varepsilon}}(y^n)
-
\mathbb P_{k_{i-1}}^{\boldsymbol{W}^{\varepsilon}}(y^n)
\right)
\leq
\mathbb P_{k_0}^{\boldsymbol{W}^{\varepsilon}}(y^n)
+
\sum_{i=1}^{M-1}
\left[
\mathbb P_{k_i}^{\boldsymbol{W}^{\varepsilon}}(y^n)
-
\mathbb P_{k_{i-1}}^{\boldsymbol{W}^{\varepsilon}}(y^n)
\right]^+.
\end{align}
Taking the supremum over $k_j$ and summing over $y^n$, we obtain
\begin{align}
\label{eq:private-telescoping-tv}
\sum_{y^n \in \mathcal{Y}^n}\sup_{k_j}
\mathbb P_{k_j}^{\boldsymbol{W}^{\varepsilon}}(y^n)
&\leq
1+
\sum_{i=1}^{M-1}
\sum_{y^n}
\left[
\mathbb P_{k_i}^{\boldsymbol{W}^{\varepsilon}}(y^n)
-
\mathbb P_{k_{i-1}}^{\boldsymbol{W}^{\varepsilon}}(y^n)
\right]^+
=
1+
\sum_{i=1}^{M-1}
\mathrm{d}_{\mathrm{TV}}
\left(
\mathbb P_{k_{i-1}}^{\boldsymbol{W}^{\varepsilon}},
\mathbb P_{k_i}^{\boldsymbol{W}^{\varepsilon}}
\right).
\end{align}
Combining~\eqref{eq:private-testing-error-max} and
\eqref{eq:private-telescoping-tv} yields
\begin{equation}
\label{eq:private-testing-overlap-sum}
e_{\boldsymbol{W}^{\varepsilon}}(\psi(\widehat k))
\ge
\sum_{i=1}^{M-1}
\left[
1-
\mathrm{d}_{\mathrm{TV}}
\left(
\mathbb P_{k_{i-1}}^{\boldsymbol{W}^{\varepsilon}},
\mathbb P_{k_i}^{\boldsymbol{W}^{\varepsilon}}
\right)
\right].
\end{equation}
Combining~\eqref{eq:private-testing-to-estimation} and
\eqref{eq:private-testing-overlap-sum}, and taking the
infimum over estimators, gives
\begin{equation}
\label{eq:private-risk-overlap-sum}
E_{\mathcal K_M}^{\boldsymbol{W}^{\varepsilon}}(\alpha)
\ge
\frac1M
\sum_{i=1}^{M-1}
\left[
1-
\mathrm{d}_{\mathrm{TV}}
\left(
\mathbb P_{k_{i-1}}^{\boldsymbol{W}^{\varepsilon}},
\mathbb P_{k_i}^{\boldsymbol{W}^{\varepsilon}}
\right)
\right].
\end{equation}
Since, we have different mechanisms across all coordinates $i$, the adjacent TV distances need not be equal, unlike in Theorem~\ref{thm:minimax_lb}. Therefore, we bound each adjacent overlap separately.

Let us fix an adjacent pair $(k_{i-1},k_i)$, and let $B_i=\{k_{i-1},\ldots,k_i-1\}$,
such that $|B_i|=m$. The two hypotheses differ only on this block. Further, let's define
\begin{align}
C_i(y_{B_i^c})
=
\prod_{j<k_{i-1}}Q_{0,j}(y_j)
\prod_{j\geq k_i}Q_{1,j}(y_j),
\end{align}
and the two block PMFs
\[
q_{h,i}(y_{B_i})
=
\prod_{j\in B_i}Q_{h,j}(y_j),
\qquad h\in\{0,1\}.
\]
Then, we have that
\[
\mathbb P_{k_{i-1}}^{\boldsymbol{W}^{\varepsilon}}(y^n)
=
C_i(y_{B_i^c})q_{1,i}(y_{B_i}),
\qquad \text{and}\qquad
\mathbb P_{k_i}^{\boldsymbol{W}^{\varepsilon}}(y^n)
=
C_i(y_{B_i^c})q_{0,i}(y_{B_i}).
\]
Using $\min(ca,cb)=c\min(a,b)$ for $c\ge0$, and on summing together
the common coordinates outside $B_i$, we obtain
\begin{align}
\label{eq:private-block-overlap}
1-
\mathrm{d}_{\mathrm{TV}}
\left(
\mathbb P_{k_{i-1}}^{\boldsymbol{W}^{\varepsilon}},
\mathbb P_{k_i}^{\boldsymbol{W}^{\varepsilon}}
\right)
&=
\sum_{y_{B_i}}
\min\left(
q_{0,i}(y_{B_i}),
q_{1,i}(y_{B_i})
\right)
\triangleq A_i^{\boldsymbol{W}^{\varepsilon}}.
\end{align}
For each coordinate, let us define
$Z_i(\lambda)
=
\sum_{y_i}
Q_{0,i}(y_i)^\lambda Q_{1,i}(y_i)^{1-\lambda}$. For the entire block, independence gives us
\begin{align}
\label{eq:private-block-normalizer}
Z_{B_i}(\lambda)
&\triangleq
\sum_{y_{B_i}}
q_{0,i}(y_{B_i})^\lambda
q_{1,i}(y_{B_i})^{1-\lambda}
=\prod_{j\in B_i}Z_j(\lambda).
\end{align}
Let
$
\lambda_i^\star
\in
\arg\min_{\lambda\in[0,1]}Z_{B_i}(\lambda),
$
and 
$
I_i
\triangleq
-\log Z_{B_i}(\lambda_i^\star)
=
\mathrm{I}_{\,\mathrm{ch}}(q_{0,i},q_{1,i})
$. If $q_{0,i}=q_{1,i}$, then $A_i^{\boldsymbol{W}^{\varepsilon}}=1$, and the
desired lower bound is immediate. Otherwise, the common positive support ensures that an optimizing $\lambda_i^\star$ lies in $(0,1)$.
Now, let's define
\[
\ell_j(y_j)
\triangleq
\log\frac{Q_{0,j}(y_j)}{Q_{1,j}(y_j)},
\]
and
\[
S_i(y_{B_i})
\triangleq
\sum_{j\in B_i}\ell_j(y_j)
=
\log\frac{q_{0,i}(y_{B_i})}{q_{1,i}(y_{B_i})}.
\]
The tilted block PMF now is
\begin{equation}
\label{eq:private-block-tilt}
R_i^\star(y_{B_i})
\triangleq
\frac{
q_{0,i}(y_{B_i})^{\lambda_i^\star}
q_{1,i}(y_{B_i})^{1-\lambda_i^\star}
}{
Z_{B_i}(\lambda_i^\star)
}.
\end{equation}
Now, we proceed following the four exponential tilting steps, similar to the proof in Theorem~\ref{thm:minimax_lb}.

\medskip
\noindent
\emph{(a). $S_i$ is centered around zero under $R_i^\star$:}
Differentiating the block normalizing function gives us
\[
\frac{d}{d\lambda}\log Z_{B_i}(\lambda)
= \frac{
\sum_{y_{B_i}}
q_{0,i}(y_{B_i})^\lambda
q_{1,i}(y_{B_i})^{1-\lambda}
S_i(y_{B_i})
}{
Z_{B_i}(\lambda)
}.
\]
At the minimizer, we have
\begin{equation}
\label{eq:private-centered-block-score}
\mathbb E_{R_i^\star}[S_i]
=\frac{d}{d\lambda}\log Z_{B_i}(\lambda)
\big|_{\lambda=\lambda_i^\star}
=0.
\end{equation}
Thus, we have that $\Sigma_i^2
\triangleq \mathbb E_{R_i^\star}[S_i^2] =\operatorname{Var}_{R_i^\star}(S_i)$. It is important to note here that the block sum $S_i$ has mean zero. The
individual summands need not be centered around zero.

\medskip
\noindent
\emph{(b). Changing the measure of the block:}
From~\eqref{eq:private-block-tilt}, we have
\begin{align}
\label{eq:private-tilting-identities}
q_{1,i}(y_{B_i})
&=
e^{-I_i}
R_i^\star(y_{B_i})
e^{-\lambda_i^\star S_i(y_{B_i})},
\notag\\
q_{0,i}(y_{B_i})
&=
e^{-I_i}
R_i^\star(y_{B_i})
e^{(1-\lambda_i^\star)S_i(y_{B_i})}.
\end{align}

\medskip
\noindent
\emph{(c). Computing the minimum.}
Let $S_i^+=\max(S_i,0)$, and $S_i^-=\max(-S_i,0)$. Considering separately the cases $S_i\geq0$ and $S_i<0$, then~\eqref{eq:private-tilting-identities} gives
\[
\min(q_{0,i},q_{1,i}) =e^{-I_i}R_i^\star
\exp\left(
-\lambda_i^\star S_i^+
-(1-\lambda_i^\star)S_i^-
\right).
\]
Now, summing over the block $B_i$, and substituting into~\eqref{eq:private-block-overlap} gives
\begin{equation}
\label{eq:private-tilted-overlap}
A_i^{\boldsymbol{W}^{\varepsilon}}
=
e^{-I_i}
\mathbb E_{R_i^\star}
\left[
\exp\left(
-\lambda_i^\star S_i^+
-(1-\lambda_i^\star)S_i^-
\right)
\right].
\end{equation}

\medskip
\noindent
\emph{(d). Bounding the expectation:}
Since $S_i=S_i^+-S_i^-$ has mean zero and $|S_i|=S_i^++S_i^-$, we have $\mathbb E_{R_i^\star}[S_i^+]=
\mathbb E_{R_i^\star}[S_i^-]=
\frac12\mathbb E_{R_i^\star}[|S_i|]
$. Therefore,
$\mathbb E_{R_i^\star}
\left[
\lambda_i^\star S_i^+
+(1-\lambda_i^\star)S_i^-
\right]
=\frac12\mathbb E_{R_i^\star}[|S_i|].
$. Finally, applying Jensen's inequality and then Cauchy--Schwarz to
\eqref{eq:private-tilted-overlap} yields
\begin{align}
\label{eq:private-overlap-before-uniform-bounds}
A_i^{\boldsymbol{W}^{\varepsilon}}
&\geq
\exp\left(
-I_i-\frac12\mathbb E_{R_i^\star}[|S_i|]
\right)
\notag\\
&\geq
\exp\left(-I_i-\frac{\Sigma_i}{2}\right).
\end{align}

Suppose that we can bound $\mathrm{D}_\infty^{\mathrm{J}}(Q_{0,j},Q_{1,j})$ and $\mathrm{I}_{\,\mathrm{ch}}(Q_{0,j},Q_{1,j})$ uniformly over every
coordinate and every $\varepsilon$-LDP mechanism:
\begin{equation}
\label{eq:private-uniform-bounds}
\mathrm{D}_\infty^{\mathrm{J}}(Q_{0,j},Q_{1,j})
\leq 2 \,\overline S_\varepsilon,
\qquad
\mathrm{I}_{\,\mathrm{ch}}(Q_{0,j},Q_{1,j})
\leq \overline I_\varepsilon.
\end{equation}
We will provide the expressions for $\overline I_{\varepsilon}$ and $\overline S_\varepsilon$ below, later in the proof. First, note that by
\eqref{eq:private-block-normalizer} and
\eqref{eq:private-uniform-bounds}, we have
\begin{align}
\label{eq:private-block-information-bound}
I_i
&=
\max_{\lambda\in[0,1]}
\sum_{j\in B_i}\left[-\log Z_j(\lambda)\right]
\notag\\
&\leq
\sum_{j\in B_i}
\max_{\lambda\in[0,1]}
\left[-\log Z_j(\lambda)\right]
\notag\\
&=
\sum_{j\in B_i}
\mathrm{I}_{\,\mathrm{ch}}(Q_{0,j},Q_{1,j})
\notag\\
&\le
m\overline I_\varepsilon.
\end{align}
We need the first inequality because the coordinate-wise
minimizing parameters need not coincide. Further, the tilted block PMF can be written as
\begin{align}
R_i^\star(y_{B_i})
=
\prod_{j\in B_i}R_{j,\lambda_i^\star}(y_j),
\end{align}
where
\[
R_{j,\lambda}(y_j)
=
\frac{
Q_{0,j}(y_j)^\lambda Q_{1,j}(y_j)^{1-\lambda}
}{
Z_j(\lambda)
}.
\]
Thus the coordinates remain independent under
$R_i^\star$, and
\begin{equation}
\label{eq:private-variance-additivity}
\Sigma_i^2
=\sum_{j\in B_i}
\operatorname{Var}_{R_{j,\lambda_i^\star}}(\ell_j).
\end{equation}
Let $a_j=\inf_{y_j}\ell_j(y_j)$, and $b_j=\sup_{y_j}\ell_j(y_j)$. Then, from~\eqref{eq:private-uniform-bounds},
\[
b_j-a_j=\mathrm{D}_\infty^{\mathrm{J}}(Q_{0,j},Q_{1,j})
\leq 2\overline S_\varepsilon.
\]
Using the midpoint $c_j=(a_j+b_j)/2$, we obtain
\begin{align}
\operatorname{Var}_{R_{j,\lambda_i^\star}}(\ell_j)
\leq
\mathbb E_{R_{j,\lambda_i^\star}}
\left[(\ell_j-c_j)^2\right]
\leq\frac{(b_j-a_j)^2}{4}
&\leq
{\overline S_\varepsilon^{\,2}}.
\end{align}
This bound holds for every exponential tilting and does not require the
individual tilted mean to be zero. Consequently,
\begin{equation}
\label{eq:private-block-variance-bound}
\Sigma_i^2
\leq {m\overline S_\varepsilon^{\,2}},
\qquad\text{and therefore, } \qquad \Sigma_i
\leq
{\overline S_\varepsilon}{}\sqrt m.
\end{equation}
On substituting~\eqref{eq:private-block-information-bound}, and~\eqref{eq:private-block-variance-bound} in~\eqref{eq:private-overlap-before-uniform-bounds}, we obtain
\begin{equation}
\label{eq:private-uniform-block-overlap}
A_i^{\boldsymbol{W}^{\varepsilon}}
\geq
\exp\left(
-m\overline I_\varepsilon
-\frac{\overline S_\varepsilon}{2}\sqrt m
\right).
\end{equation}
The above bound also holds for identical block PMFs, whose overlap
equals one. Finally, substituting
\eqref{eq:private-uniform-block-overlap} in
\eqref{eq:private-risk-overlap-sum} and noting that $E_{\mathcal K}^{\boldsymbol{W}^{\varepsilon}}(\alpha)\geq E_{\mathcal{K_{M}}}^{\boldsymbol{W}^{\varepsilon}}(\alpha)$, gives us
\[
E_{\mathcal K}^{\boldsymbol{W}^{\varepsilon}}(\alpha)
\ge
\frac{M-1}{M}
\exp\left(
-m\overline I_\varepsilon
-\frac{\overline S_\varepsilon}{2}\sqrt m
\right).
\]
The right-hand side does not depend on $\boldsymbol{W}^{\varepsilon}$. Therefore, we can take the infimum over all admissible collections of $\varepsilon$-LDP mechanisms, and choosing $M=M_\alpha$, yields
\begin{equation}
\label{eq:private-final-lower-bound}
E_{\mathcal K}^{\varepsilon}(\alpha)
\geq
\left(1-\frac1{M_\alpha}\right)
\exp\left(
-(2\alpha+1)\overline I_\varepsilon
-\frac{\overline S_\varepsilon}{2}\sqrt{2\alpha+1}
\right).
\end{equation}
It remains to establish the uniform bounds in~\eqref{eq:private-uniform-bounds} by providing the expressions for $\overline I_{\varepsilon}$ and $\overline S_\varepsilon$. We  do that next to complete our proof.

\medskip
\noindent
Fix any $W^{\varepsilon}\in\mathcal W_\varepsilon$, and let
$Q_h=P_h\cdot W$, where $h \in \{0,1\}$.~We retain the same shorthand notations from the theorem statement
\[
\delta=\mathrm{d}_{\mathrm{TV}}(P_0,P_1),
\qquad
S=\mathrm{D}_\infty^{\mathrm{J}}(P_0,P_1),
\qquad
I=\mathrm{I}_{\,\mathrm{ch}}(P_0,P_1),
\qquad
c_{\varepsilon,2}=\tanh(\varepsilon/2).
\]

\noindent
\emph{(i). The TV-based likelihood-ratio bound:}
We refine the likelihood-ratio argument used in the proof of
Lemma~\ref{cor:ldpp}. For an output $y \in \mathcal{Y}$, let $
u_y=\inf_x W^{\varepsilon}(y\mid x)$, and 
$v_y=\sup_x W^{\varepsilon}(y\mid x)$. The $\varepsilon$-LDP constraint implies that
\[
0<u_y\leq v_y\leq e^\varepsilon u_y.
\]
We know that the positive and negative parts of $P_0-P_1$ each have
total mass $\delta$. Therefore, we can write
\begin{align}
|Q_0(y)-Q_1(y)|
&=
\left|
\sum_x
(P_0(x)-P_1(x))W^{\varepsilon}(y\mid x)
\right|
\notag\\
&\le
\delta(v_y-u_y)
\leq \delta(e^{\varepsilon}-1)u_y.
\end{align}
Since, both $Q_0(y),Q_1(y)\geq u_y$, we have that
\[
 \frac{|Q_0(y)-Q_1(y)|}{\min( Q_0(y),Q_1(y))} \leq \delta(e^{\varepsilon}-1)\]
Therefore,
\[
\left|
\log\frac{Q_0(y)}{Q_1(y)}
\right|
\le
\log\left(1+(e^\varepsilon-1)\delta\right),
\]
and consequently,
\begin{equation}
\label{eq:private-tv-range-bound}
\mathrm{D}_\infty^{\mathrm{J}}(Q_0,Q_1)
\le
2\log\left(1+(e^\varepsilon-1)\delta\right).
\end{equation}

\medskip
\noindent
\emph{(ii) Refining the SDPI argument of
Theorem~\ref{thm:jrdccbd}.}
Suppose first that $0<S<\infty$. For two outputs (with non-zero probability) $y,\,z\, \in \mathcal{Y}$, let us define
\[
r(x)=\frac{W^{\varepsilon}(y\mid x)}{W^{\varepsilon}(z\mid x)},
\qquad
h(x)=\frac{P_0(x)}{P_1(x)},
\qquad
\widetilde p(x)
=
\frac{P_1(x)W^{\varepsilon}(z\mid x)}{Q_1(z)}.
\]
As we see in the normalized representation in the proof (see~\eqref{eq:tbu}) of
Theorem~\ref{thm:jrdccbd},
\begin{equation}
\label{eq:private-output-cross-ratio}
\frac{Q_0(y)Q_1(z)}{Q_1(y)Q_0(z)}
=
\frac{
\mathbb E_{\widetilde p}[rh]
}{
\mathbb E_{\widetilde p}[r]\,
\mathbb E_{\widetilde p}[h]
}.
\end{equation}
Note that the R.H.S is unchanged when $r$ and $h$ are
separately rescaled by positive constants. Moreover, we know that
\[
\frac{\sup_x h(x)}{\inf_x h(x)}=e^S,
\]
and, by the $\varepsilon$-LDP constraint, we have
\[
\frac{r(x)}{r(x')}
=
\frac{W^{\varepsilon}(y\mid x)}{W^{\varepsilon}(y\mid x')}
\frac{W^{\varepsilon}(z\mid x')}{W^{\varepsilon}(z\mid x)}
\le
e^{2\varepsilon}.
\]
Therefore, we can normalize the ranges to
\[
1\leq h(x)\leq R\triangleq e^S,
\qquad
1\leq r(x)\leq T\triangleq e^{2\varepsilon}.
\]
Now, on applying the endpoint reduction and conditional-mean
monotonicity from the proof of Theorem~\ref{thm:jrdccbd}, we can upper bound
\eqref{eq:private-output-cross-ratio} by
\begin{equation}
\label{eq:private-endpoint-optimization}
\max_{0\le\theta\le1}
\frac{1+\theta(RT-1)}
{[1+\theta(R-1)][1+\theta(T-1)]}.
\end{equation}
Here $\theta$ is the mass assigned to the upper endpoint
of $h$; the two conditional means of $r$ are bounded by
$T$ and $1$, respectively. Differentiating the expression in
\eqref{eq:private-endpoint-optimization} w.r.t $\theta$, we see that its maximum is
attained at
\[
\theta^\star=\frac1{1+\sqrt{RT}},
\]
and equals
$
\left(
\frac{1+\sqrt{RT}}{\sqrt R+\sqrt T}
\right)^2
$.
Finally, on taking logarithms of both sides and then supremum over all output pairs $(y,z)$, and from the definition of Jeffreys-Rényi divergence of order infinity, we get 
\begin{align}
\label{eq:private-nonlinear-jrd-bound}
\mathrm{D}_\infty^{\mathrm{J}}(Q_0,Q_1)
&\leq
2\log\left(
\frac{1+e^{S/2+\varepsilon}}
{e^{S/2}+e^\varepsilon}
\right)
\notag\\
&=
4\operatorname{artanh}
\left(
c_{\varepsilon,2}\tanh(S/4)
\right).
\end{align}
If $S=0$, the output distributions coincide. If
$S=\infty$, Lemma~\ref{cor:ldpp} gives $\mathrm{D}_\infty^{\mathrm{J}}(Q_0,Q_1)\le2\varepsilon$, which is precisely the right-hand side of~\eqref{eq:private-nonlinear-jrd-bound} under the convention $\tanh(\infty)=1$.

On combining~\eqref{eq:private-tv-range-bound} and
\eqref{eq:private-nonlinear-jrd-bound}, we obtain
\begin{align}
\label{eq:private-sbar-bound}
\mathrm{D}_\infty^{\mathrm{J}}(Q_0,Q_1)
&\leq
\min\left(
2\log\left(1+(e^\varepsilon-1)\delta\right),\,\,
4\operatorname{artanh}
\left(c_{\varepsilon,2}\tanh(S/4)\right)
\right)=
2\overline S_\varepsilon.
\end{align}

\medskip
\noindent
\emph{(iii). Upper bounding the Chernoff information:}
From~\eqref{eq:ich_fd} in the proof of Theorem~\ref{thm:high_eps_div}, we know that the Chernoff information is related to the $f$-divergences in the following way,
\begin{align}
    1-e^{-\mathrm{I}_{\,\mathrm{ch}}(Q_0,Q_1)}=\sup_{\lambda \in (0,1)}D_{f_{\lambda}}(Q_0\|Q_1),\label{eq:sddpii}
\end{align}
where $f_{\lambda}(u)\triangleq 1-u^{\lambda}$ is nonlinear and operator-convex for $\lambda\in(0,1)$, since
$u\mapsto u^\lambda$ is operator-concave, and $f_\lambda(1)=0$. From Theorem~1 and Remark~1 of Asoodeh and
Zhang~\cite{asoodeh}, we know that for any $f-$divergence whose generator is non-linear and operator-convex,
\[
D_{f}(Q_0\|Q_1) \leq c_{\varepsilon,2}^2 D_{f}(P_0\|P_1),
\]
On applying the above SDPI relation to~\eqref{eq:sddpii}, we obtain that
\begin{align}
 1-e^{-\mathrm{I}_{\,\mathrm{ch}}(Q_0,Q_1)}&\leq c_{\varepsilon,2}^2\sup_{\lambda \in (0,1)}D_{f_{\lambda}}(P_0\|P_1)\notag\\ 
 &=c_{\varepsilon,2}^2  (1-e^{-I})
\end{align}
Consequently, we have
\begin{align}
    \mathrm{I}_{\,\mathrm{ch}}(Q_0,Q_1) &\leq -\log(1-c_{\varepsilon,2}^2+c_{\varepsilon,2}^2e^{-I})\label{eq:chernoff1}\\
    & \leq c_{\varepsilon,2}^2I\label{eq:chernoff2},
\end{align}
where the last inequality follows due to convexity of $t \mapsto e^{-t}$. 
To get an upper bound in terms of $\overline S_{\varepsilon}$, we use a relation between the Chernoff information and the Jeffreys-Rényi divergence of order infinity (which is proved below in Appendix~\ref{app:jrdd}).
\begin{lemma}\label{lem:chernoff_jrd}
For any pair of distinct PMFs $(Q_0,Q_1)$ defined on a common support set $\mathcal{Y}$, we have
\begin{equation}\label{eq:chernoff_jrd}
\mathrm{I}_{\,\mathrm{ch}}(Q_0,Q_1)\le\log\cosh\biggl(\frac{\mathrm{D}^{\mathrm{J}}_{\infty}(Q_0,Q_1)}{4}\biggr).
\end{equation}
The qquality holds for the symmetric binary pair $Q_0=(q,1-q)$ and $Q_1=(1-q,q)$, $q\in(1/2,1)$.
\end{lemma}
Then, using the above result together with~\eqref{eq:private-sbar-bound} and the
monotonicity of $\log\cosh$ on $[0,\infty)$, we finally obtain
\begin{equation}
\label{eq:private-chernoff-range-bound}
\mathrm{I}_{\,\mathrm{ch}}(Q_0,Q_1)
\le
\log\cosh\left(
\frac{\overline S_\varepsilon}{2}
\right).
\end{equation}
On combining~\eqref{eq:chernoff1} and
\eqref{eq:private-chernoff-range-bound},
\begin{align}
\mathrm{I}_{\,\mathrm{ch}}(Q_0,Q_1)
&\le
\min\left(
-\log\left(1-c_{\varepsilon,2}^2(1-e^{-I})\right),\,
\log\cosh(\overline S_\varepsilon/2)
\right)=
\overline I_\varepsilon.
\end{align}
Since $W$ is arbitrary, these inequalities establish
\eqref{eq:private-uniform-bounds} for every coordinate, completing the proof of the theorem. 

\noindent\emph{Proof of Corollary~\ref{cor:minimax_lb_private}}
By convexity of $t\mapsto e^{-t}$, we have
$e^{-c_{\varepsilon,2}^2 I}\leq1-c_{\varepsilon,2}^2+c_{\varepsilon,2}^2 e^{-I}$, and hence
$
\bar I_\varepsilon\le-\log\left(1-c_{\varepsilon,2}^2(1-e^{-I})\right)\leq c_{\varepsilon,2}^2 I$. Moreover, $c_{\varepsilon,2}\tanh(S/4)\le\min\{c_{\varepsilon,2},\tanh(S/4)\}$ and $\operatorname{artanh}$ is increasing on $[0,1)$, so
\[
\overline S_\varepsilon\leq2\operatorname{artanh}\left(c_{\varepsilon,2}\tanh(S/4)\right)
\le\min\{2\operatorname{artanh}(c_{\varepsilon,2}),\,S/2\}=\min\{\varepsilon,\,S/2\}.
\]
Since the right-hand side of~\eqref{eq:private-minimax-lb} is decreasing in $\bar I_\varepsilon$ and $\overline S_\varepsilon$, the claim follows from Theorem~\ref{thm:private-minimax-lb}. This
completes our proof.
\end{proof}

\subsection{Proof of Lemma~\ref{lem:chernoff_jrd}}\label{app:jrdd}
\begin{proof}
Let $\ell(y)\triangleq\log\frac{Q_0(y)}{Q_1(y)}$, $a\triangleq\inf_y\ell(y)$ and
$b\triangleq\sup_y\ell(y)$, such that $\Delta\triangleq \mathrm{D}^{\mathrm{J}}_{\infty}(Q_0,Q_1)=b-a>0$ (as
$Q_0\neq Q_1$). Assume $\Delta<\infty$. Let
$Z(\lambda)\triangleq\sum_{y}Q_0(y)^{\lambda}Q_1(y)^{1-\lambda}=\mathbb{E}_{Q_1}[e^{\lambda\ell}]$,
such that, we have $\mathrm{I}_{\,\mathrm{ch}}(Q_0,Q_1)=\max_{\lambda\in[0,1]}[-\log Z(\lambda)]$. Since
$Z(0)=Z(1)=1$, it suffices to show that $Z(\lambda)\ge1/\cosh(\Delta/4)$ for every
$\lambda\in(0,1)$.

Fix $\lambda\in(0,1)$. Since $t\mapsto t^{\lambda}$ is concave, it lies above its chord on
$[e^{a},e^{b}]$. Applying this at $t=e^{\ell(y)}$, averaging over $Q_1$, using
$\mathbb{E}_{Q_1}[e^{\ell}]=1$, and writing $a=c-\Delta/2$, $b=c+\Delta/2$, we obtain the following
\[
Z(\lambda)\ge\frac{(e^{b}-1)e^{\lambda a}+(1-e^{a})e^{\lambda b}}{e^{b}-e^{a}}
=\frac{e^{\lambda c}\sinh\left(\frac{(1-\lambda)\Delta}{2}\right)
+e^{-(1-\lambda)c}\sinh\left(\frac{\lambda\Delta}{2}\right)}{\sinh(\Delta/2)}
\ge\frac{e^{\phi(\lambda)+\phi(1-\lambda)}}{\sinh(\Delta/2)},
\]
where $\phi(t)\triangleq t\log\frac{\sinh(t\Delta/2)}{t}$, and the last step holds due to the weighted
AM--GM inequality with weights $1-\lambda$ and $\lambda$, under which the powers of $e^{c}$
cancel out. The function $\phi$ is convex on $(0,1]$, since
$\phi''(t)=m(t\Delta)/\left(4t\sinh^{2}(t\Delta/2)\right)$, where
$m(y)\triangleq2y\sinh y-y^{2}-2\cosh y+2\ge0$ because $m(0)=0$ and
$m'(y)=2y(\cosh y-1)\ge0$. Hence
$\phi(\lambda)+\phi(1-\lambda)\ge2\phi(1/2)=\log\left(2\sinh(\Delta/4)\right)$, and therefore
$Z(\lambda)\ge2\sinh(\Delta/4)/\sinh(\Delta/2)=1/\cosh(\Delta/4)$.

For the symmetric binary pair, we have that $Z(1/2)=2\sqrt{q(1-q)}=1/\cosh(\Delta/4)$. As a result, $\mathrm{I}_{\,\mathrm{ch}}(Q_0,Q_1)\ge-\log Z(1/2)=\log\cosh(\Delta/4)$, and equality holds. This completes our proof.
\end{proof}

\subsection{Proof of Theorem~\ref{thm:high_eps_div} (Chernoff-optimality of $\varepsilon$-$\mathrm{ARR}$ under low privacy)}\label{app:high_eps_div}
\begin{proof}

Our proof is inspired by the techniques used in~\cite{kairouz2014extremal}. We are interested in the $f$-divergences $D_{f_\lambda}(Q_0\|Q_1)$ with $f_\lambda(t) =  1 - t^\lambda$, for $\lambda \in (0,1)$. These can be decomposed into a sum of sublinear functions:
\begin{align}
    D_{f_\lambda}(Q_0\|Q_1) = \sum_{y \in \mathcal{Y}} \mu^\lambda(y),
\end{align}
where
\begin{align}
    \mu^\lambda(y) = \sum_{x \in \mathcal{X}} P_1(x) W_{Y|X}(y|x) - \left( \sum_{x \in \mathcal{X}} P_0(x) W_{Y|X}(y|x) \right)^\lambda \left( \sum_{x \in \mathcal{X}} P_1(x) W_{Y|X}(y|x) \right)^{1-\lambda}.
\end{align}

The sufficient statistic for $f$-divergences $D_{f_\lambda}(P_0\|P_1)$ and Chernoff information $\mathrm{I}_{\,\mathrm{ch}}{(P_0,P_1)}$ is the likelihood ratio $P_0(x)/P_1(x)$. We aggregate the probability distributions into classes $C_1,\ldots,C_r$ by equality of the likelihood ratio $P_0(x)/P_1(x)$. With this aggregation, $D_{f_\lambda}(P_0\|P_1)$, $\mathrm{I}_{\,\mathrm{ch}}{(P_0,P_1)}$ and the set of induced probability distributions under $\varepsilon$-LDP remain the same. If all likelihood ratios are distinct, this coincides with relabeling. In the following, we use aggregated distributions with distinct likelihood ratios, while denoting them with $P_0(x)$ and $P_1(x)$ for simplicity.

We fix a $\lambda$. From \cite{kairouz2014extremal} [Theorem 4], we have that
\begin{align}
\max_{W_{Y|X}\in\mathcal{W}_{\varepsilon}} D_{f_\lambda}(Q_{0}\|Q_{1}) = \text{maximize } & {\mu^\lambda}^{T}\theta \\
\text{subject to } & S^{(k)}\theta=\boldsymbol{1} \\
& \theta\ge0,
\end{align}
where $k=|\mathcal{X}|$, $\mu^\lambda_{j}=\mu^\lambda(S_{j}^{(k)})$ for $j\in\{1,...,2^{k}\}$ and $S^{(k)}$ is the $k\times2^{k}$ staircase pattern matrix. 
Since there is no duality gap in this linear program, we can equivalently solve its dual:
\begin{align}
\text{minimize } & \boldsymbol{1}^{T}\alpha \\
\text{subject to } & {S^{(k)}}^{T}\alpha\ge\mu^\lambda.
\end{align}
We define $T_{j}=\{x_{i}:S_{ij}^{(k)}=e^{\varepsilon}\}$ for $j\in[2^{k}]$ and set $j_{i}=\{j:T_{j}=\{ x_{i} \} \}$ for $i\in[k]$. We propose the following dual solution:
\begin{align}
\alpha_{i}^{*} = \frac{1}{e^\varepsilon-1} \left(\mu^\lambda_{j_i} - \frac{1}{e^\varepsilon+k-1} \sum_{l \in [k]} \mu^\lambda_{j_l} \right),
\end{align}
for $i \in [k]$. We first calculate the objective value attained by this dual solution. We note that
\begin{align}
    \mu^\lambda_{j} &= \sum_{i \in [k]} P_1(x_i) S_{ij}^{(k)} - \left( \sum_{i \in [k]} P_0(x_i) S_{ij}^{(k)} \right)^\lambda \left( \sum_{i \in [k]} P_1(x_i) S_{ij}^{(k)} \right)^{1-\lambda} \\
    &= \left( P_1(T_j) (e^\varepsilon-1) + 1 \right) - \left( P_0(T_j) (e^\varepsilon-1) + 1 \right)^\lambda \left( P_1(T_j) (e^\varepsilon-1) + 1 \right)^{1-\lambda} \\
    &= (P_1(T_j) (e^\varepsilon-1) + 1) ^{1-\lambda} \left( (P_1(T_j) (e^\varepsilon-1) + 1) ^{\lambda} - \left( P_0(T_j) (e^\varepsilon-1) + 1 \right)^\lambda \right) \\
    &= g^\varepsilon_1(T_j)^{1-\lambda} \left( g^\varepsilon_1(T_j)^{\lambda} - g^\varepsilon_0(T_j)^\lambda \right),
\end{align}
where $g_\nu^\varepsilon(T) = P_\nu(T) (e^\varepsilon-1) + 1 $.
Then,
\begin{align}
    \boldsymbol{1}^{T}\alpha^* & = \sum_{i \in [k]} \alpha_{i}^{*} \\
    &= \sum_{i \in [k]} \frac{1}{e^\varepsilon-1} \left(\mu^\lambda_{j_i} - \frac{1}{e^\varepsilon+k-1} \sum_{l \in [k]} \mu^\lambda_{j_l} \right) \\
    & = \frac{1}{e^\varepsilon-1} \sum_{i \in [k]} \mu^\lambda_{j_i} - \frac{1}{e^\varepsilon-1} \frac{k}{e^\varepsilon+k-1} \sum_{i \in [k]} \mu^\lambda_{j_i} \\
    & = \frac{1}{e^\varepsilon-1} \left(1 - \frac{k}{e^\varepsilon+k-1} \right) \sum_{i \in [k]} \mu^\lambda_{j_i} \\
    & = \frac{1}{e^\varepsilon-1} \frac{e^\varepsilon-1}{e^\varepsilon+k-1} \sum_{i \in [k]} \mu^\lambda_{j_i} \\
    & = \frac{1}{e^\varepsilon+k-1} \sum_{i \in [k]} \mu^\lambda_{j_i} \\
    & = \frac{1}{e^\varepsilon+k-1} \sum_{i \in [k]} \Bigg(( P_1(x_i) (e^\varepsilon-1) + 1) ^{1-\lambda} \notag \\& \hspace{25mm} \cdot\left( \left(P_1(x_i) (e^\varepsilon-1) + 1 \right) ^{\lambda} - \left( P_0(x_i) (e^\varepsilon-1) + 1 \right)^\lambda \right) \Bigg)\\
    & = \frac{1}{e^\varepsilon+k-1} \sum_{i \in [k]} g^\varepsilon_1(x_i)^{1-\lambda} \left( g^\varepsilon_1(x_i)^{\lambda} - g^\varepsilon_0(x_i)^\lambda \right). \label{eq:dual_obj}
\end{align}

We argue that this dual solution is feasible for high enough $\varepsilon$. We have
\begin{align}
    g_\nu^\varepsilon(T) &= P_\nu(T) e^\varepsilon + O(1) \\
    \mu_j^\lambda &= g^\varepsilon_1(T_j)^{1-\lambda} \left( g^\varepsilon_1(T_j)^{\lambda} - g^\varepsilon_0(T_j)^\lambda \right) \\
    &= (P_1(T_j) e^\varepsilon + O(1))^{1-\lambda} \left( (P_1(T_j) e^\varepsilon + O(1))^{\lambda} - (P_0(T_j) e^\varepsilon + O(1))^\lambda \right) \\
    &= (P_1(T_j) e^\varepsilon)^{1-\lambda} \left( (P_1(T_j) e^\varepsilon)^{\lambda} - (P_0(T_j) e^\varepsilon)^\lambda \right) + O(1) \\
    &= e^\varepsilon \left( P_1(T_j) -  P_1(T_j)^{1- \lambda} P_0(T_j)^\lambda \right) + O(1)
\end{align}
that are valid for large $\varepsilon$. We have that
\begin{align}
    {S^{(k)}_{j}}^{T}\alpha^{*} &= \frac{1}{e^\varepsilon-1} \sum_{i \in [k]} {S_{ij}^{(k)}}^{T} \left(\mu^\lambda_{j_i} - \frac{1}{e^\varepsilon+k-1} \sum_{l \in [k]} \mu^\lambda_{j_l} \right) \\
    &= \frac{1}{e^\varepsilon-1} \sum_{i \in [k]} {S_{ij}^{(k)}}^{T} \mu^\lambda_{j_i} - \frac{1}{e^\varepsilon-1} \frac{1}{e^\varepsilon+k-1} \sum_{i \in [k]} {S_{ij}^{(k)}}^{T} \sum_{l \in [k]} \mu^\lambda_{j_l} \\
    &= \sum_{i \in [k], x_i \in T_j} \mu^\lambda_{j_i} + \frac{1}{e^\varepsilon-1} \left( 1 - \frac{(e^\varepsilon - 1)|T_j| + k}{e^\varepsilon+k-1} \right) \sum_{i \in [k]} \mu^\lambda_{j_i} \\
    &= \sum_{i \in [k], x_i \in T_j} \mu^\lambda_{j_i} + \frac{1}{e^\varepsilon-1} \frac{(e^\varepsilon - 1)(1-|T_j|)}{e^\varepsilon+k-1} \sum_{i \in [k]} \mu^\lambda_{j_i} \\
    &= e^\varepsilon \sum_{i \in [k], x_i \in T_j} \left( P_1(x_i) -  P_1(x_i)^{1- \lambda} P_0(x_i)^\lambda \right)\notag \\ & \hspace{20mm}- \frac{e^\varepsilon(|T_j|-1)}{e^\varepsilon+k-1} \sum_{i \in [k]} \left( P_1(x_i) -  P_1(x_i)^{1- \lambda} P_0(x_i)^\lambda \right) + O(1) \\
    &= e^\varepsilon \sum_{i \in [k], x_i \in T_j} \left( P_1(x_i) -  P_1(x_i)^{1- \lambda} P_0(x_i)^\lambda \right) + O(1).
\end{align}

We calculate the following to show the feasibility of the dual solution:
\begin{align}
    {S^{(k)}_{j}}^{T}\alpha^{*} - \mu_{j}^\lambda\notag \\ &\hspace{-17mm} = e^\varepsilon \left( \sum_{i \in [k], x_i \in T_j} \left( P_1(x_i) -  P_1(x_i)^{1- \lambda} P_0(x_i)^\lambda \right) - \left( P_1(T_j) -  P_1(T_j)^{1- \lambda} P_0(T_j)^\lambda \right) \right) + O(1) \\
    &\hspace{-17mm} = e^\varepsilon \left(\left( P_1(T_j)^{1- \lambda} P_0(T_j)^\lambda \right) - \sum_{i \in [k], x_i \in T_j} \left(  P_1(x_i)^{1- \lambda} P_0(x_i)^\lambda \right) \right) + O(1). \label{eq:ich_optim}
\end{align}
Set $j_0 = \{j:T_{j}= \emptyset \}$. For $j\notin\{j_{0},j_{1},j_{2},...,j_{k}\}$, $\left( P_1(T_j)^{1- \lambda} P_0(T_j)^\lambda \right) > \sum_{i \in [k], x_i \in T_j} \left(  P_1(x_i)^{1- \lambda} P_0(x_i)^\lambda \right)$ from Hölder's inequality~\cite{holder}. For $j_0 = \{j:T_{j}= \emptyset \}$, ${S^{(k)}}_{j_0}^{T}\alpha^{*} - \mu_{j_0}^\lambda ={S^{(k)}}_{j_0}^{T}\alpha^{*} \geq 0$.  For $i \in [k]$, it is easy to see that ${S^{(k)}}_{j_i}^{T}\alpha^{*} - \mu_{j_i}^\lambda = 0$. This concludes that ${S^{(k)}}_{j}^{T}\alpha^{*} - \mu_{j}^\lambda \geq 0$ for all $j \in [2^k]$, given $\varepsilon>\varepsilon^*_\lambda$, where $\varepsilon^*_\lambda$ depends on $P_0$, $P_1$, and the value of the fixed $\lambda$.

We finally note that the objective value attained by the randomized response mechanism in the primal linear program matches that of our dual solution calculated in Eq.~\ref{eq:dual_obj}:
\begin{align}
    D_{f^\lambda}(Q_0\|Q_1) &= \frac{1}{e^\varepsilon+k-1} \sum_{i \in [k]} \mu^\lambda_{j_i} \\
    &= \frac{1}{e^\varepsilon+k-1} \sum_{i \in [k]} g^\varepsilon_1(x_i)^{1-\lambda} \left( g^\varepsilon_1(x_i)^{\lambda} - g^\varepsilon_0(x_i)^{\lambda} \right).
\end{align}
From the duality theorem, we conclude that the randomized response mechanism is optimal for $f$-divergences with $f_\lambda(t) = 1 - t^{\lambda}$ given $\varepsilon \geq \varepsilon^*_\lambda$ and distributions $P_0, P_1$ with distinct likelihood ratios. For general distributions, this corresponds to the aggregated randomized response mechanism. We know that,
\begin{align}
    \mathrm{I}_{\,\mathrm{ch}}{(Q_0, Q_1)} &=  - \min_{\lambda \in (0,1)} \log \left( \sum_{y \in \mathcal{Y}} Q_0(y)^\lambda Q_1(y)^{1-\lambda}  \right) \\
    &= - \log \left( 1 - \max_{\lambda \in (0,1)} D_{f_\lambda}(Q_0\|Q_1) \right),\label{eq:ich_fd}
\end{align}
where $f_\lambda(t) = 1 - t^\lambda$. We have that an optimal mechanism for $\max_{\lambda \in (0,1)} D_{f_\lambda}(Q_0\|Q_1)$ is also optimal for Chernoff information. 
Next, we will prove that there exists an $\varepsilon^*$ such that for $\varepsilon \ge \varepsilon^*$, $\max_{\lambda \in (0,1)} D_{f_\lambda}(Q_0\|Q_1)$ is maximized with the aggregated randomized response mechanism under $\varepsilon$-LDP. To prove this, given $ D_{f_\lambda}(P_0\|P_1) = 1-\sum_x P_0(x)^\lambda P_1(x)^{1-\lambda}$, we define 
\begin{align}
    \max_{\lambda \in (0,1)} D_{f_\lambda}(P_0\|P_1) = M \\
    \lambda_0 \in \arg\max_{\lambda \in (0,1)} D_{f_\lambda}(P_0\|P_1).
\end{align}
Knowing that $D_{f_\lambda}(P_0\|P_1) \to 0 $ at both endpoints, there is a threshold $\lambda' < 1/2$ such that
\begin{align}
    \sup_{\lambda\in(0,\lambda') \cup (1-\lambda', 1)}  D_{f_\lambda}(P_0\|P_1) < M/2.
\end{align}
As $\varepsilon \to \infty$, $ D_{f_\lambda}(Q_0\|Q_1) \to  D_{f_\lambda}(P_0\|P_1)$. This implies that there is a finite $\varepsilon_0$ such that, for every $\varepsilon\ge\varepsilon_0$,
\begin{align}
    D_{f_{\lambda_0}}(Q_0\|Q_1)>M/2,
\end{align}
with the aggregated randomized response mechanism.
From the data processing inequality, we know that $\sup_{\lambda\in(0,\lambda') \cup (1-\lambda', 1)}  D_{f_\lambda}(Q_0\|Q_1) \leq \sup_{\lambda\in(0,\lambda') \cup (1-\lambda', 1)}  D_{f_\lambda}(P_0\|P_1) < M/2$, Then, we have that for the ARR mechanism and $\varepsilon\ge\varepsilon_0$
\begin{align}
    \max_{\lambda \in (0,1)} D_{f_\lambda}(Q_0\|Q_1) = \max_{\lambda \in [\lambda',1-\lambda']} D_{f_\lambda}(Q_0\|Q_1).
\end{align}
We have that
\begin{align}
    \min_{\substack{\lambda\in [\lambda',1-\lambda'] \\ |T|\ge2}} \left[ P_0(T)^\lambda P_1(T)^{1-\lambda} 
    -\sum_{x\in T}P_0(x)^\lambda P_1(x)^{1-\lambda} \right] > 0,
\end{align}
since Hölder inequality is strict, the functions are continuous, and there are finitely many subsets. Since the distributions are fully supported, remainders in \eqref{eq:ich_optim} are uniformly bounded on the compact interval $\lambda\in [\lambda',1-\lambda']$. Thus, Chernoff information is optimized by ARR for sufficiently large $\varepsilon$.

\end{proof}

\subsection{Proof of Corollary~\ref{cor:small-large-eps} (Chernoff-optimality of $\varepsilon$-$\mathrm{BM}$ under high privacy) }\label{app:small-large-eps}
\begin{proof}

We start with a lemma on the optimality of a binary mechanism on a class of $f$-divergences. This lemma is very similar to ~\cite{kairouz2014extremal}[Theorem $5$], however we noticed that their proof omits the equal probability case (when for some $x \in \cal{X}$, $P_0(x) = P_1(x)$). Here, we prove the following lemma to resolve this small issue.

\begin{lemma}\label{lem:binary-mech-f-div}
Let $(P_0,P_1)$ be any two PMFs. There exists a threshold $\varepsilon'(P_0,P_1)>0$, such that for all positive $\varepsilon \leq\varepsilon'$, all $\lambda \in (0,1)$ and $f_\lambda(u)=1-u^\lambda$, a binary mechanism (with appropriately chosen $\tau$, dependent on $\varepsilon$ and $\lambda$) maximizes the $f$-divergence $D_{f_\lambda}(Q_0\|Q_1)$ .
\end{lemma}

\begin{proof}
We first note that
\begin{align}
  D_{f_\lambda}(Q_0\|Q_1) = 1-\sum_y Q_0(y)^\lambda Q_1(y)^{1-\lambda}.
\end{align}
We define a partition of $\cal{X}$ into $A,B,N$ as follows:
\begin{align}
  A=\{x:P_0(x)>P_1(x)\},\qquad
  B=\{x:P_0(x)<P_1(x)\},\qquad
  N=\{x:P_0(x)=P_1(x)\}.
\end{align}

We are interested in binary mechanisms with $\tau = 1$ and $\tau$ slightly larger than $1$. These mechanisms correspond to the quantized set pairs $({A} \cup {N}, {B})$ and $({A}, {B} \cup {N})$. Without loss of generality, we assume that the equality set $N$ has a single element or is empty. If there are more elements, we can aggregate them into a single symbol.
We define
\begin{align}
 \Delta_T=P_0(T)-P_1(T),\qquad
  \delta=\sum_{x \in {A}} ( P_0(x)-P_1(x)).
\end{align}
We have that $\Delta_{A}=\delta$ and $\Delta_B=-\delta$. We now fix $\varepsilon>0$ and $0<\lambda<1$, and set
\begin{align}
  c_\lambda=\frac{\lambda(1-\lambda)}2,\qquad
  \phi_\lambda(a,b)=\lambda a+(1-\lambda)b-a^\lambda b^{1-\lambda}.
\end{align}
Notice that
\begin{equation}
\sum_y\phi_\lambda(Q_0(y),Q_1(y)) =1-\sum_yQ_0(y)^\lambda Q_1(y)^{1-\lambda} =D_{f_\lambda}(Q_0\|Q_1).
\end{equation}

We use~\cite{kairouz2014extremal}[Theorem $4$] in our proof. This theorem allows us to express the optimization as a linear program. We note that our utility functions of interest satisfy the assumptions of this theorem. To show the optimality of one of the two binary mechanisms, we construct a feasible dual solution with an objective value matching that of the better binary mechanism.

For each ${T} \subseteq \cal{X}$, define a staircase column and its utility by
\begin{align}
    s_T(x)=1+(e^\varepsilon - 1)\boldsymbol{1}_{\{x\in T\}},\qquad
    u_T=\frac{\phi_\lambda(1+(e^\varepsilon-1)P_0(T),1+(e^\varepsilon-1)P_1(T))}{c_\lambda (e^\varepsilon - 1)^2}.
\end{align}
We divide $\phi_\lambda$ by $c_\lambda (e^\varepsilon - 1)^2$ to remove the common leading factor. This scaling does not change the maximizing mechanism. We define the matrix $S$ via the columns $s_T$, indexed by all subsets of $\cal{X}$. With ~\cite{kairouz2014extremal}[Theorem $4$], the optimization problem can be written as
\begin{equation}\label{eq:primal}
  \sup_{W\text{ is }\varepsilon\text{-LDP}}
  \frac{D_{f_\lambda}(P_0W,P_1W)}{c_\lambda (e^\varepsilon - 1)^2}
  =\max_{\theta\ge0:\,S\theta=\boldsymbol{1}}\sum_Tu_T\theta_T.
\end{equation}
For a binary partition $H,H^c$, the equivalent primal solution is $\theta_H=\theta_{H^c}=1/(1+e^\varepsilon)$, with an objective value of $(u_H+u_{H^c})/(1+e^\varepsilon)$. The two binary mechanisms defined by the sets $A$ and $A \cup N$ achieve
\begin{equation}\label{eq:binary-values}
  V_{>} =\frac{u_A+u_{B\cup N}}{1+e^\varepsilon},\qquad
  V_{\ge}=\frac{u_{A\cup N}+u_B}{1+e^\varepsilon}.
\end{equation}
We let $V=\max\{V_{>},V_{\ge}\}$. 

The dual program is
\begin{equation}\label{eq:dual}
  \min_{\alpha\in\mathbb R^{|\cal{X}|}}\sum_x\alpha_x
  \quad\text{subject to}\quad
  s_T^\top\alpha\ge u_T
  \quad\text{for every }T\subseteq\cal{X}.
\end{equation}
Without loss of generality, we construct the dual solution when $V=V_{>}$.
We let $a_A,a_N,a_B$ denote the sums of $\alpha_x$ over the respective sets. Making the constraints for $A$ and $A\cup N$ hold with equality, and imposing $a_A+a_N+a_B=V$, determines
\begin{equation}\label{eq:group-totals}
  \begin{aligned}
    a_A&=\frac{u_A-V}{e^\varepsilon-1},\qquad
    a_N=\frac{u_{A\cup N}-u_A}{e^\varepsilon-1},\\
    a_B&=V-a_A-a_N
        =\frac{e^\varepsilon V-u_{A\cup N}}{e^\varepsilon-1}.
  \end{aligned}
\end{equation}
We fix a representative from each group and set the corresponding $\alpha$ to the set sum. All other coordinates of $\alpha$ are set to zero. If $N=\emptyset$, then $a_N=0$ and no representative is needed. This construction ensures both $\sum_x\alpha_x=V$ and
\begin{equation}\label{eq:alpha-norm}
  \|\alpha\|_1=|a_A|+|a_N|+|a_B|.
\end{equation}
With this dual solution, each constraint can be written as
\begin{equation}\label{eq:constraint-identity}
  s_T^\top\alpha
  =\sum_x\alpha_x+(e^\varepsilon-1)\sum_{x\in T}\alpha_x
  =V+(e^\varepsilon-1)\sum_{x\in T}\alpha_x.
\end{equation}
To establish optimality, we need to show that our proposed dual solution is feasible.
We substitute the set totals into \eqref{eq:constraint-identity} to verify the following constraints:
\begin{align*}
  s_A^\top\alpha
    &=V+(e^\varepsilon-1)a_A=u_A,\\
  s_{A\cup N}^\top\alpha
    &=V+(e^\varepsilon-1)(a_A+a_N)=u_{A\cup N},\\
  s_{B\cup N}^\top\alpha
    &=V+(e^\varepsilon-1)(V-a_A)=(1+e^\varepsilon)V-u_A=u_{B\cup N},\\
  s_B^\top\alpha
    &=V+(e^\varepsilon-1)(V-a_A-a_N)=(1+e^\varepsilon)V-u_{A\cup N}\ge u_B.
\end{align*}
The third equality is due to $V=V_{>}$. The final inequality is a consequence of $V\ge V_{\ge}$.

To verify the remaining constraints, we show that the leading term of the normalized utility $u_T$ is $\Delta_T^2$ and its error bound is independent of $\lambda$. We let $a,b\in[0,1]$ and $h(v)=\phi_\lambda(1+va,1+vb)$. We calculate the derivatives: $h(0)=h'(0)=0$ and $h''(v)=\lambda(1-\lambda)(a-b)^2 (1+va)^{\lambda-2}(1+vb)^{-1-\lambda}$. For $0\le v\le e^\varepsilon-1$, the product of the last two factors is sandwiched between $e^{-3\varepsilon}$ and $1$, uniformly over $0<\lambda<1$. We know that $h(e^\varepsilon-1)=\int_0^{(e^\varepsilon - 1)}((e^\varepsilon - 1)-v)h''(v)\,dv$. We get
\begin{equation}\label{eq:uniform-bound}
  e^{-3\varepsilon}\Delta_T^2\le u_T\le\Delta_T^2,
\end{equation}
by calculating the integral and then dividing the result by $c_\lambda (e^\varepsilon - 1)^2$.
Since $|\Delta_T|\le\delta$ and $1-e^{-3\varepsilon}\le 3(e^\varepsilon - 1)$, we can express $u_T$ as
\begin{equation}\label{eq:uniform}
  u_T=\Delta_T^2+ (e^\varepsilon - 1) e_T,
  \qquad |e_T|\le C_0=3\delta^2.
\end{equation}
This bound on $e_T$ holds for all subsets $T$ and all $\lambda \in (0,1)$.
We can plug this expression of $u_T$ into the dual vector and bound its value. Since $\Delta_T^2=\delta^2$for $T \in \{A, A \cup N, B, B \cup N \}$, and $V=V_{>}$,
\[
  V=\delta^2+(e^\varepsilon -1)e_V,\qquad
  e_V=\frac{e_A+e_{B\cup N}-\delta^2}{1+e^\varepsilon},\qquad
  |e_V|\le C_1=C_0+\frac{\delta^2}{2}.
\]
Then we have
\eqref{eq:group-totals}:
\begin{align}
    a_A=e_A-e_V,\qquad |a_A|\le C_0+C_1 \\
    a_N=e_{A\cup N}-e_A,\qquad |a_N|\le2C_0 \\
    a_B=V-a_A-a_N, \qquad |a_B|\le\delta^2+3C_0+C_1. 
\end{align}

where the last inequality is due to $0\le V\le\delta^2$ via \eqref{eq:uniform-bound}. Together with \eqref{eq:alpha-norm}, we have
\begin{equation}\label{eq:dual-bound}
  \|\alpha\|_1\le M=\delta^2+6C_0+2C_1,
\end{equation}
where $M$ is independent of $\varepsilon,\lambda$, and the choice of the better binary mechanism.

For any subset $T$, we define the dual slack as $\sigma_T=s_T^\top\alpha-u_T$. We want to show that $\sigma_T\ge0$ for all subsets $T$, making the dual solution feasible. We use\eqref{eq:uniform} and \eqref{eq:constraint-identity} to calculate
\begin{equation}\label{eq:slack}
  \begin{aligned}
    \sigma_T
      &=V+(e^\varepsilon-1)\sum_{x\in T}\alpha_x-u_T\\
      &=\delta^2-\Delta_T^2
        +(e^\varepsilon-1)\left(e_V+\sum_{x\in T}\alpha_x-e_T\right).
  \end{aligned}
\end{equation}
The expression in parentheses is bounded:
\[
  |e_V|+\|\alpha\|_1+|e_T|
  \le C_1+M+C_0=C,
\]
independently of $T$ and $\lambda$.
For subsets $T\notin\{A,A\cup N,B,B\cup N\}$, we have that $|\Delta_T|<\delta$, since $\Delta_T=\delta$ only if $A \subseteq T$ and $B \cap T = \emptyset$, whereas $\Delta_T=-\delta$ only if $B \subseteq T$ and $A \cap T = \emptyset$. Thus, we can write the minimum gap as
\[
  g=\min_{T\notin\{A,A\cup N,B,B\cup N\}}
       (\delta^2-\Delta_T^2)>0.
\]
For every $T\notin\{A,A\cup N,B,B\cup N\}$, \eqref{eq:slack} yields
\[
  \sigma_T\ge g-(e^\varepsilon - 1)C\ge\frac g2>0,
\]
given $0<e^\epsilon - 1 \le \frac{g}{2C}$, or equivalently $0<\varepsilon\le\log(1+\frac{g}{2C})$. Then, the dual solution is feasible. From the duality theorem, we conclude that one of the two binary mechanisms is optimal under $\varepsilon$-LDP, provided that $\varepsilon \le log(1+\frac{g}{2C})$.
\end{proof}

Given an $\varepsilon'$ that depends on $P_0$ and $P_1$, our lemma above shows that a binary mechanism (among two possible ones) maximizes the $f$-divergence between the induced marginals over all $\varepsilon$-LDP mechanisms, assuming $\varepsilon \leq \varepsilon'$. Then, we use the relation between Chernoff information and $f-$divergences as in Eq.~\ref{eq:ich_fd}.
\begin{align}
\max_{\substack{W_{Y|X} \in \mathcal{W_\varepsilon}, \\ Q_i = P_i W_{Y|X}}} \mathrm{I}_{\,\mathrm{ch}}{(Q_0, Q_1)} &= - \log \left( 1 - \max_{\substack{W_{Y|X} \in \mathcal{W^\varepsilon}, \\ Q_i = P_i W_{Y|X}}} \max_{\lambda \in (0,1)} D_{f_\lambda}(Q_0\|Q_1) \right)
\end{align}
where $f_\lambda(t) = 1 - t^\lambda$.

We denote the set of all $\varepsilon$-LDP mechanisms with $\mathcal{W^\varepsilon}$. We let \(W_>,W_\ge\) denote the two possibly optimal binary mechanisms and $\nu\in\{>,\ge\}$. We know that $D_{f_\lambda}(P_0W_\nu\|P_1W_\nu)$ is continuous and positive. Also, note that $D_{f_\lambda}(P_0 W_\nu\|P_1W_\nu) \to 0$ as $\lambda \to 0$ and $\lambda \to 1$. Then, from the lemma above, we have that
\begin{align}
 \max_{W \in \cal{W}_\varepsilon}\max_{\lambda \in (0,1)} D_{f_\lambda}(P_0W\|P_1W) = \max_{\nu \in\{>,\ge\}}\max_{\lambda \in (0,1)} D_{f_\lambda}(P_0W_\nu\|P_1W_\nu).   
\end{align}
Knowing \(z\mapsto-\log(1-z)\) is an increasing function, we have that
\begin{align}
    \max_{W \in \cal{W}_\varepsilon} {\rm I_{ch}}(P_0W,P_1W) = \max_{\nu\in\{>,\ge\}}{\rm I_{ch}}(P_0W_\nu,P_1W_\nu),
\end{align}
which concludes our proof.

\end{proof}

\subsection{Proof of Theorem~\ref{thm:rrcpd}  (Finite sample guarantees for $\mathrm{ARR-CPD}$ algorithm)}\label{app:rrcpd}
\begin{proof}
Our proof relies on the fact that the  {$\mathrm{ARR-CPD}$} estimator in Algorithm~\ref{alg:offline_rrcpd} first privatizes the dataset using the aggregated randomized response (which satisfies $\varepsilon$-LDP) and then applies the GLRT (Algorithm~\ref{alg:offline_cpd}) on the privatized dataset to estimate the change-point.

Let $(\mathcal{D},P_0,P_1)$ denote the dataset, pre-change distribution, and the post-change distribution before privatization. Let the privacy parameter be $\varepsilon>0$. After privatization via aggregated randomized response ($\varepsilon$-$\mathrm{ARR}$), let the tuple be denoted by $(\widetilde{\mathcal{D}},Q_0,Q_1)$. Then, from our results in Theorem~\ref{thm:npcpd}, we have
\begin{align}
       \beta_{\mathrm{ARR}} \leq \,\,&2\, e^{-(\alpha+1) \mathrm{I}_{\,\mathrm{ch}}(Q_0,Q_1)}\left(1-\frac{e^{-(\alpha+1) \mathrm{I}_{\,\mathrm{ch}}(Q_0,Q_1)}}{2}\right). \label{eq:cpdrr}
\end{align}
Since, the R.H.S in~\eqref{eq:cpdrr} is non-increasing in $\mathrm{I}_{\,\mathrm{ch}}$, our goal is to lower bound $\mathrm{I}_{\,\mathrm{ch}}(Q_0,Q_1)$ in terms of quantities depending on $(P_0,P_1)$, to obtain a meaningful private finite-sample guarantee.

Let $t=e^\varepsilon-1$. Recall that $\delta=\mathrm{d}_{\mathrm{TV}}(P_0,P_1)$, and
$D=\mathrm{H}^2(P_0,P_1)$,
where $\mathrm{H}^2$ is the squared-Hellinger distance defined as
\[
\mathrm{H}^2(P,Q)
=\frac{1}{2}\sum_{x \in \mathcal{X}}\left(\sqrt{P(x)}-\sqrt{Q(x)}\right)^2
=1-\sum_{x \in \mathcal{X}}\sqrt{P(x)Q(x)}.
\]
Let $\mathcal X_1 \subseteq \mathcal X$ contain one representative symbol from each likelihood-ratio class, and let $r=|\mathcal X_1|$.
For $y\in\mathcal X_1$, let's denote its class by $C_y$ and define
$ \widetilde P_\nu(y)=\sum_{x\in C_y}P_\nu(x)$, for
$\nu\in\{0,1\}$, as the aggregated PMFs. Then, the output distributions of $\varepsilon$-ARR satisfy
\[
Q_\nu(y)
=
\frac{1+t\widetilde P_\nu(y)}{r+t},
\qquad y\in\mathcal X_1.
\]
Now, we know that within each class $C_y$, the likelihood ratio is constant. Consequently, we have

\begin{align*}
\left|\widetilde P_0(y)-\widetilde P_1(y)\right|
=\sum_{x\in C_y}|P_0(x)-P_1(x)|,\qquad \text{and} \qquad
\sqrt{\widetilde P_0(y)\widetilde P_1(y)}
=\sum_{x\in C_y}\sqrt{P_0(x)P_1(x)}.
\end{align*}
On summing over all the classes, we get
\[
\mathrm{d}_{\mathrm{TV}}(\widetilde P_0,\widetilde P_1)=\delta,
\qquad \text{and} \qquad
\mathrm{H}^2(\widetilde P_0,\widetilde P_1)=D.
\]
Now, we will upper bound~\eqref{eq:cpdrr} by deriving two lower bounds on
$\mathrm{I}_{\,\mathrm{ch}}(Q_0,Q_1)$. Let
$\mathrm{BC}=\sum_{y\in\mathcal X_1}\sqrt{Q_0(y)Q_1(y)}$ denote the Bhattacharyya coefficient. On evaluating the Chernoff minimization at $\lambda=1/2$, we get
\begin{align}
\mathrm{I}_{\,\mathrm{ch}}(Q_0,Q_1)=-\log\left(
\min_{\lambda\in(0,1)}\sum_y Q_0(y)^\lambda Q_1(y)^{1-\lambda}
\right)\geq -\log \mathrm{BC}=-\log\left(1-\mathrm{H}^2(Q_0,Q_1)\right).\label{eq:ch_hellinger}
\end{align}

\noindent\emph{(i). TV distance-based bound.} On applying Cauchy--Schwarz inequality below, we can write
\begin{align*}
\mathrm{d}_{\mathrm{TV}}(Q_0,Q_1)^2
&=\frac{1}{4}\left(\sum_y \left|\,\sqrt{Q_0(y)}-\sqrt{Q_1(y)}\,\right|
\left(\sqrt{Q_0(y)}+\sqrt{Q_1(y)}\right) \right)^2\\
&\leq
\frac{1}{4} \left(\sum_y \left(\sqrt{Q_0(y)}-\sqrt{Q_1(y)}\right)^2\right)
\left(\sum_y \left(\sqrt{Q_0(y)}+\sqrt{Q_1(y)}\right)^2\right)\\
&=(1-\mathrm{BC})(1+\mathrm{BC})=1-\mathrm{BC}^2.
\end{align*}
Thus, we recover Sason's result from~\cite{sason},
\begin{align*}
\mathrm{I}_{\,\mathrm{ch}}(Q_0,Q_1)\geq-\frac{1}{2}\log\left(1-\mathrm{d}_{\mathrm{TV}}(Q_0,Q_1)^2\right).
\end{align*}
Moreover, the output PMFs formula gives us
\[
Q_0(y)-Q_1(y)=\frac{t}{r+t}\left(\widetilde P_0(y)-\widetilde P_1(y)\right).
\]
for all $y\in \mathcal{X}_1$. Therefore, we have $
\mathrm{d}_{\mathrm{TV}}(Q_0,Q_1)=c_{\varepsilon,r}\delta$, where $c_{\varepsilon,r}=\frac{t}{r+t}$. Hence, we obtain our first lower bound,
\begin{align}
\mathrm{I}_{\,\mathrm{ch}}(Q_0,Q_1)
\geq -\frac{1}{2}\log\left(1-c_{\varepsilon,r}^2\delta^2\right). \label{eq:rrlb1}
\end{align}

\noindent\emph{(ii). Hellinger distance-based bound.}
For every $y\in\mathcal X_1$, let us define
\[
a_y=\left|\sqrt{1+t\widetilde P_0(y)}-\sqrt{1+t\widetilde P_1(y)}
\right|, \qquad \text{and} \qquad z^2=\sum_y a_y^2.
\]
such that $z\geq 0$. Then, we have that
$\mathrm{H}^2(Q_0,Q_1) = \dfrac{z^2}{2(r+t)}$. Now, we will lower bound $z$ in terms of the squared-Hellinger distance between $P_0$ and $P_1$. Consider
$g(u)=\sqrt{1+u^2}$ on $[0,\infty)$.
Since $g'$ is increasing, for $v,w\geq0$, we have
\[
g(v+w)-g(v)= \int_0^w g'(v+s)\,ds
\geq \int_0^w g'(s)\,ds = g(w)-1.
\]
On applying the above inequality with
\[
v=\sqrt t\min\left( \sqrt{\widetilde P_0(y)},\sqrt{\widetilde P_1(y)} \right),
\qquad \text{and} \qquad w=\sqrt t\left|\sqrt{\widetilde P_0(y)}-\sqrt{\widetilde P_1(y)}\,\right|
\]
gives us
\[
a_y\geq\sqrt{1+t\left(\sqrt{\widetilde P_0(y)}-\sqrt{\widetilde P_1(y)}\right)^2}-1.
\]
On further rearranging and then squaring, we have
\[
t\left(\sqrt{\widetilde P_0(y)}-\sqrt{\widetilde P_1(y)}\right)^2 \leq a_y^2+2a_y.
\]
Now, on summing over all $y \in \mathcal{X}_1$ and applying Cauchy--Schwarz inequality yields
\[
2tD\leq z^2+2\sum_y a_y \leq z^2+2\sqrt r\,z.
\]
Since $z\geq0$, this implies that $z\geq\sqrt{r+2tD}-\sqrt{r}$. Consequently, we have that
\[
\mathrm{H}^2(Q_0,Q_1)\geq\frac{\left(\sqrt{r+2tD}-\sqrt r\right)^2}{2(r+t)}=L_{\varepsilon,r}(D).
\]
Together with the Chernoff--Hellinger inequality above in~\eqref{eq:ch_hellinger}, we get
\begin{align}
\mathrm{I}_{\,\mathrm{ch}}(Q_0,Q_1)\geq-\log\left(1-L_{\varepsilon,r}(D)\right).\label{eq:rrlb2}
\end{align}
On combining both lower bounds in~\eqref{eq:rrlb1} and~\eqref{eq:rrlb2}, we get
\begin{align}
\mathrm{I}_{\,\mathrm{ch}}(Q_0,Q_1)&\geq\max\left(-\frac12\log\left(1-c_{\varepsilon,r}^2\delta^2\right),
-\log\left(1-L_{\varepsilon,r}(D)\right)
\right)\\ 
&=-\log\rho_{\mathrm{ARR}},
\end{align}
where $\rho_{\mathrm{ARR}}=\min\left(\sqrt{1-c_{\varepsilon,r}^2\delta^2},\,1-L_{\varepsilon,r}(D)
\right)$. For finite $\varepsilon$, we have
$0<\rho_{\mathrm{ARR}}\leq 1$. Substituting the lower bound
$\mathrm{I}_{\,\mathrm{ch}}(Q_0,Q_1)\ge-\log\rho_{\mathrm{ARR}}$
into the preceding GLRT guarantee in~\eqref{eq:cpdrr} therefore gives
\[
\beta_{\mathrm{ARR}}\leq 2\rho_{\mathrm{ARR}}^{\alpha+1}
\left(1-\frac{\rho_{\mathrm{ARR}}^{\alpha+1}}{2}\right).
\]
The corollary follows directly from the TV distance-based bound in~\eqref{eq:rrlb1}, using the fact that $-\log(1-u)\geq u$ for $u\in[0,1)$ and substituting the resulting lower bound on the Chernoff information in~\eqref{eq:cpdrr}. This completes our proof.
\end{proof}
\subsection{Choice of $\tau$ for Algorithm~\ref{alg:offline_bmcpd} $(\mathrm{BM-CPD})$ }\label{app:tau_comp}
We compare the empirical localization error obtained with the Chernoff-optimal threshold $\tau^\star$ from step~$3$ of Algorithm~\ref{alg:offline_bmcpd} against other threshold choices. We run $10000$ Monte Carlo repetitions with $n=2000$ and $k^\star=1000$, drawing observations from $P_0$ and $P_1$ and then applying Algorithm~\ref{alg:offline_bmcpd} for several choices of $\tau$, with $\tau^\star$ computed for each $\varepsilon$. Figures~\ref{fig:tau1} and~\ref{fig:tau2} show the resulting empirical $(\alpha,\beta_{\mathrm{BM}})$ curves. In the low-privacy examples, the choice of the optimal threshold matters: at $\varepsilon=6$ in Figure~\ref{fig:tau2}, $\tau=1$ attains about half the Chernoff information of $\tau^\star$. In the high-privacy regime, the two choices give comparable empirical localization errors.
\begin{figure}[H]
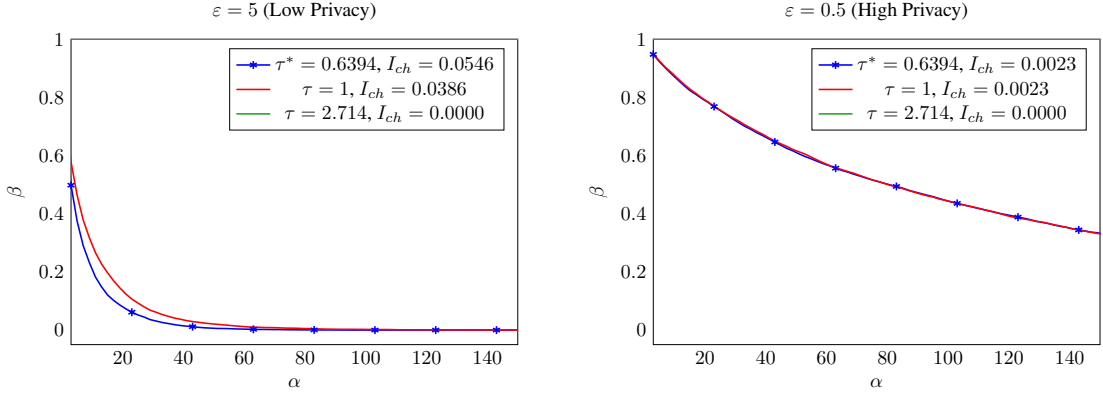

\begin{center}
        \input{figures/tau_1_he}  
        \hspace*{6mm}
        \input{figures/tau_1_le}  
\caption{Comparison of the empirical error probability $\beta_{\mathrm{BM}}$ for Algorithm~\ref{alg:offline_bmcpd} under different choices of quantizers $Z_\tau$, parameterized by $\tau > 0$, plotted as functions of $\alpha$ for different values of the privacy parameter $\varepsilon > 0$. We consider the following pre- and post-change distributions: $P_0 =[0.66266061, 0.10739055, 0.22994884]$ and  $P_1 =[0.38665800, 0.38304133, 0.23030066]$. }
        \label{fig:tau1}
        \vspace{4mm}
  \end{center}
\end{figure}
\begin{figure}[H]
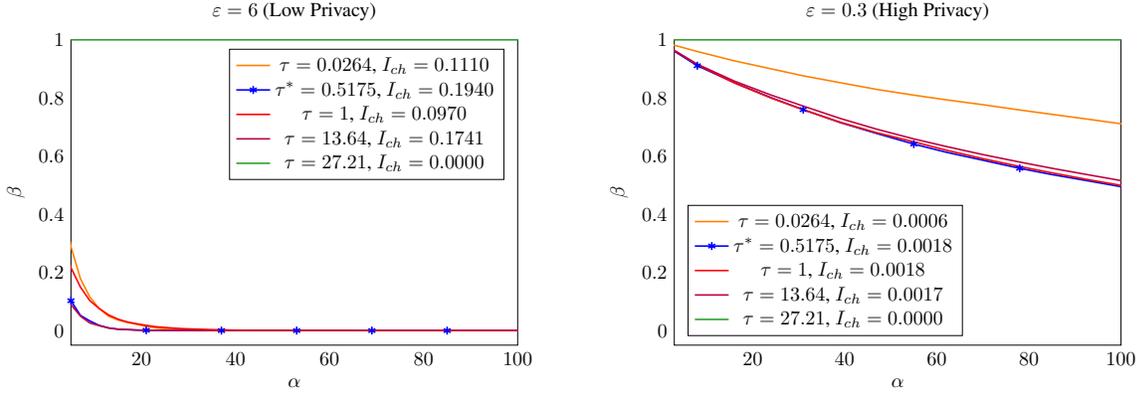

\begin{center}
        \input{figures/tau_2_he}  
          \hspace*{6mm}
        \input{figures/tau_2_le} 
               \caption{Comparison of the empirical error probability $\beta_{\mathrm{BM}}$ for Algorithm~\ref{alg:offline_bmcpd} under different choices of quantizers $Z_\tau$, parameterized by $\tau > 0$, plotted as functions of $\alpha$ for different values of the privacy parameter $\varepsilon > 0$. We consider the following pre- and post-change distributions:
$P_0 = [0.00201312, 0.41095372, 0.17140827, 0.40791425, 0.00771065]$ and
$P_1 = [0.23852892, 0.01568089, 0.16016461, 0.41178628, 0.17383930]$.}
        \label{fig:tau2}
\end{center}
\end{figure}
\subsection{Proof of Theorem~\ref{thm:bmcpd} (Finite sample guarantees for $\mathrm{BM-CPD}$ algorithm)}\label{app:bmcpd}
\begin{proof}        
Our proof relies on the fact that the Offline $\mathrm{BM-CPD}$ estimator in Algorithm~\ref{alg:offline_bmcpd} first partitions the alphabet $\mathcal{X}$ into $(\mathcal{S}_{\tau^{*}},\mathcal{S}_{\tau^{*}}^{c})$ via $Z_{\tau^{*}}$. Then, it privatizes the binary dataset (quantized) using the binary randomized response $W_{Y|V}$ (which satisfies $\varepsilon$-LDP) and then applies the GLRT (Algorithm~\ref{alg:offline_cpd}) on the privatized dataset to estimate the change-point.

Let $(\mathcal{D},P_0,P_1)$ denote the dataset, pre-change distribution, and the post-change distribution before privatization. Let the privacy parameter be $\varepsilon>0$. After privatization via chernoff-optimal binary mechanism ($\varepsilon$-$\mathrm{BM}$), let the tuple be denoted by $(\widetilde{\mathcal{D}},Q^{\tau^*}_0,Q^{\tau^*}_1)$. Then, from our results in Theorem~\ref{thm:npcpd}, we have
\begin{align}
       \beta_{\mathrm{BM}} \leq \,\,&2\, e^{-(\alpha+1) \mathrm{I}_{\,\mathrm{ch}}(Q^{\tau^*}_0,Q^{\tau^*}_1)}\left(1-\frac{e^{-(\alpha+1) \mathrm{I}_{\,\mathrm{ch}}(Q^{\tau^*}_0,Q^{\tau^*}_1)}}{2}\right). \label{eq:cpdbm}
\end{align}
Note that the R.H.S in~\eqref{eq:cpdbm} is non-increasing in $\mathrm{I}_{\,\mathrm{ch}}$. By the definition of $\tau^\star$ in Algorithm~\ref{alg:offline_bmcpd}, we have $\mathrm{I}_{\,\mathrm{ch}}(Q_0^{\tau^\star},Q_1^{\tau^\star})
\geq \mathrm{I}_{\,\mathrm{ch}}(Q_0^1,Q_1^1)$. Thus, we consider the binary mechanism with threshold $\tau=1$. Our goal is to lower bound $\mathrm{I}_{\,\mathrm{ch}}(Q^{1}_0,Q^{1}_1)$ in terms of quantities depending on $(P_0,P_1)$, to obtain a meaningful private finite-sample guarantee.

Let $R_\nu$ denote the PMF of the quantized observation
$Z_1(X) \in \{0,1\}$ when $X\sim P_\nu$, such that
$R_\nu(0)=P_\nu(S_1)$, and $R_\nu(1)=1-P_\nu(S_1)$, for $\nu\in\{0,1\}$.
Since $S_1$ contains precisely the symbols whose likelihood
ratio is at least one, we have that
$\mathrm{d}_{\mathrm{TV}}(R_0,R_1) = P_0(S_1)-P_1(S_1)=
\mathrm{d}_{\mathrm{TV}}(P_0,P_1)=\delta$. Moreover, by the definition of $D_1$, we have that $\mathrm H^2(R_0,R_1)=D_1$.
As in the proof of Theorem~\ref{thm:rrcpd}, let $t=e^\varepsilon-1$, binary randomized response gives output distributions
\[
Q_\nu^1(y)=\frac{1+tR_\nu(y)}{2+t}, \qquad \text{for } \nu, y\in\{0,1\}.
\]
These output PMFs have the same form as those considered in the proof of Theorem~\ref{thm:rrcpd}, with alphabet size $r=2$ and input PMFs $(R_0,R_1)$. Consequently, we can directly use the bounds derived in~\eqref{eq:rrlb1} and~\eqref{eq:rrlb2}, respectively, as follows,
\begin{align}
\mathrm{I}_{\,\mathrm{ch}}(Q_0^1,Q_1^1)&\geq -\frac12\log\left(1-c_{\varepsilon,2}^2\delta^2\right),\label{eq:bmlb1}\\
\mathrm{I}_{\,\mathrm{ch}}(Q_0^1,Q_1^1)&\geq -\log\left(1-L_{\varepsilon,2}(D_1)\right)\label{eq:bmlb2}.
\end{align}
On combining these inequalities above with the optimality of $\tau^\star$ in $\mathrm{BM-CPD}$, we obtain
\begin{align}
\mathrm{I}_{\,\mathrm{ch}}(Q_0^{\tau^\star},Q_1^{\tau^\star})
&\geq \max\left(
-\frac12\log\left(1-c_{\varepsilon,2}^2\delta^2\right), -\log\left(1-L_{\varepsilon,2}(D_1)\right)\right)\\
&=-\log\rho_{\mathrm{BM}}.\label{eq:bmnew}
\end{align}
where $\rho_{\mathrm{BM}}=\min\left(\sqrt{1-c_{\varepsilon,2}^{\,2}\delta^2},\,1-L_{\varepsilon,2}(D_1)
\right)$. For finite $\varepsilon$, we have
$0<\rho_{\mathrm{BM}}\le1$. Finally, on substituting the preceding lower bound in ~\eqref{eq:bmnew} into the upper bound~\eqref{eq:cpdbm} yields
\[
\beta_{\mathrm{BM}}
\le
2\rho_{\mathrm{BM}}^{\alpha+1}
\left(1-\frac12\rho_{\mathrm{BM}}^{\alpha+1}\right).
\]
The corollary follows directly from the TV distance-based bound in~\eqref{eq:bmlb1}, using the fact that $-\log(1-u)\geq u$ for $u\in[0,1)$ and substituting the resulting lower bound on the Chernoff information in~\eqref{eq:cpdbm}. This completes our proof.
\end{proof}
\subsection{Privacy-regime ordering of the finite-sample guarantees}\label{app:comparison}
Note that at $r=2$, the two mechanisms coincide up to relabeling of their outputs and they yield identical bounds. Thus, lets assume $r>2$.
Using the similar notations $\delta=\mathrm{d}_{\mathrm{TV}}(P_0,P_1)$,
$D=\mathrm{H}^2(P_0,P_1)$, and
and
$D_1=\mathrm{H}^2(R_0,R_1)$, where $R_0,R_1$ are the binary
PMFs defined in the proof of Theorem~\ref{thm:bmcpd}. Here $S_1$ is the set appearing in Theorem~\ref{thm:bmcpd}. Note that both our bounds in Thm.~\ref{thm:rrcpd} and Thm.~\ref{thm:bmcpd} have the following form
\[
B_j(\alpha,\varepsilon)=g_\alpha(\rho_j(\varepsilon)),
\qquad \text{where} \qquad
g_\alpha(x)=1-(1-x^{\alpha+1})^2.
\]
The function $g_\alpha$ is strictly increasing on $[0,1]$. Thus, it just suffices to compare $\rho_{\mathrm{ARR}}$ and
$\rho_{\mathrm{BM}}$. Every such comparison then holds
simultaneously for all $\alpha\in[n]$.

Let $t=e^\varepsilon-1>0$ and
$T_m=\sqrt{1-c_{\varepsilon,m}^2\delta^2}$, where $m \geq 2$ is an integer. Further,  let $u=\sqrt{m+2td}$, on rationalization gives us
\[
L_{\varepsilon,m}(d)
=\frac{td}{t+m}\frac{u-\sqrt m}{u+\sqrt m}
\le\frac{t^2d^2}{2m(t+m)},
\]
because $(u-\sqrt m)/(u+\sqrt m)=2td/(u+\sqrt m)^2 \leq td/2m$.
The same identity aslo gives us $L_{\varepsilon,m}(d)\to d$ as $\varepsilon\to\infty$.

\emph{High privacy case:} Since $r>2$ and $\delta>0$, we have $T_2<T_r$.
Moreover, $D\le\delta$ and
$1-T_2\ge t^2\delta^2/(2(t+2)^2)$. Therefore, whenever $(t+2)^2<r(t+r)$, we have
\[
L_{\varepsilon,r}(D)
\le\frac{t^2\delta^2}{2r(t+r)}
<\frac{t^2\delta^2}{2(t+2)^2}
\le1-T_2.
\]
Consequently, we have
\begin{align*}
\rho_{\mathrm{ARR}}
=\min(T_r,1-L_{\varepsilon,r}(D))>T_2 \geq \rho_{\mathrm{BM}}.
\end{align*}
Thus, the sufficient condition is exactly
$(e^\varepsilon+1)^2<r(e^\varepsilon+r-1)$.
It holds for every $r\ge3$ and $0<\varepsilon\le1$:
for $t\geq0$, the difference $r(t+r)-(t+2)^2$ increases with $r$ and, at $r=3$, it equals $5-t-t^2$, which is positive for $t\leq e-1$.

\emph{Low privacy case:} Let $\mathrm{BC}(P,Q)=\sum_x\sqrt{P(x)Q(x)}$ denote the Bhattacharya coefficient. As $\varepsilon \to \infty$, we know that $c_{\varepsilon,m}\to 1$ and $L_{\varepsilon,m}(d)\to d$. Consequently, from the definitions of $\rho$,
\begin{align*}
\rho_{\mathrm{ARR}}&\longrightarrow\mathrm{BC}(P_0,P_1),\\
\rho_{\mathrm{BM}}&\longrightarrow
                 \mathrm{BC}(R_0,R_1).
\end{align*}
Note that both pairs have the same TV distance $\delta$. 
Cauchy-Schwarz implies
$\mathrm{BC}(P,Q)\le\sqrt{1-\mathrm d_{\mathrm{TV}}(P,Q)^2}$,
so the Hellinger distance based term attains the minimum in each limit. For each $S\in\{S_1,S_1^c\}$, Cauchy-Schwarz inequality also gives us
\[
\sum_{x\in S}\sqrt{P_0(x)P_1(x)}
\le\sqrt{P_0(S)P_1(S)},
\]
with equality iff $P_0/P_1$ is constant on $S$. As $r>2$, equality cannot hold on both sets. Hence, we have
\[
\mathrm{BC}(P_0,P_1)
<\mathrm{BC}(R_0,R_1).
\]
It implies that
$\rho_{\mathrm{ARR}}(\varepsilon)<\rho_{\mathrm{BM}}(\varepsilon)$ for every $\varepsilon>\bar\varepsilon(P_0,P_1)$, for some finite $\bar\varepsilon(P_0,P_1)$. Using the monotonicity of $g_\alpha$, this proves our claim.
\subsection{Privacy-cost analysis via error exponents}~\label{app:cop}
Let's fix a pair of pre- and post-change (full support) distributions $(P_0,P_1)$. Let $E_{\mathcal K_n}(\alpha)$ and $E^\varepsilon_{\mathcal K_n}(\alpha)$ denote the non-private and private minimax risks, respectively, with the set $\mathcal K_n=\{2,\ldots,n\}$ and the locally private, non-interactive $\varepsilon$-LDP mechanism class studied within Theorem~\ref{thm:private-minimax-lb}.

Fix a common sequence of positive integers $\alpha_n\to\infty$ such that $\alpha_n= o(n)$, independently of $\varepsilon>0$. For each fixed
privacy parameter $\varepsilon$, let us define
\begin{equation}\label{eq:privacy-cost-exponents}
 E_{{np}}\triangleq
 \liminf_{n\to\infty}\frac{-\log E_{\mathcal K_n}(\alpha_n)}{\alpha_n},
 \qquad
 E_p\triangleq
 \liminf_{n\to\infty}\frac{-\log E^\varepsilon_{\mathcal K_n}(\alpha_n)}{\alpha_n}
 \end{equation}
and 
$$ C_{\exp}(\varepsilon)\triangleq\frac{E_{{np}}}{E_p}.$$
Now, We will show that for every such sequence ${\alpha_n}$ and every $\varepsilon>0$, we have
\begin{equation}\label{eq:privacy-cost-bound}
 \max\left(1,\frac{1}{2\tanh^2(\varepsilon/2)}\right)
 \leq C_{\exp}(\varepsilon)
 \leq \frac{4\mathrm{I}_{\,\mathrm{ch}}(P_0,P_1)}{\tanh^2(\varepsilon/2)\,\mathrm{d^2_{\mathrm{TV}}}(P_0,P_1)}
\end{equation}
\begin{proof}
Since $\alpha_n=o(n)$, we have $M_{\alpha_n}\geq 2$
for all sufficiently large $n$, we have that
\[
0\leq \frac{-\log(1-1/M_{\alpha_n})}{\alpha_n} \leq \frac{\log 2}{\alpha_n}\to 0.
\]
Thus, the pre-factor in the minimax lower bound does not
affect the asymptotic exponent. Also, for every fixed $\varepsilon>0$, the square-root corrections in the converse bounds also vanish after normalization, since $\sqrt{2\alpha_n+1}/\alpha_n\to0$. The non-private finite-sample guarantee of the GLRT (Algorithm~\ref{alg:offline_cpd}) in Thm.~\ref{thm:npcpd} and minimax lower bound in Thm.~\ref{thm:minimax_lb} consequently give
\begin{equation}\label{eq:privacy-cost-nonprivate}
\mathrm{I}_{\,\mathrm{ch}}(P_0,P_1)\leq E_{{np}}\leq 2\,\mathrm{I}_{\,\mathrm{ch}}(P_0,P_1).
\end{equation}
as their leading exponential terms are
$e^{-(\alpha_n+1)\mathrm{I}_{\,\mathrm{ch}}(P_0,P_1)}$ and $e^{-(2\alpha_n+1)\mathrm{I}_{\,\mathrm{ch}}(P_0,P_1)}$, respectively.

In the locally private setting, we consider the finite-sample guarantees for the $\mathrm{BM-CPD}$ in Thm.~\ref{thm:bmcpd}. The bound based on the TV distance gives
$E^\varepsilon_{\mathcal K_n}(\alpha_n)
\leq 
2\exp(-(\alpha_n+1)\tanh^2(\varepsilon/2)\mathrm{d^2_{\mathrm{TV}}}(P_0,P_1)/2)$. It upper bounds the private minimax risk under $\varepsilon$-LDP. The private minimax lower bound in Thm.~\ref{thm:private-minimax-lb} similarly yields $E_p\leq 2\overline I_\varepsilon$. Furthermore,
\[
 \overline I_\varepsilon
 \le-\log\bigl(1-\tanh^2(\varepsilon/2)(1-e^{-\mathrm{I}_{\,\mathrm{ch}}(P_0,P_1)})\bigr)
 \leq \tanh^2(\varepsilon/2)\, \mathrm{I}_{\,\mathrm{ch}}(P_0,P_1).
\]
The last inequality follows from convexity of
$g(t)=-\log(1-t(1-e^{-\mathrm{I}_{\,\mathrm{ch}}(P_0,P_1)}))$ on $[0,1]$, since $g(0)=0$ and $g(1)=\mathrm{I}_{\,\mathrm{ch}}(P_0,P_1)$.
Therefore,
\begin{equation}\label{eq:privacy-cost-private}
 \frac{\tanh^2(\varepsilon/2)\delta^2}{2}
 \leq E_p\leq 2 \tanh^2(\varepsilon/2) \mathrm{I}_{\,\mathrm{ch}}(P_0,P_1).
\end{equation}
In particular, both localization error exponents are finite and strictly positive. Moreover, a non-private procedure can simulate any private procedure, so $E_{\mathcal K_n}(\alpha)\leq  E^\varepsilon_{\mathcal K_n}(\alpha)$ and hence $E_p\leq E_{\mathrm{np}}$. Combining this ordering with~\eqref{eq:privacy-cost-nonprivate} and~\eqref{eq:privacy-cost-private}
proves our result in~\eqref{eq:privacy-cost-bound}.
Finally, $\tanh(\varepsilon/2)\sim\varepsilon/2$ implies
$C_{\exp}(\varepsilon)=\Theta_{(P_0,P_1)}(\varepsilon^{-2})$
as $\varepsilon\to 0$.
\end{proof}

\subsection{Private finite-sample guarantees in terms of $C$ and $S$}
\label{app:private-cs}
In the spirit of Lemma~\ref{cor:npcpd}, we derive alternate finite-sample guarantees for {$\mathrm{ARR-CPD}$} and $\mathrm{BM-CPD}$ in terms of $C$ (minimum KL divergence) and $S$ (sensitivity of GLRT statistic), using Theorem~\ref{thm:jrdccbd-value} and Lemma~\ref{cor:ldpp} to control the contraction of the log-likelihood range under privatization via $\mathrm{ARR}$ and $\mathrm{BM}$. However, it is important to note here that our finite-sample guarantees presented in Theorem~\ref{thm:rrcpd} and Theorem~\ref{thm:bmcpd} are provably better, than the results presented below.
\begin{lemma}[$C,S$ guarantees after privatization]
\label{lem:private-cs}
Let $k^\star\in\mathcal K$ be the true change-point of the
dataset $\mathcal D$, with independent data
points drawn from  $(P_0,P_1)$ on $\mathcal X$. Define
\[
C=\min\{D_{\mathrm{KL}}(P_0\|P_1),
        D_{\mathrm{KL}}(P_1\|P_0)\},
\qquad \text{and}\qquad
S=\mathrm{D}^{\mathrm{J}}_\infty(P_0,P_1).
\]
Let $r$ be the number of groups in $\varepsilon$-$\mathrm{ARR}$ and set
denote $c_{\varepsilon,r}
=\frac{e^\varepsilon-1}{e^\varepsilon+r-1}$ and $
b_\varepsilon=\min\{2\varepsilon,c_{\varepsilon,2} S\}$. Further, define
\[
\underline C_{\mathrm{ARR}}
=\frac{8c_{\varepsilon,r}^{\,2}C^2}{S^2},
\qquad \text{and}\qquad
\underline C_{\mathrm{BM}}
=\frac{2c_{\varepsilon,2}^{\,2}C^2}{S^2}.
\]
Then, for every integer $\alpha \in [n]$, \\

\noindent (i). The {$\mathrm{ARR-CPD}$} estimator in Algorithm~\ref{alg:offline_rrcpd} satisfies
$\varepsilon$-LDP and is
$(\alpha,\beta_{\mathrm{ARR}})$-accurate, with
\begin{equation}
\beta_{\mathrm{ARR}}
\leq
2\sum_{i=1}^{i^*}
\exp\!\left(
-\frac{2^{i-1}\alpha\underline C_{\mathrm{ARR}}^{\,2}}
      {b_\varepsilon^{\,2}}
\right),
\label{eq:private-cs-arr}
\end{equation}
(ii). The $\mathrm{BM-CPD}$ estimator in Algorithm~\ref{alg:offline_bmcpd} satisfies
$\varepsilon$-LDP and is
$(\alpha,\beta_{\mathrm{BM}})$-accurate, with
\begin{equation}
\beta_{\mathrm{BM}}
\leq
2\sum_{i=1}^{i^*}
\exp\!\left(
-\frac{2^{i-1}\alpha\underline C_{\mathrm{BM}}^{\,2}}
      {b_\varepsilon^{\,2}}
\right).\label{eq:private-cs-bm}
\end{equation}
where
$i^*
=\left\lceil\log_2\left(\frac{n-1}{\alpha}\right)\right\rceil$.

\end{lemma}
\begin{proof}
For an output pair $(Q_0,Q_1)$, post privatization through an $\varepsilon$-LDP mechanism, let
\[
C_Q=\min\{D_{\mathrm{KL}}(Q_0\|Q_1),
          D_{\mathrm{KL}}(Q_1\|Q_0)\},
\qquad \text{and} \qquad
S_Q=D^J_\infty(Q_0,Q_1).
\]
Privatization under aggregated randomized response and binary mechanism preserves independence and the change-point location,
and both algorithms apply the non-private GLRT to the privatized
observations. Thus, our non-private Lemma~\ref{cor:npcpd} applies directly with $(C,S)$ replaced by $(C_Q,S_Q)$.

Similar to the rest of the paper, let $\delta=\mathrm{d}_{\mathrm{TV}}(P_0,P_1)$. The positive and negative
parts of the difference $P_0-P_1$ each have mass $\delta$, so we have that
\begin{align*}
2C
&\leq D_{\mathrm{KL}}(P_0\|P_1) +D_{\mathrm{KL}}(P_1\|P_0)\\
&=\sum_x\left(P_0(x)-P_1(x\right))\log\frac{P_0(x)}{P_1(x)} \leq\, S\delta.
\end{align*}
In particular, we obtain that $\delta\geq 2C/S$. Let $W^{\varepsilon}$ be $m$-ary randomized response $(m\geq 2)$, whose diagonal and off-diagonal entries are
\[
v_m=\frac{e^\varepsilon}{e^\varepsilon+m-1},
\qquad \text{and} \qquad
u_m=\frac{1}{e^\varepsilon+m-1}.
\]
Then, using our result proved in Theorem~\ref{thm:jrdccbd-value} on the SDPI coefficient for Jeffreys-Rényi divergence gives
\begin{equation}
\eta^J_\infty(W_{\varepsilon,m})
=\frac{v_m-u_m}{v_m+u_m}
=\frac{e^\varepsilon-1}{e^\varepsilon+1}
=c_{\varepsilon,2}.\label{eq:private-cs-sdpi}
\end{equation}
Thus, this coefficient is independent of the alphabet size $m$.
Moreover, $\varepsilon$-LDP implies
$e^{-\varepsilon}\leq Q_0(y)/Q_1(y)\leq e^\varepsilon$,
and hence $S_Q\leq 2\varepsilon$ (see Lemma~\ref{cor:ldpp}).

\emph{Analysis for $\mathrm{ARR-CPD}$}.
Let $\widetilde P_0,\widetilde P_1$ be the distributions obtained by
aggregating equal likelihood ratios, such that
$Q_\nu=\widetilde P_\nu W^{\varepsilon}$ for $\nu\in\{0,1\}$.
This aggregation preserves both $\mathrm D^\mathrm J_\infty$ and $\mathrm{d}_{\mathrm{TV}}$. From the definition of the SDPI coefficient and~\eqref{eq:private-cs-sdpi} we have,
\begin{align*}
S_Q&\leq {\eta}^\mathrm J_\infty(W^{\varepsilon})\mathrm D^\mathrm J_\infty(\widetilde P_0,\widetilde P_1)=S\,c_{\varepsilon,2}.
\end{align*}
Combining this with the LDP range bound gives us the result of Lemma~\ref{cor:ldpp}
\[
S_Q\le\min\{2\varepsilon,c_{\varepsilon,2} S\}=b_\varepsilon.
\]
Since, the randomized response mechanism multiplies the total-variation distance by $c_{\varepsilon,r}$, on applying
Pinsker's inequality we obtain that
\begin{align*}
C_Q
&\ge 2\mathrm{d}^{2}_{\mathrm{TV}}(Q_0,Q_1)
 =2c_{\varepsilon,r}^{\,2}\delta^2\\
&\ge \frac{8c_{\varepsilon,r}^{\,2}C^2}{S^2}
 =\underline C_{\mathrm{ARR}}.
\end{align*}
Substituting these bounds into the private version of Lemma~\ref{cor:npcpd} proves~\eqref{eq:private-cs-arr}.

\emph{Analysis for $\mathrm{BM-CPD}$.}
Let $(Q_0,Q_1)=(Q_0^{\tau^*},Q_1^{\tau^*})$ be the actual
induced output pair of distributions from the BM. Let us denote by,
\[
\widetilde P_\nu^\tau=\bigl(P_\nu(\mathcal S_\tau),P_\nu(\mathcal S_\tau^c)\bigr),
\qquad \text{and}\qquad
Q_\nu^\tau=\widetilde P_\nu^\tau W^{\varepsilon}.
\]
By data-processing under quantization in BM and~\eqref{eq:private-cs-sdpi}, we have that
\begin{align*}
S_Q&\leq \eta^\mathrm J_\infty(W^{\varepsilon})\mathrm{D}^{\mathrm{J}}_\infty(\widetilde P_0^{\tau^*}, \widetilde P_1^{\tau^*})\\
&\leq c_{\varepsilon,2} S.
\end{align*}
The LDP range bound again (also Lemma~\ref{cor:ldpp}) yields
\[
S_Q\le\min\{2\varepsilon,c_{\varepsilon,2} S\}=b_\varepsilon.
\]
Let $(Q_0^1,Q_1^1)$ denote the output pair for the set
$\mathcal S_1=\{x:P_0(x)\geq P_1(x)\}$. It satisfies that
$\mathrm d_{\mathrm{TV}}(Q_0^1,Q_1^1)=c_{\varepsilon,2}\,\delta$.
Since, chernoff information is at most either KL divergence. Using this fact, the defining optimality of $\tau^*$, and
$\mathrm{I}_{\,\mathrm{ch}}(R_0,R_1)
\geq \mathrm{d}^2_{\mathrm{TV}}(R_0,R_1)/2$, we obtain
\begin{align*}
C_Q
&\geq \mathrm I_{\mathrm{ch}}(Q_0^{\tau^*},Q_1^{\tau^*})\\
&\geq \mathrm I_{\mathrm{ch}}(Q_0^1,Q_1^1)\\
&\geq \frac12 c_{\varepsilon,2}^2\delta^2
 \geq \frac{2c_{\varepsilon,2}^2C^2}{S^2}
 =\underline C_{\mathrm{BM}}.
\end{align*}
Substituting these bounds into the private version of Lemma~\ref{cor:npcpd}, we finally obtain~\eqref{eq:private-cs-bm}.
\end{proof}

\subsection{Comparison of Theoretical Private Bounds with Empirical Performance}\label{app:priv_comp}
We further provide an empirical comparison by plotting the bounds for several synthetic datasets in Figures~\ref{fig:priv_bern_small}-~\ref{fig:priv_tpoiss_big}.
We fix the block size $n=2000$, the true change-point $k^{\star}=1000$ and compare the upper bounds for different choices of pre- and post-change distributions $(P_0,P_1)$. The empirical curve is computed via Monte Carlo simulations by repeatedly sampling data from $(P_0,P_1)$, estimating the change-point using Algorithms ~\ref{alg:offline_rrcpd} and~\ref{alg:offline_bmcpd}, and repeating this procedure $10000$ times. 
\begin{figure}[H]
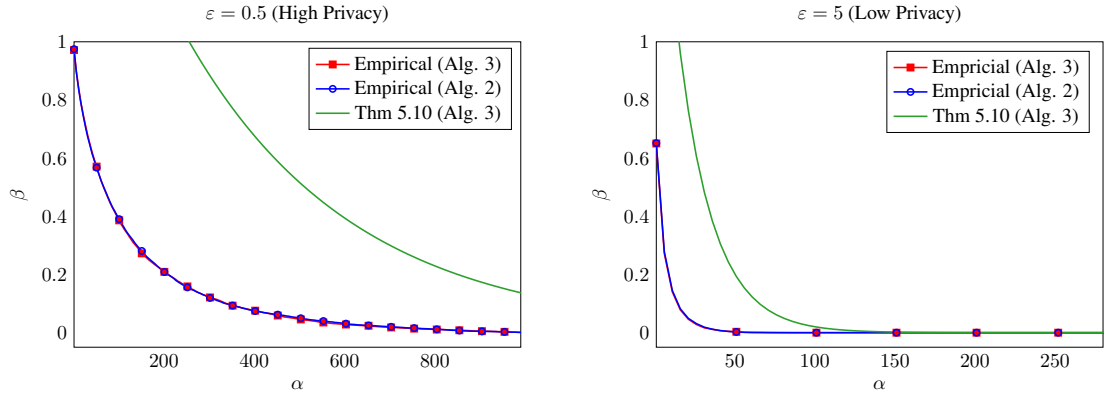

\begin{center}
        \input{figures/priv_bern_1_small}  
                \hspace*{6mm}
        \input{figures/priv_bern_2_small}  
                \caption{Comparison of the empirical error probability $\beta$ and theoretical {upper bounds} on $\beta$ in the private setting plotted as functions of $\alpha$. We consider binary (Bernoulli) distributions; $P_0 \sim \mathrm{Ber}(0.1)$, $P_1 \sim \mathrm{Ber}(0.4)$. We use Algorithms~\ref{alg:offline_rrcpd} and~\ref{alg:offline_bmcpd} with ARR and Binary mechanisms. Note that Algorithms~\ref{alg:offline_rrcpd} and~\ref{alg:offline_bmcpd} reduce to the same algorithm when the pre- and post-change distributions are binary. \textbf{Left:} $\varepsilon = 0.5$. \textbf{Right:} $\varepsilon = 5$. }
        \label{fig:priv_bern_small}
\end{center}
\end{figure}
\vspace*{-4mm}
\begin{figure}[H]
\begin{center} 
        \input{figures/priv_bern_1_big}  
                \hspace*{6mm}
        \input{figures/priv_bern_2_big}  
                \caption{Comparison of the empirical error probability $\beta$ and theoretical {upper bounds} on $\beta$ in the private setting plotted as functions of $\alpha$. We consider binary (Bernoulli) distributions; $P_0 \sim \mathrm{Ber}(0.1)$, $P_1 \sim \mathrm{Ber}(0.95)$. We use Algorithms~\ref{alg:offline_rrcpd} and~\ref{alg:offline_bmcpd} with ARR and Binary mechanisms. Note that Algorithms~\ref{alg:offline_rrcpd} and~\ref{alg:offline_bmcpd} reduce to the same algorithm when the pre- and post-change distributions are binary. \textbf{Left:} $\varepsilon = 0.5$. \textbf{Right:} $\varepsilon = 5$.}
        \label{fig:priv_bern_big}
\end{center}
\end{figure}
\vspace*{-4mm}
\begin{figure}[H]
\begin{center} 
        \input{figures/priv_tpoiss_1}
                \hspace*{6mm}
        \input{figures/priv_tpoiss_5}
                \caption{Comparison of the empirical error probability $\beta$ and theoretical {upper bounds} on $\beta$ in the private setting plotted as functions of $\alpha$ for Algorithm~\ref{alg:offline_bmcpd} (Offline $\mathrm{BM-}$CPD). We consider truncated Poisson distributions $\mathrm{TPois}(\lambda,m)$ with truncation parameter $m=10$. $P_0 \sim \mathrm{TPois}(1,10)$, $P_1 \sim \mathrm{TPois}(4,10)$. \textbf{Left:} $\varepsilon = 1$. \textbf{Right:} $\varepsilon = 5$.}
        \label{fig:priv_tpoiss_small_bm}
\end{center}
\end{figure}
\begin{figure}[H]
\begin{center} 
        \input{figures/priv_tpoiss_1_rr}
                \hspace*{6mm}
        \input{figures/priv_tpoiss_5_rr}
                \caption{Comparison of the empirical error probability $\beta$ and theoretical {upper bounds} on $\beta$ in the private setting plotted as functions of $\alpha$ for Algorithm~\ref{alg:offline_rrcpd} (Offline $\mathrm{ARR-}$CPD). We consider truncated Poisson distributions $\mathrm{TPois}(\lambda,m)$ with truncation paramater $m=10$. $P_0 \sim \mathrm{TPois}(1,10)$, $P_1 \sim \mathrm{TPois}(4,10)$. \textbf{Left:} $\varepsilon = 1$. \textbf{Right:} $\varepsilon = 5$.}
        \label{fig:priv_tpoiss_small}
\end{center}
\end{figure}
\begin{figure}[H]
\begin{center}
        \input{figures/priv_tpoiss_1-2}
                \hspace*{6mm}
        \input{figures/priv_tpoiss_5-2}
                \caption{Comparison of the empirical error probability $\beta$ and theoretical {upper bounds} on $\beta$ in the private setting plotted as functions of $\alpha$ for Algorithm~\ref{alg:offline_bmcpd} (Offline $\mathrm{BM-}$CPD). We consider truncated Poisson distributions $\mathrm{TPois}(\lambda,m)$ with truncation parameter $m=10$. $P_0 \sim \mathrm{TPois}(1,10)$, $P_1 \sim \mathrm{TPois}(10,10)$.  \textbf{Left:} $\varepsilon = 1$. \textbf{Right:} $\varepsilon = 5$.}
        \label{fig:priv_tpoiss_big_bm}
        \end{center}
\end{figure}
\begin{figure}[H]
\begin{center}
            \input{figures/priv_tpoiss_1-2_rr}
                \hspace*{6mm}
        \input{figures/priv_tpoiss_5-2_rr}
                \caption{Comparison of the empirical error probability $\beta$ and theoretical {upper bounds} on $\beta$ in the private setting plotted as functions of $\alpha$ for Algorithm~\ref{alg:offline_rrcpd} (Offline $\mathrm{ARR-}$CPD). We consider truncated Poisson distributions $\mathrm{TPois}(\lambda,m)$ with truncation parameter $m=10$ i.e., $P_0 \sim \mathrm{TPois}(1,10)$, $P_1 \sim \mathrm{TPois}(10,10)$. \textbf{Left:} $\varepsilon = 1$. \textbf{Right:} $\varepsilon = 5$.}
        \label{fig:priv_tpoiss_big}
\end{center}
\end{figure}
\subsection{Comparison of Empirical and Theoretical Performances of Algorithm~\ref{alg:offline_rrcpd} and Algorithm~\ref{alg:offline_bmcpd}}\label{app:exp3}
\begin{figure}[H]
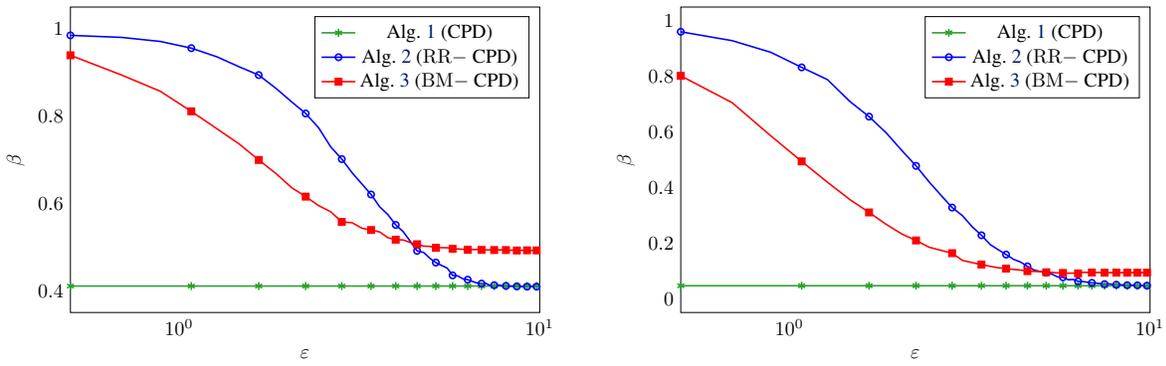

\begin{center}        
        \input{figures/emp_tgeom_1}
        \hspace*{4mm}
        \input{figures/emp_tgeom_2}
           \caption{Comparison of empirical error  probabilities as functions of privacy parameter $\varepsilon$, for truncated geometric distributions $\mathrm{TGeom}(p,m)$ with truncation parameter $m=10$, and for fixed $\alpha =5$. \textbf{Left:} $P_0 \sim \mathrm{TGeom}(0.2,10)$ and $P_1 \sim \mathrm{TGeom}(0.4,10)$. \textbf{Right:} $P_0 \sim \mathrm{TGeom}(0.2,10)$ and $P_1 \sim \mathrm{TGeom}(0.8,10)$.} 
        \label{fig:cop2}
\end{center}
\end{figure}
\begin{figure}[H]
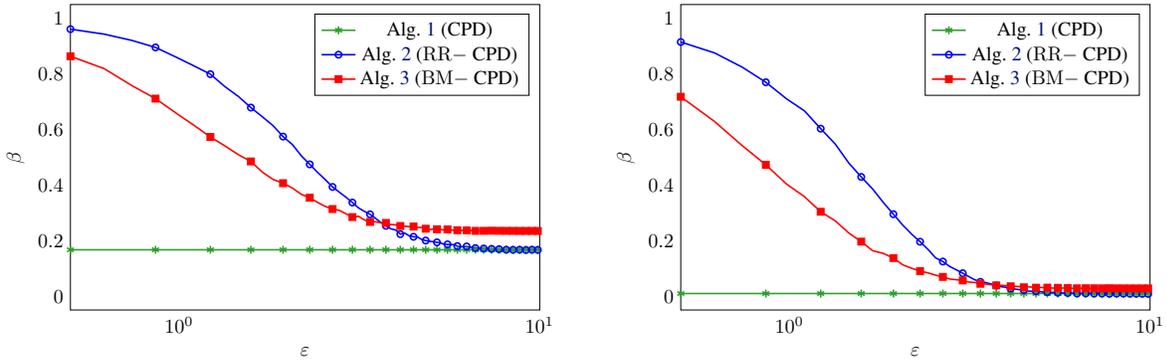

\begin{center}        
        \input{figures/emp_bin_1}
        \hspace*{4mm}
        \input{figures/emp_bin_2}
           \caption{Comparison of empirical error  probabilities as functions of privacy parameter $\varepsilon$, for binomial distributions $\mathrm{Bin}(n,p)$, and for fixed $\alpha =5$. \textbf{Left:} $P_0 \sim \mathrm{Bin}(5,0.2)$ and $P_1 \sim \mathrm{Bin}(5,0.4)$. \textbf{Right:} $P_0 \sim \mathrm{Bin}(5,0.2)$ and $P_1 \sim \mathrm{Bin}(5,0.6)$.} 
        \label{fig:cop3}
\end{center}
\end{figure}
\begin{figure}[H]
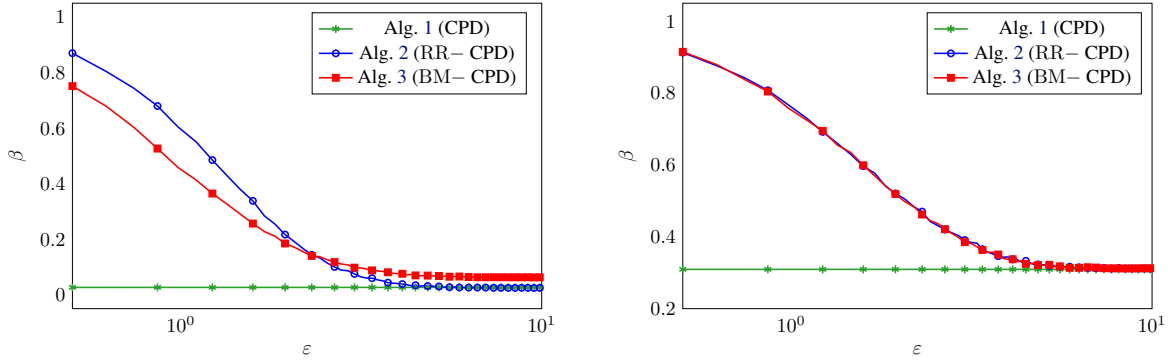

\begin{center}        
        \input{figures/emp_special}
                \hspace*{4mm}
            \input{figures/emp_binary}
           \caption{Comparison of empirical error  probabilities as functions of privacy parameter $\varepsilon$, for fixed $\alpha =5$. \textbf{Left:}  Quaternary distributions $P_0=[0.55, 0.25, 0.15, 0.05]$ and $P_1=[0.05, 0.15, 0.25, 0.55]$. \textbf{Right:} Binary distributions $P_0 \sim \mathrm{Ber}(0.1)$ and $P_1 \sim \mathrm{Ber}(0.4)$. Algorithm~\ref{alg:offline_rrcpd} and Algorithm~\ref{alg:offline_bmcpd} exhibit similar performance over binary alphabets.} 
        \label{fig:cop33}
\end{center}
\end{figure}
\subsection{Experiment plots for Continuous data}\label{app:continuous}

\begin{figure}[h]
\centering
\input{figures/timeseries_gaussian}
\caption{Gaussian (and its discretized version) time-series data before and after privatization. We consider continuous Gaussian distributions with pre-change distribution $P_0 \sim \mathcal{N}(0,5)$ and the post-change distribution $P_1 \sim \mathcal{N}(5,5)$.}
\end{figure}
\vspace*{-3mm}
We consider the Gaussian distributed data for our experiments. We sample $n=2000$ samples with the true (unknown) CP $k^{*}=1000$, where $P_0 \sim \mathcal{N}(0,5)$ \& $P_1\sim \mathcal{N}(5,5)$. As our private algorithms (Alg. $2$ and $3$) are restricted to finite alphabets, we discretize the continuous data points by constructing a Lloyd-Max quantizer from the Gaussian mixture: $0.5*P_0 + 0.5*P_1$, which yields a common alphabet $\mathcal{X}$ of size $10$ and induces discrete PMFs $P'_0$ and $P'_1$. Then, each Gaussian data point is discretized by mapping it to the corresponding quantization level obtained from the Lloyd-Max quantizer.

We discretize the Gaussian distributions $(P_0,P_1)$ into $(P'_0,P'_1)$ over a finite alphabet $\mathcal{X}$, using the Lloyd-Max algorithm Post discretization, we have the following:
\begin{align*}
 \mathcal{X} &= \{-4.3098, -2.3406, -0.8536, 0.4934, 1.8263, 3.1737, 4.5066, 5.8536, 7.3406, 9.3098\}\\
 P'_0 &= \{0.0685, 0.1690, 0.2304, 0.2301, 0.1702, 0.0888, 0.0327, 0.0087, 0.0015, 0.0001\}\\
P'_1 &= \{0.0001, 0.0015, 0.0087, 0.0327, 0.0888, 0.1702, 0.2301, 0.2304, 0.1690, 0.0685 \}.
\end{align*}
{
\begin{figure}[]
  \centering
\resizebox{0.44\textwidth}{!}{\begin{tikzpicture}[scale=0.6, transform shape]

\begin{axis}[
  width=9.5cm,
  height=7cm,
  xlabel={$\alpha$ (tolerance level)},
  ylabel={$\beta$ (error probability)},
  xmin=0, xmax=20,
  ymin=-0.01, ymax=0.1,
legend style={
  legend cell align=left,
  align=left,
  draw=white!15!black,
  at={(0.40,0.95)},
  anchor=north west
},
  grid=none,
  minor tick num=0,
  tick style={draw=none},
  minor tick style={draw=none},
]
\definecolor{tabblue}{HTML}{1F77B4}
\definecolor{taborange}{HTML}{FF7F0E}
\definecolor{tabgreen}{HTML}{2CA02C}
\definecolor{tabred}{HTML}{D62728}
\definecolor{tabpurple}{HTML}{9467BD}
\definecolor{tabgreen}{HTML}{2CA02C}

\addplot[blue, thick, mark=asterisk, mark options={scale=0.8}, mark repeat=5] coordinates {(1,0.08) (1.322,0.0630932) (1.6441,0.0461864) (1.9661,0.0292797) (2.2881,0.0220254) (2.6102,0.0159068) (2.9322,0.00978814) (3.2542,0.00761017) (3.5763,0.00648305) (3.8983,0.00535593) (4.2203,0.00444915) (4.5424,0.00364407) (4.8644,0.00283898) (5.1864,0.0020339) (5.5085,0.00122881) (5.8305,0.000423729) (6.1525,1e-300) (6.4746,1e-300) (6.7966,1e-300) (7.1186,1e-300) (7.4407,1e-300) (7.7627,1e-300) (8.0847,1e-300) (8.4068,1e-300) (8.7288,1e-300) (9.0508,1e-300) (9.3729,1e-300) (9.6949,1e-300) (10.017,1e-300) (10.339,1e-300) (10.661,1e-300) (10.983,1e-300) (11.305,1e-300) (11.627,1e-300) (11.949,1e-300) (12.271,1e-300) (12.593,1e-300) (12.915,1e-300) (13.237,1e-300) (13.559,1e-300) (13.881,1e-300) (14.203,1e-300) (14.525,1e-300) (14.847,1e-300) (15.169,1e-300) (15.492,1e-300) (15.814,1e-300) (16.136,1e-300) (16.458,1e-300) (16.78,1e-300) (17.102,1e-300) (17.424,1e-300) (17.746,1e-300) (18.068,1e-300) (18.39,1e-300) (18.712,1e-300) (19.034,1e-300) (19.356,1e-300) (19.678,1e-300) (20,1e-300)};

\addlegendentry{CPD on Gaussian data}

\addplot[red, thick, mark=asterisk, mark options={scale=0.8}, mark repeat=5] coordinates {(1,0.105) (1.322,0.0831017) (1.6441,0.0612034) (1.9661,0.0393051) (2.2881,0.0308051) (2.6102,0.0238814) (2.9322,0.0169576) (3.2542,0.0132119) (3.5763,0.0103136) (3.8983,0.00741525) (4.2203,0.00550847) (4.5424,0.00405932) (4.8644,0.00261017) (5.1864,0.00172034) (5.5085,0.00123729) (5.8305,0.000754237) (6.1525,0.000423729) (6.4746,0.000262712) (6.7966,0.000101695) (7.1186,1e-300) (7.4407,1e-300) (7.7627,1e-300) (8.0847,1e-300) (8.4068,1e-300) (8.7288,1e-300) (9.0508,1e-300) (9.3729,1e-300) (9.6949,1e-300) (10.017,1e-300) (10.339,1e-300) (10.661,1e-300) (10.983,1e-300) (11.305,1e-300) (11.627,1e-300) (11.949,1e-300) (12.271,1e-300) (12.593,1e-300) (12.915,1e-300) (13.237,1e-300) (13.559,1e-300) (13.881,1e-300) (14.203,1e-300) (14.525,1e-300) (14.847,1e-300) (15.169,1e-300) (15.492,1e-300) (15.814,1e-300) (16.136,1e-300) (16.458,1e-300) (16.78,1e-300) (17.102,1e-300) (17.424,1e-300) (17.746,1e-300) (18.068,1e-300) (18.39,1e-300) (18.712,1e-300) (19.034,1e-300) (19.356,1e-300) (19.678,1e-300) (20,1e-300)};

\addlegendentry{CPD on quantized data}

\end{axis}
\end{tikzpicture}}
  \caption{\textbf{Alg.~1 (non-private):} Empirical comparison of CPD on Gaussian and quantized time-series.}
\end{figure}
\begin{figure}[h]
\centering
\resizebox{0.44\textwidth}{!}{\input{figures/beta_v_alpha_gauss}}
\hfill
\resizebox{0.44\textwidth}{!}{\input{figures/priv_beta_alpha_gauss}}
                \caption{ \textbf{Left:} {(Non-private setting)-} Comparison of the empirical performance and theoretical bounds with existing results. \textbf{Right:} {(Private setting)- }Comparison of the empirical performances and resp. theoretical bounds with $\varepsilon =2$.}
        \label{fig:priv_tpoiss_big_bm2}
\end{figure}
\begin{figure}[H]
\centering
\resizebox{0.44\textwidth}{!}{\input{figures/beta_v_epsilon_gauss}}
                \caption{({Empirical performance of private algorithms})- Comparison of empirical performances of Alg.~2 and Alg.~3 at different values of $\varepsilon$, for a fixed value of $\alpha =5$.}
        \label{fig:priv_tpoiss_big_bm3}
\end{figure}
Furthermore, $\mathbb{E}(P'_0)=0.0291$, $\mathbb{E}(P'_1)=4.9710$, and Var($P'_0$) $=$ Var($P'_1$) $=$ $4.9351$. Thus, $(P'_0,P'_1)$ are good discrete approximations of $(P_0,P_1)$. Finally, we carry out CPD on the discrete data in both the non-private and private settings. Empirically, Alg. 1 (non-private) on the quantized data exhibits a similar performance as its continuous Gaussian counterpart. The private algorithms (Alg. $2$ and $3$) also perform well on discrete Gaussian data. 
}
\clearpage
\newpage
\section{Empirical comparison with Central-DP CPD algorithms}\label{app:dp_ldp_cpd}

In this section, we provide a comparison of the empirical performances of our LDP-CPD algorithms with the prior central/curator DP-CPD algorithm (based on the Laplace mechanism) proposed in~\cite{zhang_cpd}, on Bernoulli and Poisson distributed data. 

Interestingly, we also observe that, on the Poisson-distributed data, our Alg. 3 (BM-CPD) performs comparably to the curator DP-CPD algorithm.
    \vspace*{2mm}

\begin{figure*}[h]
\centering
\begin{subfigure}[h]{0.48\textwidth}
  \centering
 \begin{tikzpicture}[scale=0.7, transform shape]
\begin{axis}[
  width=9cm,
  height=7cm,
  xlabel={$\alpha$ (threshold)},
  ylabel={$\beta$ (error probability)},
  xmin=0, xmax=100,
  ymin=-0.05, ymax=1,
legend pos=north east,
  grid=none,
  minor tick num=0,
  tick style={draw=none},
  minor tick style={draw=none},
]
\definecolor{tabblue}{HTML}{1F77B4}
\definecolor{taborange}{HTML}{FF7F0E}
\definecolor{tabgreen}{HTML}{2CA02C}
\definecolor{tabred}{HTML}{D62728}
\definecolor{tabpurple}{HTML}{9467BD}
\definecolor{tabgreen}{HTML}{2CA02C}

\addplot[tabgreen, thick, mark=asterisk, mark options={scale=1}, mark repeat=10] coordinates {(0,0.8149) (2,0.5044) (4,0.3593) (6,0.2653) (8,0.1989) (10,0.1549) (12,0.1191) (14,0.0927) (16,0.072) (18,0.0571) (20,0.0457) (22,0.0373) (24,0.0315) (26,0.0252) (28,0.0215) (30,0.0173) (32,0.0141) (34,0.0116) (36,0.0099) (38,0.0083) (40,0.0075) (42,0.0064) (44,0.0051) (46,0.0042) (48,0.0034) (50,0.0031) (52,0.0027) (54,0.0023) (56,0.0018) (58,0.0013) (60,0.001) (62,0.0008) (64,0.0007) (66,0.0007) (68,0.0005) (70,0.0003) (72,0.0001) (74,0.0001) (76,0) (78,0) (80,0) (82,0) (84,0) (86,0) (88,0) (90,0) (92,0) (94,0) (96,0) (98,0)};

\addlegendentry{Empirical (Alg. $1$)}

\addplot[blue, thick, mark=o, mark options={scale=0.8}, mark repeat=10] coordinates {(0,0.9629) (2,0.8637) (4,0.789) (6,0.7277) (8,0.6748) (10,0.6293) (12,0.5858) (14,0.5515) (16,0.5201) (18,0.4903) (20,0.4664) (22,0.4414) (24,0.4164) (26,0.3944) (28,0.3764) (30,0.3593) (32,0.3425) (34,0.3236) (36,0.3089) (38,0.2949) (40,0.2805) (42,0.2673) (44,0.2577) (46,0.2469) (48,0.2373) (50,0.2283) (52,0.2205) (54,0.2102) (56,0.202) (58,0.1925) (60,0.1854) (62,0.1776) (64,0.1709) (66,0.1638) (68,0.1589) (70,0.153) (72,0.1474) (74,0.1411) (76,0.1358) (78,0.1305) (80,0.1262) (82,0.1209) (84,0.118) (86,0.1133) (88,0.1087) (90,0.1032) (92,0.0991) (94,0.0946) (96,0.0911) (98,0.0879)};

\addlegendentry{Empirical (Alg. $2$)}

\addplot[red, thick, mark=square*, mark options={scale=0.8}, mark repeat=10]coordinates {(0,0.9663) (2,0.871) (4,0.7946) (6,0.7303) (8,0.6789) (10,0.6334) (12,0.5904) (14,0.5547) (16,0.5245) (18,0.4978) (20,0.4694) (22,0.4446) (24,0.4209) (26,0.3987) (28,0.3804) (30,0.3625) (32,0.3467) (34,0.3317) (36,0.3151) (38,0.3018) (40,0.2882) (42,0.2774) (44,0.2661) (46,0.2546) (48,0.2445) (50,0.2345) (52,0.2269) (54,0.2182) (56,0.2094) (58,0.2018) (60,0.194) (62,0.1868) (64,0.1783) (66,0.1713) (68,0.1643) (70,0.1583) (72,0.1516) (74,0.1461) (76,0.1407) (78,0.1353) (80,0.1287) (82,0.1231) (84,0.1191) (86,0.1149) (88,0.1102) (90,0.1062) (92,0.1019) (94,0.098) (96,0.0946) (98,0.091)};

\addlegendentry{Empirical (Alg. $3$)}

\addplot[black, thick, mark=square*, mark options={scale=0.8}, mark repeat=10] coordinates {(0,0.939) (2,0.748) (4,0.6024) (6,0.4868) (8,0.4006) (10,0.3303) (12,0.2724) (14,0.2292) (16,0.1896) (18,0.1577) (20,0.1296) (22,0.1096) (24,0.0914) (26,0.0759) (28,0.064) (30,0.0549) (32,0.0463) (34,0.04) (36,0.0329) (38,0.0283) (40,0.0236) (42,0.0207) (44,0.017) (46,0.0144) (48,0.0128) (50,0.0112) (52,0.0093) (54,0.0081) (56,0.007) (58,0.006) (60,0.0051) (62,0.0042) (64,0.0033) (66,0.0028) (68,0.0025) (70,0.0024) (72,0.0022) (74,0.0021) (76,0.002) (78,0.0016) (80,0.0014) (82,0.0014) (84,0.0013) (86,0.0013) (88,0.0013) (90,0.0011) (92,0.001) (94,0.0008) (96,0.0006) (98,0.0004)};

\addlegendentry{Empirical (Zhang et al.)$[56]$}

\end{axis}
\end{tikzpicture}
  \caption{\textbf{Bernoulli family:} $P_0 \sim \text{Ber}(0.1)$; $P_1 \sim \text{Ber}(0.4)$.}
  
\end{subfigure}\hfill
\begin{subfigure}[h]{0.48\textwidth}
  \centering
  \begin{tikzpicture}[scale=0.7, transform shape]
\begin{axis}[
  width=9cm,
  height=7cm,
  xlabel={$\alpha$ (threshold)},
  ylabel={$\beta$ (error probability)},
  xmin=0, xmax=100,
  ymin=-0.05, ymax=1,
legend pos=north east,
  grid=none,
  minor tick num=0,
  tick style={draw=none},
  minor tick style={draw=none},
]
\definecolor{tabblue}{HTML}{1F77B4}
\definecolor{taborange}{HTML}{FF7F0E}
\definecolor{tabgreen}{HTML}{2CA02C}
\definecolor{tabred}{HTML}{D62728}
\definecolor{tabpurple}{HTML}{9467BD}
\definecolor{tabgreen}{HTML}{2CA02C}

\addplot[tabgreen, thick, mark=asterisk, mark options={scale=1}, mark repeat=10] coordinates {(0,0.4006) (2,0.0952) (4,0.0326) (6,0.0124) (8,0.0047) (10,0.0019) (12,0.0007) (14,0.0003) (16,0.0001) (18,0.0001) (20,0) (22,0) (24,0) (26,0) (28,0) (30,0) (32,0) (34,0) (36,0) (38,0) (40,0) (42,0) (44,0) (46,0) (48,0) (50,0) (52,0) (54,0) (56,0) (58,0) (60,0) (62,0) (64,0) (66,0) (68,0) (70,0) (72,0) (74,0) (76,0) (78,0) (80,0) (82,0) (84,0) (86,0) (88,0) (90,0) (92,0) (94,0) (96,0) (98,0)};

\addlegendentry{Empirical (Alg. $1$)}

\addplot[blue, thick, mark=o, mark options={scale=0.8}, mark repeat=10] coordinates {(0,0.8387) (2,0.563) (4,0.4109) (6,0.3183) (8,0.245) (10,0.1978) (12,0.1599) (14,0.129) (16,0.1031) (18,0.0865) (20,0.0708) (22,0.06) (24,0.0491) (26,0.0433) (28,0.0358) (30,0.0309) (32,0.0266) (34,0.0233) (36,0.0202) (38,0.0167) (40,0.0139) (42,0.0122) (44,0.01) (46,0.0078) (48,0.007) (50,0.0062) (52,0.0049) (54,0.0045) (56,0.0038) (58,0.0031) (60,0.0028) (62,0.0022) (64,0.0017) (66,0.0012) (68,0.0011) (70,0.0009) (72,0.0008) (74,0.0007) (76,0.0007) (78,0.0006) (80,0.0004) (82,0.0004) (84,0.0004) (86,0.0003) (88,0.0003) (90,0.0002) (92,0.0002) (94,0.0002) (96,0.0002) (98,0.0002)};

\addlegendentry{Empirical (Alg. $2$)}

\addplot[red, thick, mark=square*, mark options={scale=0.8}, mark repeat=10]coordinates {(0,0.838) (2,0.5673) (4,0.4184) (6,0.3193) (8,0.2486) (10,0.2016) (12,0.1632) (14,0.1304) (16,0.1087) (18,0.0891) (20,0.0746) (22,0.0635) (24,0.0533) (26,0.0442) (28,0.0373) (30,0.0322) (32,0.0281) (34,0.0234) (36,0.0209) (38,0.0183) (40,0.016) (42,0.0141) (44,0.0118) (46,0.01) (48,0.0087) (50,0.0073) (52,0.0069) (54,0.0059) (56,0.0048) (58,0.0038) (60,0.0035) (62,0.0028) (64,0.0024) (66,0.0019) (68,0.0018) (70,0.0016) (72,0.0016) (74,0.0013) (76,0.0011) (78,0.001) (80,0.0008) (82,0.0006) (84,0.0005) (86,0.0005) (88,0.0005) (90,0.0005) (92,0.0004) (94,0.0003) (96,0.0003) (98,0.0003)};

\addlegendentry{Empirical (Alg. $3$)}

\addplot[black, thick, mark=square*, mark options={scale=0.8}, mark repeat=10] coordinates {(0,0.8301) (2,0.4353) (4,0.239) (6,0.1307) (8,0.0769) (10,0.0439) (12,0.0254) (14,0.0159) (16,0.0093) (18,0.0051) (20,0.0028) (22,0.0017) (24,0.001) (26,0.0009) (28,0.0003) (30,0.0001) (32,0.0001) (34,0) (36,0) (38,0) (40,0) (42,0) (44,0) (46,0) (48,0) (50,0) (52,0) (54,0) (56,0) (58,0) (60,0) (62,0) (64,0) (66,0) (68,0) (70,0) (72,0) (74,0) (76,0) (78,0) (80,0) (82,0) (84,0) (86,0) (88,0) (90,0) (92,0) (94,0) (96,0) (98,0)};

\addlegendentry{Empirical (Zhang et al.)$[56]$}

\end{axis}
\end{tikzpicture}
 
  \caption{\textbf{Bernoulli family:} $P_0 \sim \text{Ber}(0.1)$; $P_1 \sim \text{Ber}(0.8)$.}
\end{subfigure}
   \vspace*{2mm}
\caption{Empirical comparison of $\varepsilon$-LDP-CPD algorithms, i.e. Alg.~2 ($\mathrm{ARR-CPD}$) and  Alg.~3 (BM-CPD) with $\varepsilon-$Central DP-CPD algorithm based on Laplace mechanism studied in~\cite{zhang_cpd}. For comparison, we also plot the non-private CPD estimator in Alg.~1. We consider binary data sampled from Bernoulli distributions and we fix privacy parameter $\varepsilon=1$. We fix the dataset size $n=2000$ and the true change-point $k^{*}=1000$.}
\end{figure*}
\begin{figure*}[h]
\centering
\begin{subfigure}[h]{0.48\textwidth}
  \centering
 \begin{tikzpicture}[scale=0.7, transform shape]
\begin{axis}[
  width=9cm,
  height=7cm,
  xlabel={$\alpha$ (threshold)},
  ylabel={$\beta$ (error probability)},
  xmin=0, xmax=100,
  ymin=-0.05, ymax=1,
legend pos=north east,
  grid=none,
  minor tick num=0,
  tick style={draw=none},
  minor tick style={draw=none},
]
\definecolor{tabblue}{HTML}{1F77B4}
\definecolor{taborange}{HTML}{FF7F0E}
\definecolor{tabgreen}{HTML}{2CA02C}
\definecolor{tabred}{HTML}{D62728}
\definecolor{tabpurple}{HTML}{9467BD}
\definecolor{tabgreen}{HTML}{2CA02C}

\addplot[tabgreen, thick, mark=asterisk, mark options={scale=1}, mark repeat=10] coordinates {(0,0.3574) (2,0.0578) (4,0.0129) (6,0.0031) (8,0.0011) (10,0.0003) (12,0.0001) (14,0.0001) (16,0) (18,0) (20,0) (22,0) (24,0) (26,0) (28,0) (30,0) (32,0) (34,0) (36,0) (38,0) (40,0) (42,0) (44,0) (46,0) (48,0) (50,0) (52,0) (54,0) (56,0) (58,0) (60,0) (62,0) (64,0) (66,0) (68,0) (70,0) (72,0) (74,0) (76,0) (78,0) (80,0) (82,0) (84,0) (86,0) (88,0) (90,0) (92,0) (94,0) (96,0) (98,0)};

\addlegendentry{Empirical (Alg. $1$)}

\addplot[blue, thick, mark=o, mark options={scale=0.8}, mark repeat=10] coordinates {(0,0.976) (2,0.9022) (4,0.8451) (6,0.7931) (8,0.7465) (10,0.7059) (12,0.6729) (14,0.6398) (16,0.6085) (18,0.5791) (20,0.5514) (22,0.526) (24,0.5065) (26,0.4844) (28,0.4632) (30,0.4441) (32,0.4255) (34,0.4085) (36,0.3933) (38,0.38) (40,0.366) (42,0.3523) (44,0.3398) (46,0.3273) (48,0.3153) (50,0.3058) (52,0.2965) (54,0.287) (56,0.2783) (58,0.2679) (60,0.259) (62,0.2519) (64,0.2436) (66,0.2374) (68,0.2299) (70,0.2233) (72,0.2153) (74,0.2084) (76,0.2033) (78,0.1962) (80,0.1902) (82,0.1848) (84,0.1795) (86,0.1734) (88,0.1688) (90,0.1635) (92,0.1583) (94,0.1537) (96,0.149) (98,0.1453)};

\addlegendentry{Empirical (Alg. $2$)}

\addplot[red, thick, mark=square*, mark options={scale=0.8}, mark repeat=10] coordinates {(0,0.846) (2,0.5743) (4,0.4306) (6,0.3363) (8,0.2692) (10,0.2165) (12,0.1775) (14,0.1424) (16,0.1196) (18,0.1) (20,0.0823) (22,0.0675) (24,0.0573) (26,0.0472) (28,0.04) (30,0.0337) (32,0.0284) (34,0.024) (36,0.0205) (38,0.0175) (40,0.0151) (42,0.0129) (44,0.0112) (46,0.01) (48,0.0089) (50,0.0078) (52,0.0071) (54,0.0057) (56,0.0049) (58,0.0041) (60,0.004) (62,0.0034) (64,0.003) (66,0.0027) (68,0.0024) (70,0.002) (72,0.0018) (74,0.0013) (76,0.0012) (78,0.0011) (80,0.001) (82,0.0009) (84,0.0007) (86,0.0007) (88,0.0006) (90,0.0004) (92,0.0002) (94,0.0002) (96,0.0002) (98,0.0001)};

\addlegendentry{Empirical (Alg. $3$)}

\addplot[black, thick, mark=square*, mark options={scale=0.8}, mark repeat=10] coordinates {(0,0.9222) (2,0.6797) (4,0.5078) (6,0.3755) (8,0.2875) (10,0.219) (12,0.1691) (14,0.1302) (16,0.097) (18,0.0746) (20,0.0563) (22,0.0426) (24,0.0327) (26,0.0251) (28,0.0197) (30,0.0155) (32,0.0123) (34,0.01) (36,0.0068) (38,0.0055) (40,0.0041) (42,0.0034) (44,0.0026) (46,0.0021) (48,0.0019) (50,0.0015) (52,0.0014) (54,0.0009) (56,0.0005) (58,0.0005) (60,0.0005) (62,0.0004) (64,0.0003) (66,0.0002) (68,0.0001) (70,0.0001) (72,0.0001) (74,0.0001) (76,0.0001) (78,0) (80,0) (82,0) (84,0) (86,0) (88,0) (90,0) (92,0) (94,0) (96,0) (98,0)};

\addlegendentry{Empirical (Zhang et al.)$[56]$}

\end{axis}
\end{tikzpicture}
  \caption{\textbf{Truncated Poisson family:} $P_0 \sim \text{TPois}(1,10)$; $P_1 \sim \text{TPois}(4,10)$.}

\end{subfigure}\hfill
\begin{subfigure}[ht]{0.48\textwidth}
  \centering
  \begin{tikzpicture}[scale=0.7, transform shape]
\begin{axis}[
  width=9cm,
  height=7cm,
  xlabel={$\alpha$ (threshold)},
  ylabel={$\beta$ (error probability)},
  xmin=0, xmax=100,
  ymin=-0.05, ymax=1,
legend pos=north east,
  grid=none,
  minor tick num=0,
  tick style={draw=none},
  minor tick style={draw=none},
]
\definecolor{tabblue}{HTML}{1F77B4}
\definecolor{taborange}{HTML}{FF7F0E}
\definecolor{tabgreen}{HTML}{2CA02C}
\definecolor{tabred}{HTML}{D62728}
\definecolor{tabpurple}{HTML}{9467BD}
\definecolor{tabgreen}{HTML}{2CA02C}

\addplot[tabgreen, thick, mark=asterisk, mark options={scale=1}, mark repeat=10] coordinates {(0,0.0794) (2,0.0015) (4,0.0001) (6,0) (8,0) (10,0) (12,0) (14,0) (16,0) (18,0) (20,0) (22,0) (24,0) (26,0) (28,0) (30,0) (32,0) (34,0) (36,0) (38,0) (40,0) (42,0) (44,0) (46,0) (48,0) (50,0) (52,0) (54,0) (56,0) (58,0) (60,0) (62,0) (64,0) (66,0) (68,0) (70,0) (72,0) (74,0) (76,0) (78,0) (80,0) (82,0) (84,0) (86,0) (88,0) (90,0) (92,0) (94,0) (96,0) (98,0)};

\addlegendentry{Empirical (Alg. $1$)}

\addplot[blue, thick, mark=o, mark options={scale=0.8}, mark repeat=10] coordinates {(0,0.9577) (2,0.8491) (4,0.7708) (6,0.7073) (8,0.6506) (10,0.6061) (12,0.5643) (14,0.527) (16,0.4927) (18,0.4639) (20,0.4361) (22,0.4115) (24,0.3872) (26,0.3635) (28,0.3443) (30,0.3266) (32,0.3104) (34,0.2966) (36,0.2796) (38,0.2676) (40,0.2549) (42,0.2423) (44,0.2296) (46,0.2194) (48,0.2094) (50,0.2003) (52,0.1889) (54,0.1808) (56,0.1722) (58,0.1647) (60,0.1581) (62,0.1507) (64,0.1446) (66,0.1383) (68,0.1305) (70,0.1264) (72,0.1208) (74,0.1148) (76,0.1096) (78,0.1059) (80,0.1009) (82,0.096) (84,0.0926) (86,0.0896) (88,0.0856) (90,0.0822) (92,0.0783) (94,0.0758) (96,0.0731) (98,0.0698)};
\addlegendentry{Empirical (Alg. $2$)}

\addplot[red, thick, mark=square*, mark options={scale=0.8}, mark repeat=10] coordinates {(0,0.741) (2,0.4041) (4,0.2531) (6,0.1685) (8,0.1168) (10,0.0797) (12,0.0578) (14,0.0411) (16,0.028) (18,0.0216) (20,0.0155) (22,0.0116) (24,0.0086) (26,0.0061) (28,0.0042) (30,0.0029) (32,0.0025) (34,0.0014) (36,0.0012) (38,0.0009) (40,0.0006) (42,0.0006) (44,0.0006) (46,0.0004) (48,0.0003) (50,0.0003) (52,0.0003) (54,0.0003) (56,0.0003) (58,0.0002) (60,0.0001) (62,0.0001) (64,0.0001) (66,0.0001) (68,0.0001) (70,0.0001) (72,0.0001) (74,0.0001) (76,0.0001) (78,0.0001) (80,0.0001) (82,0) (84,0) (86,0) (88,0) (90,0) (92,0) (94,0) (96,0) (98,0)};

\addlegendentry{Empirical (Alg. $3$)}

\addplot[black, thick, mark=square*, mark options={scale=0.8}, mark repeat=10] coordinates {(0,0.8421) (2,0.4598) (4,0.2615) (6,0.1469) (8,0.0882) (10,0.0517) (12,0.0319) (14,0.0206) (16,0.0124) (18,0.0073) (20,0.0037) (22,0.0023) (24,0.0014) (26,0.0012) (28,0.0006) (30,0.0004) (32,0.0001) (34,0.0001) (36,0) (38,0) (40,0) (42,0) (44,0) (46,0) (48,0) (50,0) (52,0) (54,0) (56,0) (58,0) (60,0) (62,0) (64,0) (66,0) (68,0) (70,0) (72,0) (74,0) (76,0) (78,0) (80,0) (82,0) (84,0) (86,0) (88,0) (90,0) (92,0) (94,0) (96,0) (98,0)};

\addlegendentry{Empirical (Zhang et al.)$[56]$}

\end{axis}
\end{tikzpicture}
 \caption{\textbf{Truncated Poisson family:} $P_0 \sim \text{TPois}(1,10)$; $P_1 \sim \text{TPois}(8,10)$.}
\end{subfigure}
  \vspace*{2mm}
\caption{Empirical comparison of $\varepsilon$-LDP-CPD algorithms, i.e. Alg.~2 ($\mathrm{ARR-CPD}$) and  Alg.~3 (BM-CPD) with $\varepsilon-$Central DP-CPD algorithm based on Laplace mechanism studied in~\cite{zhang_cpd}. For comparison, we also plot the non-private CPD estimator in Alg.~1. We consider Poisson data sampled from Truncated Poisson distributions and we fix privacy parameter $\varepsilon=1$. We fix the dataset size $n=2000$ and the true change-point $k^{*}=1000$.}
\end{figure*}

\textbf{Scope of the comparison.} It is important to note that such a comparison is not an apples-to-apples benchmark since both privacy notions are fundamentally different: LDP is more restrictive than curator-model DP and typically incurs a larger utility cost. Thus, one should not expect LDP methods to uniformly outperform the curator-DP baselines at the same privacy level. Rather, this comparison is to contextualize the additional cost of `local' privacy. 
\newpage
\section{Comparison with Hypothesis-testing based benchmark}\label{app:testing}

We construct a black-box benchmark for locally private change-point localization by combining the testing-to-localization reduction of Canonne et al.~\cite{canonne2019structure} with the locally private hypothesis-testing results of Pensia et al.~\cite{pensia}. Canonne et al. develop their private CPD application under central differential privacy. The construction below adapts their reduction to the local differential privacy setting. Throughout as in the paper, here too,  we consider our setting of exactly one change point between two known distributions $(P_0, P_1)$.

\textbf{Testing-based construction.} Let $\mathrm{SC}_\varepsilon(P_0,P_1)$ denote the minimum sample size required by a non-interactive $\varepsilon$-LDP test that distinguishes $P_0$ from $P_1$, with each error probability at most $1/3$. Choose such a test using $m$ observations, where $m\asymp\mathrm{SC}_\varepsilon(P_0,P_1)$. Let us partition the observations into $B=\lfloor n/m\rfloor$ consecutive blocks of length $m$, discarding any terminal remainder. Apply independent copies of the testing protocol to these blocks. Let us further define
$$
Z_j=
\begin{cases}
+1,&\text{if the test outputs }P_0,\ \\
-1,&\text{if the test outputs }P_1.
\end{cases}
$$
For every block lying entirely before the change, $\Pr(Z_j=+1)\ge2/3$, and for every block lying entirely after it, $\Pr(Z_j=-1)\ge2/3$. At most one retained block contains observations from both distributions, and its output can be treated as arbitrary.

One implementation of the resulting Bernoulli localization procedure is
\[
  \widehat b\in\operatorname*{arg\,max}_{0\le b\le B}\sum_{j=1}^{b}Z_j,
  \qquad \text{and} \qquad
  \widehat k_{\mathrm{hp}}=\min\bigl\{n,\;\max\{2,\;m\widehat b+1\}\bigr\},
\]
where the empty sum is zero and $\mathcal K=\{2,\ldots,n\}$. Including $b=0$ and $b=B$ handles changes near the endpoints. The testing-to-localization argument gives universal constants $C_0,\,c_0>0$ such that
\begin{align*}
\beta_{\mathrm{hp}}(\alpha)
&\triangleq
\sup_{k^\star\in\mathcal K}
\mathbb{P}_{k^\star}\left(
|\widehat k_{\mathrm{hp}}-k^\star|>\alpha
\right)
\le
\min\left(
1,
C_0\exp\left[
-\frac{c_0\alpha}{\mathrm{SC}_\varepsilon(P_0,P_1)}
\right]
\right).
\end{align*}
The block containing the change and the conversion from block indices to observation indices affect only the universal constants. This is an achievability bound, rather than an equality characterizing the procedure’s actual error. Each observation is privatized once through an $\varepsilon$-LDP channel; the block decisions and the final estimate are post-processing. Consequently, the construction satisfies $\varepsilon$-LDP.

\textbf{Scope of the comparison.} The benchmark optimizes the testing protocol for the given pair $(P_0,P_1)$. Its guarantee requires a privatization mechanism chosen for that pair. Pensia et al.~\cite{pensia} provide constructive, computationally efficient testing procedures under LDP. For $0<\varepsilon\le1$, the binary mechanism of Definition~\ref{def:bm} with $\tau=1$ attains the optimal testing sample complexity up to universal constants. Thus, the block construction can use the same binary-mechanism family as BM-CPD, within which BM-CPD selects the threshold that maximizes the induced Chernoff information.~Both of our algorithms, BM-CPD and ARR-CPD, apply the GLRT directly to individual privatized observations, avoiding block aggregation and giving explicit finite-sample localization guarantees in terms of the induced Chernoff information.

\textbf{High-privacy guarantees and minimax optimality.} Let $\delta=\mathrm d_{\mathrm{TV}}(P_0,P_1)$,
$I=\mathrm I_{\mathrm{ch}}(P_0,P_1)$, and
$c_{\varepsilon,2}=\tanh(\varepsilon/2)$. Throughout our analysis, $\asymp$ hides
universal constants. For $0<\varepsilon\leq 1$, the locally private testing
complexity satisfies $\mathrm{SC}_\varepsilon(P_0,P_1)\asymp1/(\varepsilon^2\delta^2)$
\cite{pensia}, so the block construction gives
\[
  \beta_{\mathrm{hp}}\le
  \min\bigl(1,\,C_1e^{-c_1\alpha\varepsilon^2\delta^2}\bigr)
\]
for universal constants $C_1,c_1>0$. The BM-CPD finite-sample guarantee attains the same dependence
with explicit constants i.e., by Theorem~\ref{thm:bmcpd}, we have
\[
  \beta_{\mathrm{BM}}\le
  2\exp\Bigl(-\frac{\alpha+1}{2}\,
  c_{\varepsilon,2}^2\delta^2\Bigr),
\]
and $c_{\varepsilon,2}\asymp\varepsilon$ on $(0,1]$. This scale is unavoidable. The quantities in our private minimax lower bound in Theorem~\ref{thm:private-minimax-lb} satisfy
$\bar S_\varepsilon\le\log\bigl(1+(e^\varepsilon-1)\delta\bigr)
\le(e^\varepsilon-1)\delta$ and, since $\log\cosh u\le u^2/2$, we have
\[
  \bar I_\varepsilon
  \le\log\cosh(\bar S_\varepsilon/2)
  \le\frac{(e^\varepsilon-1)^2\delta^2}{8}.
\]
Hence, whenever $M_\alpha\ge2$,
\[
  E^\varepsilon_{\mathcal K}(\alpha)\ge
  \Bigl(1-\frac1{M_\alpha}\Bigr)
  \exp\Bigl(
    -\frac{(2\alpha+1)(e^\varepsilon-1)^2\delta^2}{8}
    -\frac{(e^\varepsilon-1)\delta}{2}\sqrt{2\alpha+1}
  \Bigr).
\]
Note that our minimax lower bound covers all non-interactive $\varepsilon$-LDP
protocols, including those that use different mechanisms at different
observations. Let $E_{p}$ and $E_{{np}}$ be the minimax exponents
defined in Appendix~\ref{app:cop}. Since
$E^\varepsilon_{\mathcal K}(\alpha)\le\beta_{\mathrm{BM}}$,
taking negative logarithms of the preceding bounds, dividing by
$\alpha_n$, and taking the lim inf yields
\[
  \frac{c_{\varepsilon,2}^2\delta^2}{2}
  \le E_{\mathrm p}\le
  \frac{(e^\varepsilon-1)^2\delta^2}{4},
\]
as the logarithm of the pre-factors and square-root correction vanish after
normalization. Thus $E_{ p}\asymp\varepsilon^2\delta^2$ for
$0<\varepsilon\le1$, and both BM-CPD and the optimized testing-based
construction attain this scale.

Since Theorems~\ref{thm:minimax_lb} and~\ref{thm:npcpd} give $I\leq E_{\mathrm{np}}\leq 2I$, the cost
of privacy $C_{\exp}(\varepsilon)=E_{{np}}/E_{p}$
satisfies
\[
  \frac{4I}{(e^\varepsilon-1)^2\delta^2}
  \le C_{\exp}(\varepsilon)\le
  \frac{4I}{c_{\varepsilon,2}^2\delta^2}.
\]
The lower bound strengthens~\eqref{eq:cost_priv} when $I/\delta^2$ is sufficiently
large. The ratio of the upper and lower bounding expressions is $(e^\varepsilon+1)^2$, which is below $14$ on $(0,1]$ and tends to $4$ as $\varepsilon\to 0$. Consequently, $C_{\exp}(\varepsilon)\asymp I/(\varepsilon^2\delta^2)$ for $0<\varepsilon\le1$.

\textbf{Low-privacy guarantees.} Let $D=\mathrm H^2(P_0,P_1)$. Once, we have $e^{\varepsilon}\gtrsim D^{-1}\log(1/D)$, the locally private testing complexity reaches the non-private scale $\mathrm{SC}_\varepsilon(P_0,P_1)\asymp1/D$ \cite{pensia}, giving a block-based
guarantee of the form
\[
  \beta_{\mathrm{hp}}\le\min\bigl(1,\,C_2e^{-c_2\alpha D}\bigr).
\]
This conclusion uses testing mechanisms with sufficiently rich output alphabets; it need not follow from a binary-output mechanism alone.

Our ARR-CPD also recovers the Hellinger scale. For fixed $(P_0,P_1)$, the Hellinger distance based term in Theorem~\ref{thm:rrcpd} satisfies $L_{\varepsilon,r}(D)\to D$ as $\varepsilon\to\infty$, so $\rho_{\mathrm{ARR}}\to1-D$ and the guaranteed decay rate $-\log\rho_{\mathrm{ARR}}$ approaches $-\log(1-D)\geq D$. Moreover, applying Theorem~\ref{thm:npcpd} directly to the ARR output distributions gives a Chernoff-based
guarantee whose rate converges to $\mathrm I_{\mathrm{ch}}(P_0,P_1)$. This holds because ARR approaches the identity on the likelihood-ratio classes, and aggregation within these classes preserves Chernoff information. Hence, at low privacy the relevant guarantees are the Hellinger and Chernoff based bounds, not the weaker TV based term.

\textbf{Practical comparison.} From an implementation perspective, the block construction replaces each block of
$m\asymp\mathrm{SC}_\varepsilon(P_0,P_1)$ observations by one testing decision. Its estimate therefore lies on a grid of spacing $m$, and the block length is an additional design parameter that trades the reliability of each decision
against resolution. On the other hand, ARR-CPD and BM-CPD need no block length and search at observation-level resolution. Moreover, for a common channel with output distributions $(Q_0,Q_1)$, a block decision with distributions $(R_0,R_1)$ satisfies $\mathrm I_{\mathrm{ch}}(R_0,R_1)\le m\,\mathrm I_{\mathrm{ch}}(Q_0,Q_1)$, so reducing a block to one bit cannot increase its Chernoff information; this is the advantage of retaining the full privatized sequence. These are implementation advantages for our GLRT based algorithms.

\textbf{Summary.}
For $0<\varepsilon\leq1$, the BM-CPD algorithm attains the minimax exponential scale $\varepsilon^2 \mathrm d_{\mathrm{TV}}^2(P_0,P_1)$, as does the optimized testing-based construction. For BM-CPD, the guaranteed asymptotic exponent $c_{\varepsilon,2}^2\mathrm d_{\mathrm{TV}}^2(P_0,P_1)/2$ is within a factor
$(e^\varepsilon+1)^2/2$ of the minimax exponent $E_{ p}$. This comparison factor is below $7$ on $(0,1]$ and tends to $2$ as $\varepsilon\to 0$, independently of $(P_0,P_1)$.

For sufficiently large $\varepsilon$, the testing construction guarantees decay on the scale $\mathrm H^2(P_0,P_1)$. As $\varepsilon\to\infty$, the ARR-CPD bound (Thm.~\ref{thm:rrcpd}) recovers this dependence, and its direct Chernoff-based guarantee approaches the rate $\mathrm I_{\mathrm{ch}}(P_0,P_1)$ in the non-private GLRT
bound, which is within a factor of two of the non-private minimax exponent $E_{{np}}$ (Thms.~\ref{thm:minimax_lb} and~\ref{thm:npcpd}). Since
\[
\mathrm  I_{\mathrm{ch}}(P_0,P_1)
  \geq -\log\bigl(1-\mathrm H^2(P_0,P_1)\bigr)
  \geq \mathrm H^2(P_0,P_1),
\]
this limiting rate is never below the Hellinger scale and can be much larger for well-separated distributions. Our procedures also give explicit finite-sample bounds at every $\varepsilon>0$ and avoid block aggregation. These results do not imply a uniform ordering of the procedures' actual localization errors; we do not establish an ordering of their guarantees at intermediate privacy levels.

\end{document}